\documentclass[journal]{IEEEtran}
\usepackage{cite}
\usepackage{graphicx}
\usepackage{subcaption}
\usepackage{amsmath}
\usepackage{amssymb}
\usepackage{amsthm}
\usepackage{algorithmic,algorithm}
\usepackage{array}
\usepackage{hyperref}
\usepackage{bm}
\usepackage{url}
\usepackage{color}

\usepackage{nccmath,mathtools}

\newtheorem{proposition}{Proposition}
\newtheorem{theorem}{Theorem}
\newtheorem{lemma}{Lemma}
\newtheorem{assumption}{Assumption}

\theoremstyle{definition}

\theoremstyle{remark}
\newtheorem{remark}{Remark}

\newcommand{\tr}{\operatorname{Tr}}

\newcommand{\modified}[1]{{\color{black}{#1}}}

\DeclareMathOperator{\median}{median}

\begin{document}
	\title{Distributed Information-based Source Seeking}
	
	\author{Tianpeng Zhang$^{1}$,
	    Victor Qin$^{2}$,
        Yujie Tang$^{3}$,
         Na Li$^{1}$

\thanks{The work is supported by AFOSR YIP FA9550-18-1-0150, ONR YIP N00014-19-1-2217, and NSF CNS: 2003111.}
\thanks{$^{1}$ T. Zhang and N. Li are with Harvard School of Engineering and Applied Sciences, Cambridge, MA, USA. Emails: {\tt\small tzhang@g.harvard.edu, nali@seas.harvard.edu.} }
\thanks{$^2$ V. Qin is with MIT School of Engineering,
Cambridge, MA, USA. Email: {\tt\small victorqi@mit.edu.}}
\thanks{$^3$ Y. Tang is with the Department of Industrial Engineering and Management, Peking University, Beijing, China. Email: {\tt\small yujietang@pku.edu.cn.}}
}

	\maketitle
	
\begin{abstract}
	 In this paper, we design an information-based multi-robot source seeking algorithm where a group of mobile sensors localizes and moves close to a single source using only local range-based measurements. In the algorithm, the mobile sensors perform source identification/localization to estimate the source location; meanwhile, they move to new locations to maximize the Fisher information about the source contained in the sensor measurements. In doing so, they improve the source location estimate and move closer to the source. Our algorithm is superior in convergence speed compared with traditional field climbing algorithms, is flexible in the measurement model and the choice of information metric, and is robust to measurement model errors. Moreover, we provide a fully distributed version of our algorithm, where each sensor decides its own actions and only shares information with its neighbors through a sparse communication network. We perform extensive simulation experiments to test our algorithms on large-scale systems and implement physical experiments on small ground vehicles with light sensors, demonstrating success in seeking a light source.
\end{abstract}
	
	\IEEEpeerreviewmaketitle

\section{Introduction}

Multi-agent source seeking is a robotics task that uses autonomous vehicles with sensors to locate a source of interest whose position is unknown. The source of interest can be a light source\cite{duisterhof2021learning}, a radio signal transmitter\cite{xu_optimal_2019}, or a chemical leakage point \cite{khodayi-mehr_model-based_2019}. The source seeking vehicles, or mobile sensors, can measure the source's influence on the environment and use this information to locate the source. 

A large body of source seeking research investigates field climbing methods \cite{angelico_source_2019,bachmayer_vehicle_2002,ogren_cooperative_2004,moore_source_2010,li_cooperative_2014,brinon-arranz_distributed_2016}. Assuming the source signal gets stronger as the sensor-source distance shortens, the mobile sensors can ``climb" the source signal field to physically approach the source. These methods do not require explicit knowledge of the measurement model, making them easy to implement for different applications. However, field climbing methods are not necessarily the most effective source seeking methods. Firstly, they only exploit local information of the source field, with the typical requirement that sensors must maintain a tight formation to make a reasonable ascent direction estimate, as is the case in \cite{ogren_cooperative_2004,moore_source_2010}. Furthermore, the sensors cannot move
too fast as a group for the measurement value to increase stably

\begin{figure}
		\centering
		\includegraphics[width=0.65\linewidth]{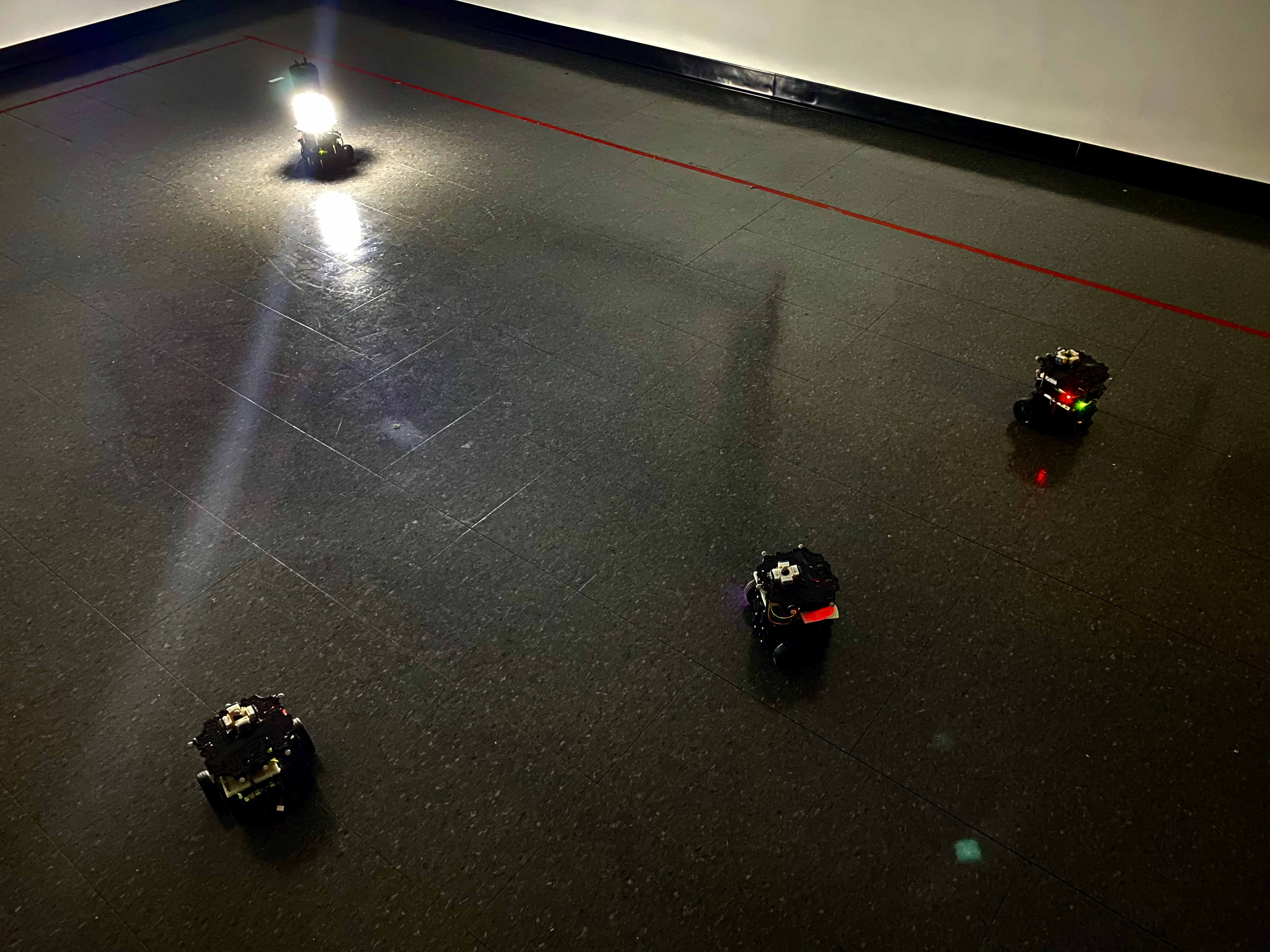}
		\caption{A snapshot of seeking a light source in a dark room with three mobile light sensors.}
	\label{fig:teaser}
	\vspace{-15pt}
\end{figure}

An alternative approach is to perform source identification/localization using various estimation methods, such as (Extended) Kalman Filter (EKF)\cite{morbidi_active_2013,martinez_optimal_2006}, Particle Filter (PF)\cite{bayat_optimal_2016,hoffmann_mobile_2010}, and so on, to estimate the source location over time. 
These methods enable the fusion of measurements from multiple sensors to identify a global view of the measurement field. A central theme in this line of work is to improve the estimation through sensor movements. Many have proposed to move the sensors to optimize specific information metrics, in particular, variants of Fisher information measures\cite{martinez_optimal_2006,khodayi-mehr_model-based_2019,moreno-salinas_optimal_2013,xu_optimal_2019,bayat_optimal_2016,bishop_optimal_2008,lee_active_2018,ponda_trajectory_2009}, which relate closely to the famous Cram\'{e}r-Rao lower bound\cite{cramir1946mathematical,rao1945information}. 
However, these studies typically focus on deriving closed-form solutions of optimal sensor placement for particular types of measurement models \cite{xu_optimal_2019, bayat_optimal_2016,lee_active_2018}. Their results are usually not generalizable or robust to modeling error, and closed-form solutions might not exist for many general measurement models.

This paper draws advantages from the abovementioned methods (field climbing and source localization) to develop multi-robot source-seeking algorithms. In particular, we assume the measurement model is available and propose an algorithm using range-based measurements. Each iteration of the algorithm consists of three steps: 
\begin{enumerate}
    \item Collect range-based measurements.
    \item Perform source location estimation.
    \item Move along the negative gradient direction of information-based loss function $L$, defined as the trace of the inverse of Fisher information. 
\end{enumerate}
In particular, step 3) seeks to increase the information about the source location in the measurement rather than looking for stronger source signals as in field climbing. Therefore, we name our method \textbf{information-based source seeking}.

\textit{Our contributions:} We first introduce our information-based method in the centralized setting (Section \ref{sec:info-src-seeking}), where we formally define the loss function $L$ and the three-step outline of our method. We also provide theoretical justifications for the choice of loss function $L$, as minimizing $L$ improves the estimation and gets the sensors close to the source location. Then we extend our method to the distributed setting (Section \ref{sec:distr-src-seeking}), which makes our method scalable in the number of sensors. We show how distributed estimation and gradient calculation can be effectively done using simple consensus schemes. In Section \ref{sec:experiment} we conduct extensive numerical experiments to study the performance of our method in both centralized and distributed settings. The results show that our method outperforms field climbing methods and is flexible and robust in multiple aspects. Lastly, 
we implement our algorithm on small ground vehicles carrying light sensors to seek a light source in a dark room (Section  \ref{sec:real-lab-demo}), as shown in Fig. \ref{fig:teaser}. The hardware implementation further demonstrates the effectiveness of our algorithm.

The advantages of our methods can be summarized as
\begin{itemize}
    \item Compared with field climbing algorithms, numerical studies quantitatively demonstrate that our algorithm converges much faster to the source and performs more consistently over repeated trials. See \ref{section:sensor-swarms}. The algorithm takes advantage of the sensing capacity of multi-sensors in the sense that the performance improves as the number of sensors increases.
    \item Our algorithm, especially its gradient-guided movement, provides flexibility in handling various measurement models and picking different information metrics as loss functions. Section \ref{sec:Metric-Compare} confirms at least three applicable metrics.
    \item The algorithm is more robust to modeling error than the source localization with stationary sensors, see \ref{sec:robustness_modeling_error}. 
    \item Our distributed algorithm achieves comparable performance as the centralized algorithm. See \ref{sec:distr-experiment}. Moreover, results show that the distributed algorithm is more robust than the centralized implementation to the error in initial guesses and communication delay. See \ref{sec:sensitive-initial-guess}, \ref{sec:delay}.
\end{itemize}
This paper is the journal extension to our previous conference work \cite{Zhang2021SourceSeeking}. Compared to \cite{Zhang2021SourceSeeking}, besides including more technical details for the performance of the methods, this paper extends the centralized algorithm in \cite{Zhang2021SourceSeeking} to the distributed setting, where a central controller is absent, and the sensors decide their actions individually, as described in Section \ref{sec:distr-src-seeking}. Numerical results in Section \ref{sec:distr-experiment}$\sim$\ref{sec:delay} provide detailed comparisons between the centralized and distributed algorithm. Moreover, the hardware implementation migrates from the centralized robotics platform in the previous work to a distributed robotics platform in this paper, with the programming environment switching from ROS 1 to ROS 2. 
\modified{Lastly, although the algorithms are designed following a rigorous theoretical framework of information and optimization methods, the global theoretical convergence remains an open question because of the nonlinearity and nonconvexity associated with the problem. The proof of global convergence is, therefore, left for future work.
}

\subsection{Related Work}
\textit{Field climbing methods.} Field climbing originates from scientific studies of animal behavior in exploring nutrition or chemical concentration fields\cite{gazi_stability_2002, okubo_dynamical_1986}. The studies inspire source seeking methods that use field value measurements to estimate the field gradient and apply formation control to climb along the gradient\cite{ogren_cooperative_2004, bachmayer_vehicle_2002}. Gradient-free field climbing methods are also studied in the literature: \cite{moore_source_2010} maintains the sensors in a circular formation and uses measured field values as directional weights to guide the overall movement. Later works such as \cite{li_cooperative_2014, brinon-arranz_distributed_2016} extend the algorithms above to the distributed setting by applying consensus in the calculation of ascent direction. Single-agent field climbing is also possible if a gradient ascent control law is combined with proper zeroth-order gradient estimation algorithms, for example, using the control framework in \cite{baronov2011motion}.

The main differences between field climbing and our algorithm are that:
\begin{enumerate}
    \item Field climbing maximizes the source field value directly, while our algorithm exploits the Fisher information, an indicator of both estimation accuracy and source-sensor distance
    \item Field climbing does not assume a given measurement function, but ours does
    \item Field climbing algorithms often require a tight sensor formation for a stable field ascent, which only exploits local information; while our algorithm uses the sensors to collect global information for source localization. 
\end{enumerate}
We will show in Section \ref{sec:experiment} that, unlike field climbing, the sensors under our algorithm tend to spread out so that measurements contain more diverse information about the source location.

\textit{Source localization and optimal sensor placement.} If a measurement model is available, the source location can be estimated using methods including EKF\cite{morbidi_active_2013,martinez_optimal_2006,mavrommati2017real,abraham2018decentralized}, PF\cite{hoffmann_mobile_2010,bayat_optimal_2016}, and so on. It is even possible to reconstruct the entire source field  \cite{khodayi-mehr_distributed_2018,khodayi-mehr_model-based_2019}. Such estimation of source location is also known as source localization. 
An important problem in source localization is how to improve estimation via sensor movements. This problem is often investigated under the optimal sensor placement framework, in which sensors move to optimize various information metrics, including covariance \cite{yang_performance_2012,ke_zhou_multirobot_2011}, mutual information\cite{hoffmann_mobile_2010,charrow_information-theoretic_2015}, and Fisher information measures\cite{martinez_optimal_2006,khodayi-mehr_model-based_2019,moreno-salinas_optimal_2013,xu_optimal_2019,bayat_optimal_2016,bishop_optimal_2008,lee_active_2018,ponda_trajectory_2009,mavrommati2017real,abraham2018decentralized} such as the determinant of Fisher information, the largest eigenvalue of its inverse, and the trace of its inverse (D-, E-, and A-optimality criterion, respectively).

Our method also employs Fisher information but is different from the above works in 4 major ways:
\begin{enumerate}
    \item Most of the previous works focus on deriving closed-form solutions of optimal sensor placement for a particular type of measurement model, for example, RSS\cite{xu_optimal_2019}, pollutant diffusion\cite{bayat_optimal_2016}, and gamma camera\cite{lee_active_2018}. Our focus is not on providing closed-form solutions. Instead, we provide a gradient-based method applicable to a large class of range-based measurement models. This method also provides flexibility in choosing different information metrics as loss functions.
    \item Many previous works are about finding the optimal angular placement of the sensors at a fixed distance to the source or in a restricted area\cite{moreno-salinas_optimal_2013,xu_optimal_2019,ponda_trajectory_2009,martinez_optimal_2006}. In contrast, we allow the sensors to move freely and eventually reach the source.
    \item Some studies relax the restrictions on the sensor movement\cite{bishop_optimal_2008,lee_active_2018}. However, these methods produce a spiraling sensor movement which is inefficient for source seeking purposes. Our method does not produce such movement.
    \item The line of work by \cite{mavrommati2017real,abraham2018decentralized} also allows the sensors to move freely, but the control framework is very different from our work. The proposed ergodic control, which involves improving a control objective in integral form, is conceptually much more complex than our method--a gradient-based control. Ultimately, the ergodic control is designed for a very general set of mobile sensor applications, while our work is catered for the source seeking problem.
\end{enumerate}

\textit{Bayesian inference and optimization. }The recent advances in Bayesian learning have inspired many source seeking studies to adopt the Bayesian methods \cite{prabowo2020bayesian,sanchez2018probabilistic,benevento2020multi}. These studies view the environment as a field characterized by an (unknown) density function related to measurement and use a Gaussian Process or other likelihood models as a surrogate to guide the sensor movements for new measurement collections. 
The computation (for running the posterior update and Bayesian optimization) and memory (for storing historical measurements as in non-parametric Bayesian methods) demands of these methods are usually much higher than those of our method. Overall, Bayesian-based methods and our method are designed from different principles. A detailed comparison is left for future work.  

Lastly, it is worth mentioning that our work is significantly inspired by \cite{martinez_optimal_2006}, which studies the optimal sensor placement problem on a surveillance boundary. We leverage the ideas in \cite{martinez_optimal_2006} to introduce the Fisher information in our objective. However, we change the objective from maximizing the determinant of Fisher information to minimizing the trace of its inverse. We also generalize the measurement model to fit into our experiments and other real-world measurement settings. To handle such more general measurement models, we let the sensors only compute the gradient of the loss function rather than solving for its optimum at each time step.

\section{The Problem Statement}
Consider the problem of using a team of mobile sensors to find a source whose position is unknown. The source can be any object of interest, such as a lamp in a dark room. Mobile sensors are robotic vehicles that can measure the influence of the source on the environment, such as small ground vehicles carrying light sensors. 
	
Specifically, we use $q\in\mathbb{R}^k$ to denote the source position, and use $p_1,p_2,\ldots,p_m\in\mathbb{R}^k$ to denote the sensor positions, where $m$ is the number of sensors and $k$ is the spatial dimension. The measurement $y_i$ made by the $i$th sensor is modeled by
\begin{equation}\label{MeasurementFunctionDef}
	y_i = h_i(p_i,q) + \nu_i,
\end{equation}
where $h_i:\mathbb{R}^k\times\mathbb{R}^k\rightarrow \mathbb{R}$ is a known continuously differentiable function and $\nu_i$ is the measurement noise. The value of $h_i$ depends on the position of sensor $i$ but not on other sensors, and different sensors may have different $h_i$.  We let $\mathbf{y}=[y_1,y_2,\ldots,y_m]^\top$ denote the vector of all measurement values, let $\mathbf{\nu}=[\nu_1,\nu_2,\ldots,\nu_m]^\top$ be the noise vector,  and use $\mathbf{p}=
[p_1^\top, p_2^\top, \ldots, p_m^\top
]^\top$ to denote the joint location of all mobile sensors.  

We let $H:\mathbb{R}^{mk}\times\mathbb{R}^k\rightarrow\mathbb{R}^m$ be the mapping that describes the joint measurement made by all the sensors, i.e.,
\begin{equation}\label{eq:MeasurementFunc}
		H(\mathbf{p},q) = \begin{bmatrix}
			h_1(p_1,q)\\
			\vdots\\
			h_m(p_m,q)
		\end{bmatrix}.
\end{equation}
We define the global measurement model as
\begin{equation}\label{GlobalMeasurementModel}
\mathbf{y} = H(\mathbf{p},q)+\mathbf{\nu}.
\end{equation}
This model relates all measurements to the source location and all sensor locations.

The source seeking objective is to have at least one of the sensors get within $\epsilon_0$ distance to the source. In other words, we want to achieve
\begin{equation}\label{SourceSeekingGoal}
	\left\|p_i-q\right\|\leq \epsilon_0, ~ \exists i\in \{1,2,...,m\},
\end{equation} 
where $\epsilon_0$ is some small positive number.

\section{Information-based Source Seeking}\label{sec:info-src-seeking}
This section presents the key components of our information-based source seeking, including the three-step algorithmic flow and the loss function design. We assume a central server gathers all sensor locations $p$ and measurements $y$ and decides on new locations for all the sensors. We will remove this assumption in Section \ref{sec:distr-src-seeking}, where we introduce the distributed setting.

\subsection{The Algorithm}
Our information-based source seeking consists of three consecutive steps in one iteration: measurement, source location estimation, and sensor movement. The algorithm is illustrated in Figure~\ref{fig:algdiagram} and detailed in Algorithm \ref{alg:centralized-alg}. 

\begin{figure}[ht]
	\centering
	\vspace{5pt}
	\includegraphics[width=0.95\linewidth]{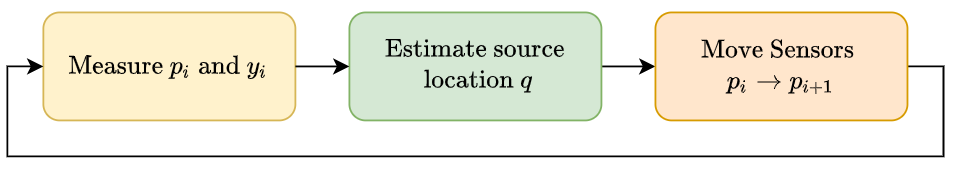}
	\vspace{3pt}
	\caption{The three consecutive steps in information-based source seeking.
	}
	\label{fig:algdiagram}
\end{figure}

\begin{algorithm}[b]
	\caption{Information-based Source Seeking}\label{alg:centralized-alg}
	\begin{algorithmic}[1]
		 \renewcommand{\algorithmicrequire}{\textbf{Input:}}
		\renewcommand{\algorithmicensure}{\textbf{Result:}}
		\REQUIRE
		Small constant $\epsilon_0>0$, the location estimator $E:(\mathbf{y},\mathbf{p})\mapsto \hat{q}$, motion planner $MP:(\tilde{p}_i(0),\tilde{p}_i(1),\ldots,\tilde{p}_i(T))\mapsto (u_i(1),\ldots,u_i(T))$.
		\REPEAT

		\STATE Get sensor locations $p_i$ from all mobile sensors $i$, forming $\mathbf{p}=\left[p_1^\top,p_2^\top,\ldots,p_m^\top\right]^\top$. 
		\STATE Get measurement $y_i$ from all mobile sensors $i$, forming $\mathbf{y}=[y_1,y_2,\ldots,y_m]^\top$.
		\STATE Estimate the location of the source by $\hat{q}\gets E(\mathbf{y},\mathbf{p})$.
		\STATE Set the initial waypoints  $\tilde{\mathbf{p}}(0)\gets \mathbf{p}$.
		\FOR{$t=1$ \TO $T$}
			\STATE $\tilde{\mathbf{p}}(t+1) \gets \tilde{\mathbf{p}}(t)-\alpha_t M_t \nabla_\mathbf{p} L|_{\tilde{\mathbf{p}}(t),\hat{q}}$.
		\ENDFOR
		\STATE Extract waypoints $(\tilde{p}_i(t))_{t=0}^T$ from $(\tilde{\mathbf{p}}(t))_{t=0}^T$ for mobile sensor $i$, and generate the control inputs	$(u_i(1),\ldots,,u_i(T))\gets MP(\tilde{p}_i(0),\ldots,\tilde{p}_i(T))$.
		\STATE Each mobile sensor $i$ executes control input $u_i(1)$.
		\UNTIL $\min_{i=1,2...m}\{||p_i-\hat{q}||\}\leq\epsilon_0$
	\end{algorithmic}
\end{algorithm}

\noindent\textbf{1) Measurement (Line 2--3)
}: The mobile sensors report their locations $\{p_i\}_{i=1}^m$ and latest measurements $\{y_i\}_{i=1}^m$ to the central server, forming the location vector $\mathbf{p}$ and the measurement vector $\mathbf{y}$.

\noindent\textbf{2) Source Location Estimation (Line 4)}: Using the information available (i.e., $\mathbf{p}$ and $\mathbf{y}$), the location estimation algorithm $E$ generates the estimated source location $\hat{q}$. 

The Extended Kalman Filter (EKF) is employed as the estimation algorithm in our implementations. We define $z:=[q^\top,v^\top]^\top$ as the source state, where $v$ is the velocity of the source. We do not assume any prior knowledge about the motion of the source, except that it evolves according to the second-order dynamics defined by
\begin{equation}\label{eq:zfz}
	\begin{aligned}
		z &:= \begin{bmatrix}
			q\\
			v
		\end{bmatrix} 	\\
		z^{+}&=f(z):= \begin{bmatrix}
			q+v\\
			v
		\end{bmatrix}. 	
	\end{aligned}
\end{equation}
Equation \eqref{eq:zfz} models the source moving at some unknown but constant velocity $v$. The applicability of this motion model is backed by extensive practical evidence in GPS problems \cite{brown1997introduction,wickert2018exploring} and moving target tracking\cite{shi2020sequential}.
In general, it is sensible to make a minimal assumption as above on the motion model in source localization.
While the actual dynamics of the source might deviate from the model in Equation \ref{eq:zfz} - for example, in the experiments shown in Figure \ref{fig:moving_source_seeking} - the localization performance typically remains satisfactory.
\modified{Analogous to Equation \ref{eq:zfz}, another prevalent assumption is to define the state $z:=q$ so that the transition model is $z^+ = f(z):=z$(see the third paragraph of section IV.B in \cite{mavrommati2017real}), but we observe that the estimation from this model often fails to converge to the true source location. In comparison, estimation with the model in \eqref{eq:zfz} converges to the source location much more consistently.} 

At each time step, the EKF takes in the sensor measurements and sensor locations and returns an estimate $\hat{z}$ of the true source state $z$. Please see \cite[Definition 3.1]{reif1999stochastic} for details about the EKF update.

\vspace{2pt}
\noindent\textbf{3) Sensor Movement (Line 5-10):}
The sensors move along the gradient descent directions of the following loss function
\begin{equation}\label{FIMLossFunction_ALG}
	L(\mathbf{p},q)=\tr\left[\big(\nabla_q H(\mathbf{p},q)\cdot\nabla_q H(\mathbf{p},q)^\top\big)^{-1}\right].
\end{equation}
We will explain the motivation of this loss function later in Section~\ref{sec:info-objective}, by relating it to the Fisher information matrix and the Cram\'{e}r-Rao lower bound \cite{cramir1946mathematical,rao1945information}. Sensor movement includes two parts: waypoint planning and waypoint tracking. 

\paragraph{Waypoint planning} In Lines 5--8, with the source location estimate $\hat{q}$ obtained in Line 4, we generate a set of waypoints $\tilde{\mathbf{p}}(0),\ldots, \tilde{\mathbf{p}}(T)$ by applying gradient descent on $L$ with respect to $\mathbf{p}$, with $q=\hat{q}$ fixed. In Line 7, $\alpha_t>0$ is the step size and $M_t \succ 0$ is a directional regularization matrix. 

\paragraph{Waypoint tracking} 
In Lines 9--10, after the joint waypoints $\tilde{\mathbf{p}}(0),\ldots,\tilde{\mathbf{p}}(T)$ have been calculated by the gradient descent on $L$, we extract the waypoints $\tilde{p}_i(0),\ldots,\tilde{p}_i(T)$ for each mobile sensor $i$, and use the motion planner $MP$ to calculate a sequence of control inputs $u_i(1),\ldots,u_i(T)$ and apply the control at the first instance $u_i(1)$ to the corresponding mobile sensor.\footnote{This is inspired by the framework of receding horizon control/model predictive control\cite{mayne2014model} which works well when the tracking trajectory is time-varying. For more information about motion planning, please refer to Appendix \ref{append:motion-planning}.}  
\begin{remark}\label{remark:complement}
	Note the algorithm aims to minimize $L(\mathbf{p},\hat{q})$ rather than $L(\mathbf{p}, q)$ ($\mathbf{p}$ being the decision variable), which raises the question of whether improving the former leads to a decrease in the value of the latter. Intuitively, it is expected that if $\hat{q}$ and $q$ are sufficiently close to each other, i.e., the estimation error is small, then changing $\mathbf{p}$ to reduce the value of $L(\mathbf{p},\hat{q})$ will make $L(\mathbf{p},q)$ decrease. These two algorithm components rely on each other to function as a whole.
\end{remark}

\subsection{Information-based Loss Function for Sensor Movement}\label{sec:info-objective}
One naive approach to determining sensor movement after Step 2) could be heading directly toward the estimated location. However, we observed that such movement often makes sensors cluster tightly, which reduces the diversity in measurements and is undesirable. As we later show in Fig. \ref{fig:gostraightline_II} and \ref{fig:gostraightline_III}, the resulting estimation and seeking are unstable. 
	
Creating a good condition for estimation through sensor movements is crucial for successful source seeking. Therefore, instead of following the haphazard approach above, our algorithm takes the gradient steps in Step 3). The idea is to let the sensors maximize some information metrics related to the source location. Specifically, we consider the Fisher information about the source location defined as follows: let $\mathbf{y},\mathbf{p}$ be the joint measurement and joint position of \textbf{all} the sensors, and $H$ be the corresponding global measurement function, as in \eqref{GlobalMeasurementModel}. We treat $\mathbf{y}$ as a random vector, with joint probability density $f(\mathbf{y};q)$ parameterized by the true source location $q$. The Fisher information($\textup{FIM}$) about $q$ is considered to measure the amount of information about the unknown $q$ contained in measurement $\mathbf{y}$. It is defined as 
	
\begin{equation}\label{eq:Fisher-info}
		\textup{FIM}_{\mathbf{y}}(q):= \text{Cov}[\nabla_{q}\log f(\mathbf{y};q)|q].
\end{equation}

Fisher information is closely connected with estimation quality, as it is a lower bound on the covariance of unbiased estimators. This result is formally known as the Cram\'er-Rao Lower Bound.
	
\begin{theorem}[Cram\'{e}r--Rao Lower Bound (CRLB)\cite{cramir1946mathematical,rao1945information}]\label{theorem:CRLB}
		For any unbiased estimator $\hat{q}$ of $q$, the following matrix inequality holds
		\begin{equation}\label{eq:crlb}
			\mathbb{E}[(\hat{q}-q)(\hat{q}-q)^\top]\succeq \textup{FIM}^{-1}.
		\end{equation}
	\end{theorem}
	
As a special case, if we assume $\mathbf{y}$ follows multi-variate Normal distribution described by
\begin{equation}
    \mathbf{y} = H(q,\mathbf{p})+\mathbf{\nu},~\mathbf{\nu}\sim \mathcal{N}(0,R),
\end{equation}
	then the Fisher information takes the form of
\begin{equation}
    \textup{FIM}_{\mathbf{y}}(q) = 4\nabla_{q} H(q,\mathbf{p})R^{-1}\nabla_{q} H(q,\mathbf{p})^\top.
\end{equation}

Under the further assumption that $\mathbf{\nu}$ is i.i.d. Normal($R\propto I$), there is 
	 \begin{equation}\label{eq:FIM_prop}
	     \textup{FIM} \propto\nabla_q H(\mathbf{p},q)\cdot\nabla_q H(\mathbf{p},q)^\top.
	 \end{equation}
	
Motivated by the form of $\textup{FIM}$ in \eqref{eq:FIM_prop}, we use 
\begin{equation}\label{FIMLossFunction_Motivation}
		L(\mathbf{p},q)=\tr\left[\big(\nabla_q H(\mathbf{p},q)\cdot\nabla_q H(\mathbf{p},q)^\top\big)^{-1}\right]
	\end{equation}
as the loss function. Note that by minimizing \eqref{FIMLossFunction_Motivation}, the CRLB is driven closer to the zero matrix. Furthermore, if $H$ satisfies Assumption \ref{assumption:measurement} below, then minimizing $L$ also results in approaching the source as stated in Propositions \ref{prop:reach-the-source} and \ref{prop:reach-the-source-2} afterward.

\begin{assumption}\label{assumption:measurement}
We make the following assumptions on the measurement functions $h_i$:
\begin{enumerate}
	\item {\bf Isotropic measurement}: The measurement values depend only on source-sensor distance, i.e.,
		\begin{equation}\label{eq:isotropic_measurement}
				h_i(p_i,q)=g_i(\|p_i-q\|)=g_i(r_i)
		\end{equation}
	for some function $g_i:(0,+\infty)\rightarrow\mathbb{R}$, where $r_i \coloneqq \|p_i-q\|$.
	\item {\bf Monotonicity}: The absolute value of the derivative of each function, $|g_i'(r)|$, is monotonically decreasing in $r$. \modified{Here, $g_i'(r)$ is the derivative of $g_i$ with respect to $r$.} 
	{\item{\bf Non-degeneracy}: 	Let $\hat{r}_i$ denote the unit direction vector from the source to the $i$-th mobile sensor, i.e.,
		$\hat{r}_i  = (p_i-q)/\|p_i-q\|$. We assume $\sum_{i=1}^m \hat{r}_i\hat{r}_i^\top\succ 0$.}
\end{enumerate}
		
\end{assumption}
	
	\begin{proposition}[Reaching the Source I]\label{prop:reach-the-source}

	Under Assumption~\ref{assumption:measurement}, we have
		\begin{equation}
		 \frac{1}{m\cdot L(\mathbf{p},q) }\leq \max_{i} \left|g_i'(r_i)\right|^2 ,
	    \end{equation}
		which means that if $L(\mathbf{p},q)$ decreases, then $\max_i|g_i'(r_i)|$ tends to increase and consequently $ \min_i r_i$ tends to decrease.
	\end{proposition}

 \begin{proposition}[Reaching the Source II]\label{prop:reach-the-source-2}
     Assume $|g_i'(r_i)|$ is \textbf{strictly} monotone in $r_i$ in Assumption \ref{assumption:measurement}. Then if the sequence of sensor locations $\{\mathbf{p}(t):t=1,2,3,...\}$ converges and satisfies $$\lim_{t\rightarrow\infty} L(\mathbf{p}(t),q) = \inf_{\mathbf{p}:p_{i}\neq q,\forall i} L(\mathbf{p},q),$$ then there must be $$\lim_{t\rightarrow \infty}\min_i ||p_i(t)-q|| = 0. $$
 \end{proposition}
	Please see Appendix \ref{append:reach-the-source} for the proof\modified{ of Propositions \ref{prop:reach-the-source} and \ref{prop:reach-the-source-2}. }
 
 In Proposition \ref{prop:reach-the-source-2}, we assume \modified{that} the $\mathbf{p}(t)$ sequence converges for technical convenience, but it is also reasonable to assume the sensors will reach a stationary formation when they have minimized $L$. Also, the infimum of $L$ is taken over the set $\{\mathbf{p}:p_{i}\neq q,\forall i\}$ because the $\textup{FIM}$ could be undefined if $p_i=q$ for some $i$.

	
\begin{remark}Assumption 1.3 is a regularity condition that ensures $L$ is well-defined, which is essential for analyzing the relationship between $L$ and the source-sensor distance. We enforce this assumption in the implementations by adding a positive definition matrix $\delta I$, with $\delta>0$ being a small constant, to the estimated $\textup{FIM}$ when calculating $L$ and its gradient. This technique proves to eliminate most of the pathological numerical behaviors while maintaining the soundness of source-seeking performance.\end{remark}

\modified{\begin{remark}
Despite the properties demonstrated in Propositions \ref{prop:reach-the-source} and \ref{prop:reach-the-source-2}, showing the full convergence of Algorithm \ref{alg:centralized-alg} remains an open theoretical question. We leave this question to future work, but a heuristic justification for the convergence behavior is that climbing the information field and getting more accurate estimations of the source location complement each other (see Remark \ref{remark:complement}), so the $L$ value decreases over time. Therefore, the sensors get close to the source by the propositions above.
\end{remark}}

\section{Distributed Information-based Source Seeking}\label{sec:distr-src-seeking}
In this section, we present our distributed source seeking algorithm. We no longer assume there is a central controller. Instead, the mobile sensors decide on their actions individually. Meanwhile, the sensors can communicate through a network to exchange information about measurements, sensor positions, and estimations. We model the communication network by a directed graph, denoted by $G=(\mathcal{N},\mathcal{E})$, where $\mathcal{N}$ is the set of nodes(sensors) and $\mathcal{E}$ is the set of edges. If an edge $(i,j)\in \mathcal{E}$, then sensor $j$ can directly receive information from sensor $i$. Let $\mathcal{N}_j=\{i|(i,j)\in \mathcal{E}\}\cup \{j\}=\{j,i_1,i_2,...,i_{m_j}\}$ be the in-coming neighborhood of $j$, including $j$ itself. Here $m_j$ denotes the number of its incoming neighbors except itself. For each $(i,j)\in\mathcal{E}$, there is a consensus weight $w_{ji}>0$ satisfying $\sum_{i\in\mathcal{N}_i} w_{ji}=1$ for all $j$ (see \cite[Assumption 2.1-2.3]{olshevsky2009convergence} for conditions on the consensus weights).

We assume the information available to sensor $j$ includes those from its measurements and direct communication. In particular, mobile sensor $j$ knows the following information:

\begin{enumerate}
	\item The measurement functions in its neighborhood, or equivalently the local joint measurement function $H_j$
	
	\begin{equation}\label{eq:Hj}
	    	H_j(q,p_{j},p_{i_1},...,p_{i_{m_j}}):=\begin{bmatrix}
		h_{j}(q,p_{j})\\
		h_{i_1}(q,p_{i_1})\\
		\vdots\\
		h_{i_{m_j}}(q,p_{i_{m_j}})\\
	\end{bmatrix}.
	\end{equation}
	\item The joint position and joint measurement of its neighborhood(in bold font). \begin{equation}
	    \mathbf{p}_j:=[p_{j}^\top,p_{i_1}^\top,...,p_{i_{m_j}}^\top]^\top,\mathbf{y}_j:=[y_{j}^\top,y_{i_1}^\top,...,y_{i_{m_j}}^\top]^\top.
	\end{equation}
\end{enumerate}

 We assume the local joint measurement model for $j$ is 
\begin{equation}
	\mathbf{y}_j = H_j(q,\mathbf{p}_j)+\mathbf{\nu}_j,
\end{equation}
where $\mathbf{\nu}_j=[\nu_{j}^\top,\nu_{j_2}^\top,...,\nu_{j_w}^\top]^\top$ is a random vector following multi-variate Normal distribution $\mathcal{N}(0,R_j)$.

We assume the sensors act cooperatively, that is their objective is still to have some members in the team to reach the source, as defined in \eqref{SourceSeekingGoal}.

\subsection{The Distributed Algorithm}

		\begin{algorithm}[ht]
			\caption{Distributed Information-based Source Seeking(for sensor $j$) }\label{alg:distri-alg}
			\begin{algorithmic}[1]
				\renewcommand{\algorithmicrequire}{\textbf{Input:}}
				\renewcommand{\algorithmicensure}{\textbf{Result:}}
				\REQUIRE
				Consensus weights $\{w_{jk}\}_{k\in \mathcal{N}_j}$, consensus-based location estimator $E$, motion planner $MP$, termination threshold $\epsilon_0$.

				\STATE Initialize $\hat{z}:=[\hat{q}_j,\hat{v}_j]^\top$, $\hat{F}_j$. 
				\REPEAT

				\STATE Take measurements and communicate with neighbors to obtain $\mathbf{p}_j,\mathbf{y}_j, \{\hat{{z}}_k, \hat{F}_{k}\}_{k\in \mathcal{N}_j}$.

				\STATE Estimate the source location using consensus-based Kalman Filter as shown in (\ref{eq:consensus_KF}), $$\hat{q}\gets E\left(\mathbf{y}_j,\mathbf{p}_j,\{\hat{{z}}_k, w_{jk}\}_{k\in \mathcal{N}_j}\right).$$
				\STATE Record $F_j$ in the previous step $F_j^-\gets F_j$
				\STATE Compute the new partial $\textup{FIM}$ $$F_j\gets\nabla_{q}h_j(p_j,\hat{q})\nabla_{q}h_j(p_j,\hat{q})^\top$$
				\STATE Estimate the global $\textup{FIM}$ $$\hat{F}_j\gets [\sum_{k\in \mathcal{N}_j} w_{jk}\hat{F}_k] + F_j-F_j^-$$
				\STATE Set the initial waypoint  $\tilde{p}_j(0)\gets p_j$.
				\STATE Calculate waypoints $\tilde{p}_j(1),...,\tilde{p}_j(T)$ by gradient descent
				$$
				\begin{aligned}
					&A_{j,t}=\nabla_{q}h_j(\tilde{p}_j(t),\hat{q})\\
					&\mathbf{d}_t = -2 \nabla_{p} A_{j,t}~\hat{F}_j^{-2} A_{j,t}\\
				&\tilde{p}_j(t+1) \gets \tilde{p}_j(t)-\alpha_t M_t \mathbf{d}_t\\
				\end{aligned}
				$$
				with $M_t\succ 0$ being regularization matrices.
				\STATE Generate the control actions	$(u_j(1),\ldots,,u_j(T))\gets MP(\tilde{p}_j(0),\ldots,\tilde{p}_j(T))$.
				\STATE Execute control action $u_j(1)$.
				
				\UNTIL $||p_j-\hat{q}||\leq\epsilon_0$
			\end{algorithmic}
		\end{algorithm}
		
In the distributed algorithm, each sensor $j$ maintains an estimate for the source location and velocity, denoted by $\hat{z}_j:=[\hat{q}_j^\top,\hat{v}_j^\top]^\top$. The sensor also maintains an estimate for the global Fisher information, denoted by $\hat{F}_j$. Our source seeking algorithm requires sensor $j$ to repeat three consecutive steps: information gathering, source location estimation, and movement. Algorithm \ref{alg:distri-alg} contains the detailed operations.

\textbf{1) Information Gathering (Line 3):} Sensor $j$ takes measurements and communicates with its neighbors to obtain the sensor positions $\mathbf{p}_j$ and measurements $\mathbf{y}_j$. Sensors also communicate their source location estimates $\hat{z}_j$ and their Fisher Information estimates $\hat{F}_j$, as required in the subsequent steps.
	
\textbf{2) Source Location Estimation (Line 4):} The gathered information is used to form an updated source location estimate $\hat{q}_j$, through \textbf{consensus-based estimation algorithms}, elaborated in Section~\ref{sec:distributed-est}. {For ease of exposition, we assume the source location estimator to be the consensus EKF (\ref{eq:consensus_KF}), though other estimators could also be adapted to the distributed setting following a similar procedure.} 
	
	\textbf{3) 
	Sensor Movement (Line 6-11):}
	{
	Sensor $j$ calculates the partial gradient of the loss function $L$ (\ref{FIMLossFunction_ALG}) through \textbf{Distributed Information Gradient Update} (Line 6-9), elaborated in Section ~\ref{sec:digu} .  }
	Then it calls a motion planner \textit{MP} to compute the corresponding control actions for the gradient steps and executes the first action (Line 10-11).
	
	The key differences between the distributed Algorithm \ref{alg:distri-alg} and Algorithm \ref{alg:centralized-alg} are:
	\begin{itemize}
	    \item Algorithm \ref{alg:centralized-alg} runs on the central controller. The distributed algorithm runs on individual sensors.
	    \item Comparing Step 1), the Algorithm \ref{alg:centralized-alg} knows the measurements and positions of all sensors. The distributed algorithm only knows those from the neighborhood. As a result, the Algorithm \ref{alg:distri-alg} modifies Steps 2) and 3) to accommodate the distributed setting.
\end{itemize}

\modified{\begin{remark}
Showing the full convergence of Algorithm \ref{alg:distri-alg} also remains an open question. We leave it to future work.
\end{remark}}
	
Next, we elaborate on how agents perform distributed source location estimation and distributed information gradient update in Algorithm \ref{alg:distri-alg}.
	
	\subsection{Distributed Source Location Estimation}\label{sec:distributed-est}
	In Step 2), each sensor must \textit{locally} form a source location estimate based on the information available to it, and one simple way is to run \textit{local} EKF based on its local information. For each sensor $j$, the \textit{local} EKF maintains two running variables, $\hat{z}_j$ and $P_j$, where  $$ \hat{z}_{j} := \begin{bmatrix}
		\hat{q}_{j}\\
		\hat{v}_{j}
	\end{bmatrix}$$ is sensor $j$'s estimate of the position-velocity vector of the source, and $P_j$ is a matrix quantifying the uncertainty in the estimate. These two variables are updated according to
	\begin{equation}\label{eq:local_KF}
	\begin{aligned}
		\hat{z}_{j} &\gets f(\hat{z}_{j})+K_j(\mathbf{y}_{j}-H_j(\hat{q}_{j},\mathbf{p}_{j}))\\
		P_j &\gets A_jP_jA_j^\top + Q_j-K_j(C_jP_jC_j^\top +R_j)K_j^\top
			\end{aligned}
	\end{equation}
where $Q_j,R_j$ are pre-defined constant positive matrices, $f$ is given in \eqref{eq:zfz}, and $H_j$ is given in \eqref{eq:Hj}. $Q_j$ and $R_j$ represent our prior belief of the covariances of process and measurement noises, therefore are usually the same among all $j$. The quantities $A_j,~C_j$, and $K_j$ are defined by
\begin{equation}\label{eq:ACK}
\begin{aligned}
	A_j&:= \nabla_z f(\hat{z}_j),		C_j:= \nabla_z H_j(\hat{q}_j)=\begin{bmatrix}
		\nabla_q H_j(\hat{q}_j,\mathbf{p}_{j})\\
		0
	\end{bmatrix}\\
	K_j &:= A_jP_jC_j^\top(C_jP_jC_j^\top+R_j)^{-1}
\end{aligned}
\end{equation}

The \textit{local} EKF may be one of the simplest estimation algorithms based on sensors' local information from neighboring sensors. However, its estimation is not great in our source seeking experiments, as we later show in \ref{sec:distr-experiment}. One of the main reasons is that the communication network is typically sparse, and each sensor is connected to only a few neighbors. Therefore, the measurements available to each sensor are few, and the resulting estimation is poor.

The \textbf{consensus EKF}, an extension to the Kalman Consensus Filter II algorithm \cite{kalman_consensus_filter}, addresses the limitations of \textit{local} EKF by introducing a consensus procedure among the estimates.
 The estimation $\hat{z}_j$ is updated with the following formula
	
		\begin{equation}\label{eq:consensus_KF}
			\begin{aligned}
				\hat{z}_{j}&\gets\sum_{i\in \mathcal{N}_j}  w_{ji} \hat{z}_{i}	\\	+&\left[f(\hat{z}_{j})+K_j(\mathbf{y}_{j}-H_j(\hat{z}_{j},\mathbf{p}_{j}))-\hat{z}_{j}\right]\\
			\end{aligned}
		\end{equation}
	with $w_{ji}\geq 0,~\sum_{i\in \mathcal{N}_j}w_{ji} = 1$. The update of $P_j$ remains the same as the \textit{local} EKF. 
	
	The update in \eqref{eq:consensus_KF} can be viewed as adding a consensus term $\sum  w_{ji} \hat{z}_{i}$ to the difference between \textit{local} EKF update and current $\hat{z}_j$. Meanwhile, the \textit{local} EKF can be viewed as a special case of \eqref{eq:consensus_KF} with $w_{jj}=1$ and $w_{ji}=0$ for $i\ne j$. We highlight here the advantage of the consensus EKF over the \textit{local} EKF:
	
	\begin{itemize}
		\item Since the neighbors' estimates use the measurements from possibly outside $\mathcal{N}_j$, sensor $j$ is indirectly exposed to the data outside of its immediate neighborhood through consensus. 
		\item With properly chosen consensus weights as specified in \cite[Assumption 2.1-2.3]{olshevsky2009convergence}, the sensors can reach consensus on estimates much quicker than the rate of change in \textit{local} EKF update. Effectively, the limiting estimate at consensus is an estimator that uses the measurement from all the sensors in the network. It serves as a target location for the sensor to coordinate their movement, and we observe much better convergence to the source using the consensus EKF. 
		
	\end{itemize}
Our consensus EKF algorithm is an extension of the Kalman-Consensus Filter II algorithm\cite[Section IV, eq. (20)]{kalman_consensus_filter} to the non-linear measurement setting. Some related studies also employ consensus in extended
Kalman-like filters\cite{long2012distributed,ding_collaborative_2012,katragadda_consensus_2014}, but they take consensus over both the mean estimate and covariance estimate, while our algorithm only takes consensus over the mean estimate to lower the communication load, which improves the responsiveness of source seeking in our real-time environment. This scheme is also backed by the theoretical result in \cite[Theorem 2]{kalman_consensus_filter}, that performing consensus on the mean estimate but not on the covariance estimate can still lead to convergence in estimation. It is also worth pointing out that although our consensus EKF shares similarities with the existing literature,
our consensus-based distributed gradient update (described later in
Section \ref{sec:digu}) is a novel contribution.


		


\subsection{Distributed Information Gradient Update}\label{sec:digu}

Similar to the centralized setting, our distributed algorithm requires the sensors to move to new locations along the negative gradient of the Fisher-Information-based loss function $L$ (\ref{FIMLossFunction_ALG}).
However, since $L$ depends on the global measurement function $H$ and all agents' positions $\mathbf{p}$, the challenge is for individual sensors to compute the correct partial gradients in a distributed fashion. Specifically, the partial gradient of $L$ can be shown to be
	\begin{equation}\label{eq:FIM_grad}
		\nabla_{p_j}L(\mathbf{p},q)=-2 \nabla_{p_j} A_j~\text{FIM}^{-2} A_j,
	\end{equation}
where $A_j = \nabla_{q}h_j(p_j,q)$. Note that both $A_j$ and $ \nabla_{p_j} A_j$ can be computed using the local information of sensor $j$ already. Therefore, the remaining problem is to estimate the value of $\text{FIM}$ that depends on global information $H$ and $\mathbf{p}$. We note that 
	\begin{equation}\label{eq:FIM_sum}
		\begin{aligned}
				\textup{FIM}\propto& \sum_{i=1}^m \nabla_q h_j(q,p_j) \nabla_q h_j(q,p_j)^\top,
		\end{aligned}
	\end{equation}
meaning that the Fisher information is proportional to the equally weighted sum of the rank-one matrices 
$$F_j:=\nabla_q h_j(q,p_j) \nabla_q h_j(q,p_j)^\top,$$ 
which we name as \textbf{partial $\textup{FIM}$}. Note each $F_j$ is again computable using the local information of sensor $j$.  Therefore, the sensors can recover the global Fisher information through consensus. Specifically, let each sensor maintain a running estimate $\hat{F}_j$ of the global $\text{FIM}$, and update $\hat{F}_j$ according to the following formula
	
\begin{equation}\label{eq:FIM_consensus}
		\hat{F}_j^+ = \left(\sum_{i\in \mathcal{N}_j} w_{ji}\hat{F}_i\right) + F_j-F_j^-.
	\end{equation}
	Here $\{w_{ji}\}_{i\in \mathcal{N}_j}$ are consensus weights, and $F_j^-$ denotes the partial $\text{FIM}$ value from the previous time step. To ensure $\hat{F}_j$ converges to an equally weighted sum of $F_j$'s, the consensus weights can be pre-defined to form a static doubly-stochastic consensus matrix. Alternatively, the update in \eqref{eq:FIM_consensus} can be implemented to follow the parallel two-pass algorithm \cite{olshevsky2009convergence}.
	
Finally, sensor $j$ calculates an estimated gradient descent direction by substituting the $\text{FIM}$ in \eqref{eq:FIM_grad} with $\hat{F}_j$, as shown in Line 9 of Algorithm \ref{alg:distri-alg}.

\section{Experiments}\label{sec:experiment}
\subsection{The advantage of Information-based Source Seeking}\label{sec:advantage}
	
This subsection compares Algorithm \ref{alg:centralized-alg} with field climbing methods via numerical experiments, showing the advantage of information-based source seeking. All experiments are implemented in a centralized way. The algorithm performance under actual robot dynamics is studied in simulations in Section \ref{section:Numerical-Experiments}. In the subsequent studies, we remove the robot dynamics in simulations to efficiently conduct repetitive trials and assume the sensors follow the gradient steps exactly. We study the influence of the number of sensors in Section \ref{section:sensor-swarms}, the difference of various information metrics in Section \ref{sec:Metric-Compare}, and the robustness to modeling error in Section \ref{sec:robustness_modeling_error}.

\subsubsection{Gazebo Numerical Experiments}\label{section:Numerical-Experiments}
\begin{figure*}
\vspace{5pt}
	\centering
		\begin{subfigure}[t]{0.30\textwidth}
			\centering
			\includegraphics[width=\textwidth]{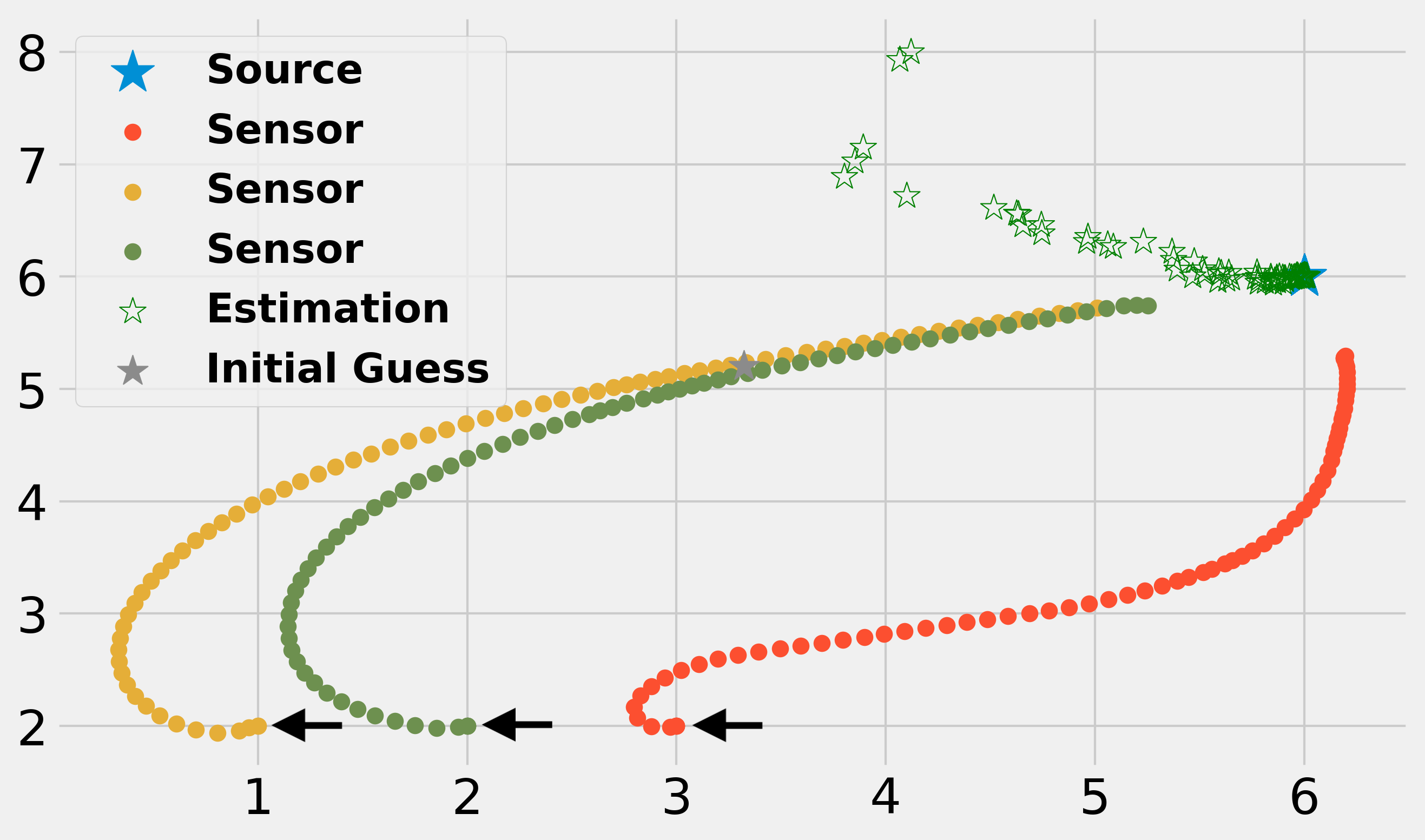}
			\caption{Information-based seeking\ref{alg:centralized-alg}}
			\label{fig:gradientdescenttraj_II}
		\end{subfigure}
		\begin{subfigure}[t]{0.30\textwidth}
			\centering
			\includegraphics[width=\textwidth]{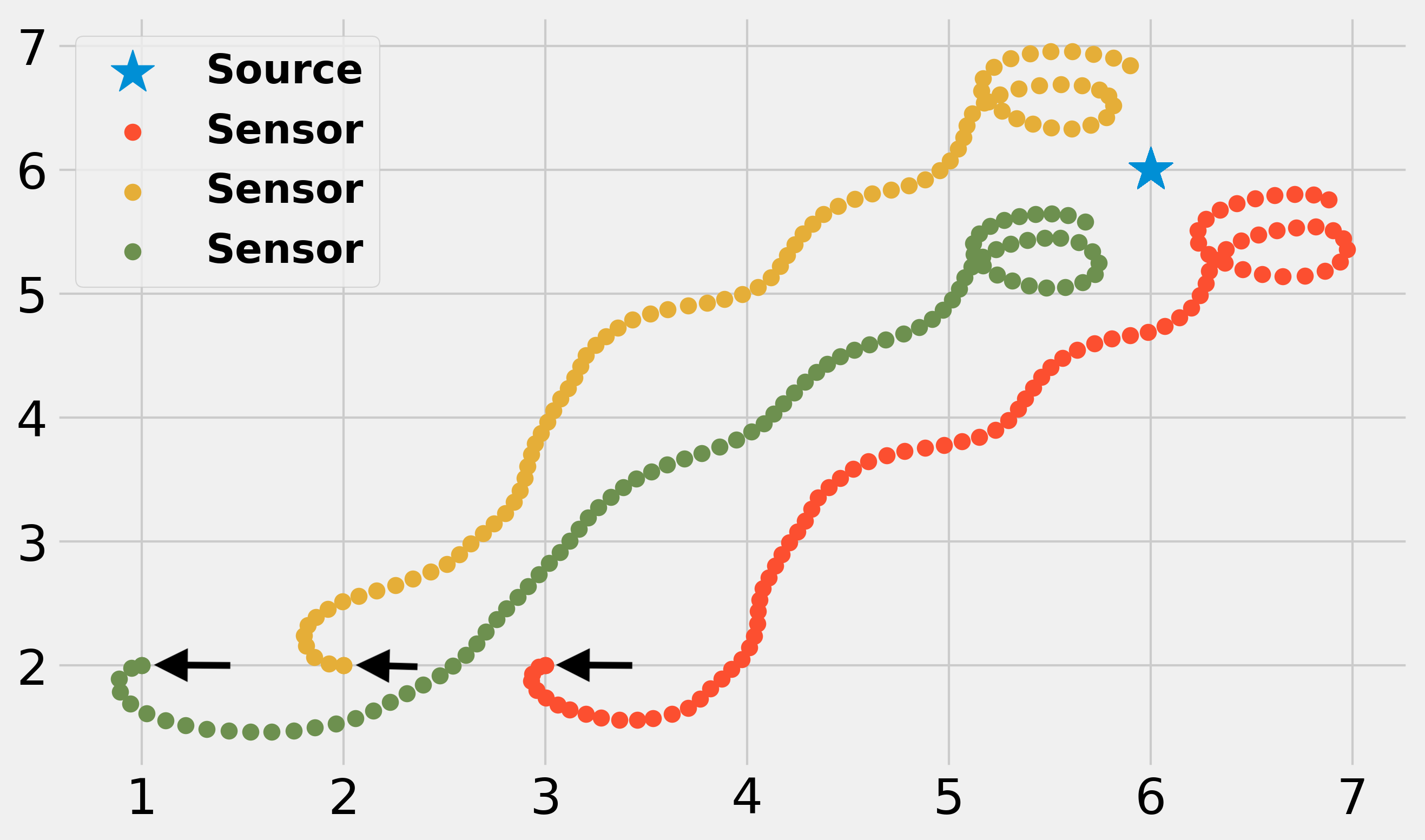}
			\caption{Field climbing \cite{moore_source_2010}.}
			\label{fig:hillclimbing_II}
		\end{subfigure}
		\begin{subfigure}[t]{0.30\textwidth}
			\centering
			\includegraphics[width=\textwidth]{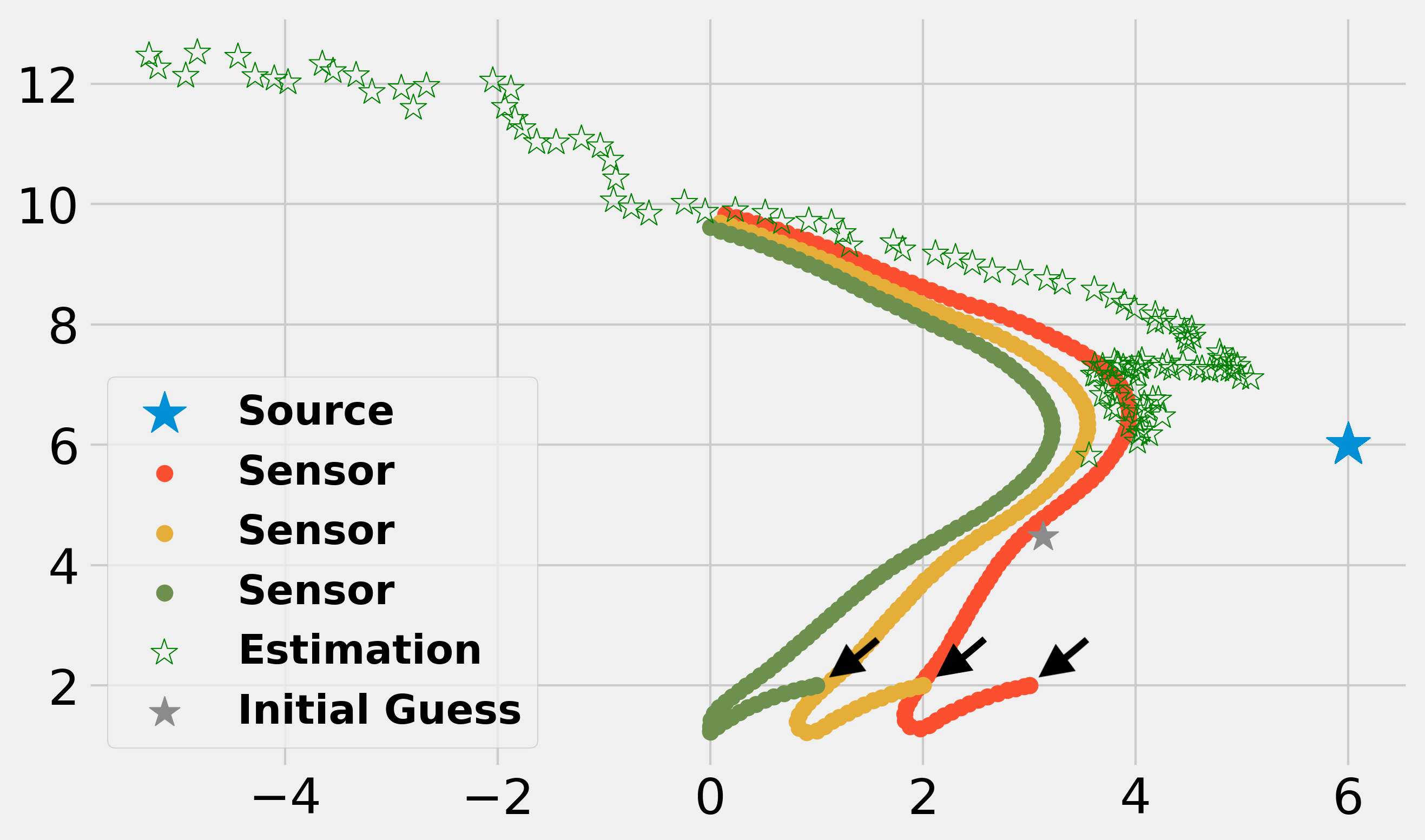}
			\caption{Following straight lines to the estimated location.}
			\label{fig:gostraightline_II}
		\end{subfigure}
		\vspace{-6pt}
		\caption{The general behaviors of seeking a stationary source. The black arrows indicate the starting locations of the sensors. The sensors in (b) do not make any estimations, since field climbing only uses measured signal strength to guide sensor movements.}
		\label{fig:static_source_seeking}
		\vspace{-3pt}
\end{figure*}
\begin{figure*}
	\begin{subfigure}[t]{0.3\linewidth}
		\centering
		\includegraphics[width=\linewidth]{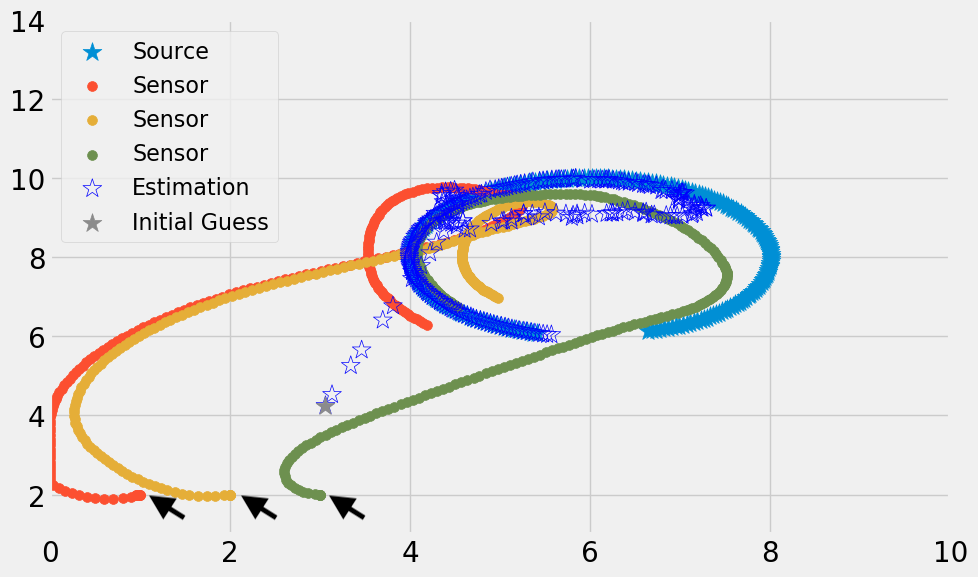}
		\caption{Information-based seeking\ref{alg:centralized-alg}}
		\label{fig:gradientdescenttraj_III}
	\end{subfigure}
	\begin{subfigure}[t]{0.3\linewidth}
		\centering
		\includegraphics[width=\linewidth]{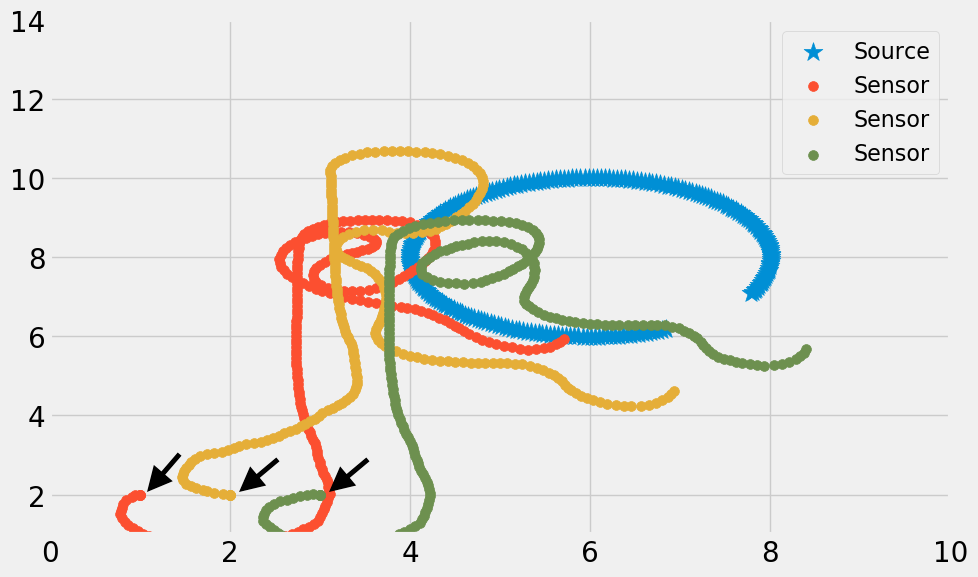}
		\caption{Field climbing \cite{moore_source_2010}.}
		\label{fig:hillclimbing_III}
	\end{subfigure}
	\centering
	\begin{subfigure}[t]{0.3\linewidth}
		\centering
		\includegraphics[width=\linewidth]{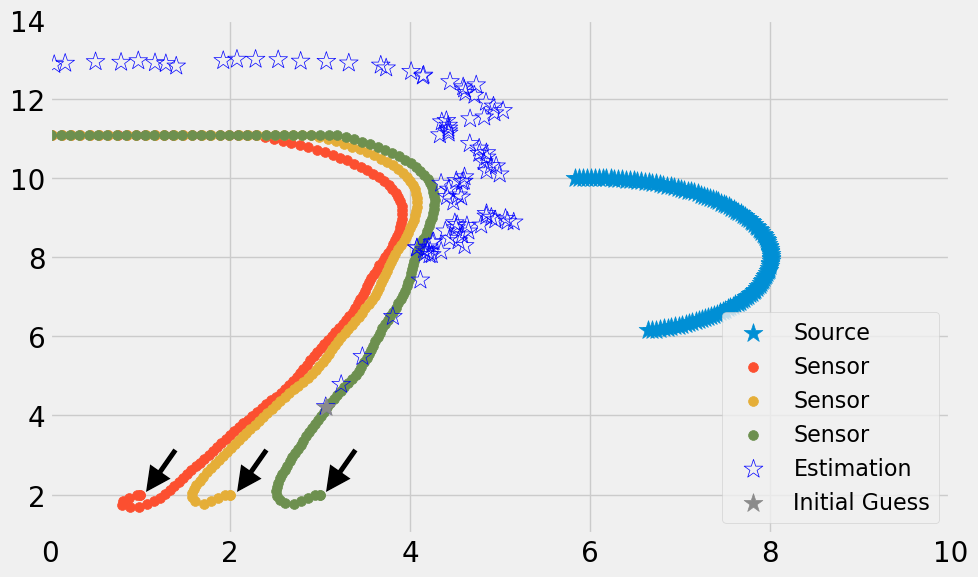}
		\caption{Following straight lines to the estimated location.}
		\label{fig:gostraightline_III}
	\end{subfigure}
	\vspace{-6pt}
	\caption{The general behaviors of seeking a moving source. The black arrows indicate the starting locations of the sensors.}
	\label{fig:moving_source_seeking}
	\vspace{-10pt}
\end{figure*}

The following numerical experiments are carried out using the Gazebo simulation toolbox\cite{koenig_design_2004}, with virtual mobile sensors simulating the same dynamics as the actual robots. We generate simulated measurement values of the sensors by
\begin{equation}\label{eq:experiment-measurement-model}
    y_i = 1 / r_i^2+\nu_i,
\end{equation}
with $\nu_i$ drawn independently from $\mathcal{N}(0,0.01)$. The measurement function $h_i(p_i,q)=1/||p_i-q||^2 = 1/r_i^2$ is given to the EKF for estimation.

\textbf{\textit{Stationary Source: }}
In the first set of simulations, we use three mobile sensors to seek a stationary source. The source is fixed at position $(6.0,6.0)$, while the mobile sensors are initially placed at $(1.0,2.0),(2.0,2.0),(3.0,2.0)$. The initial guess of source location given to the EKF is $(3.0,4.0)$. The terminal condition threshold $\epsilon_0=0.5$.
We compare the convergence to a stationary source among three algorithms: (a) our algorithm;
(b) the field climbing algorithm introduced by \cite{moore_source_2010} that only maximizes measured signal strength; (c) following straight lines to the estimated location. 
 (c) is included to show the importance of exploiting Fisher information in obtaining accurate estimation. The results are displayed in Figure~\ref{fig:static_source_seeking}.

First, notice that the straight-line algorithm fails to converge to the source, as shown in Fig. \ref{fig:gostraightline_II}. We suspect the reason is sensors cluster together quickly as they move to the (same) estimated location and cannot provide sufficiently rich, diverse measurements for a reasonable estimation. Consequently, the estimate gradually deviates from the source location, as do the sensors.
On the other hand, if taking a trajectory that improves the Fisher information, the sensors cover the space more thoroughly, resulting in a stable decrease in the estimation error and the final success of reaching the source, as shown in Figure~\ref{fig:gradientdescenttraj_II}.

Comparing Fig. \ref{fig:gradientdescenttraj_II} and \ref{fig:hillclimbing_II}, note that sensors using our algorithm first spread out to estimate the source location better and then converge to the source, whereas sensors doing field climbing maintain a tight formation while steadily approaching the source. Since we use constant rather than diminishing step sizes for Fig. \ref{fig:hillclimbing_II}, the virtual robots do not stop completely near the source and perform a looping behavior. Although our algorithm and the field climbing algorithm \cite{moore_source_2010} are both successful with a stationary source, our algorithm consistently converges faster over repetitive trials, as is shown later in Section \ref{section:sensor-swarms}. 

\textbf{\textit{Moving Source: }}
In this set of experiments, all parameters are kept the same as the stationary case, except that now the source moves in a circular motion with constant speed. See Figures \ref{fig:gradientdescenttraj_III},  \ref{fig:hillclimbing_III}, \ref{fig:gostraightline_III}. 
Note the straight-line algorithm again leads to a sensor formation that causes the estimation to deviate from the actual source location. Both our algorithm and the field climbing algorithm\cite{moore_source_2010} successfully get close to the source. However, the field climbing method exhibits unnecessary irregular motion when the sensors are near the source. We suspect that the field climbing direction becomes very sensitive to the source movement as the sensors get close to the source, which leads to such motion. In comparison, sensors following our algorithm trace much more stable paths. 

\subsubsection{Performance with Sensor Swarms}\label{section:sensor-swarms}

\begin{figure*}
\vspace{5pt}
	\centering
	\begin{subfigure}[b]{0.30\linewidth}
		\centering
		\includegraphics[width=\linewidth]{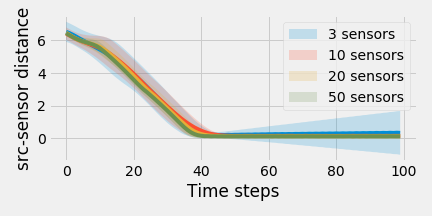}
		\caption{Information-based seeking\ref{alg:centralized-alg}}
	\end{subfigure}
	\centering
	\begin{subfigure}[b]{0.30\linewidth}
		\centering
		\includegraphics[width=\linewidth]{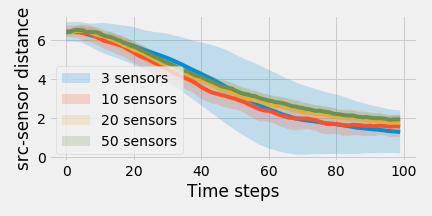}
		\caption{Climbing estimated gradient\cite{ogren_cooperative_2004}.}
	\end{subfigure}
		\begin{subfigure}[b]{0.30\linewidth}
		\centering
		\includegraphics[width=\linewidth]{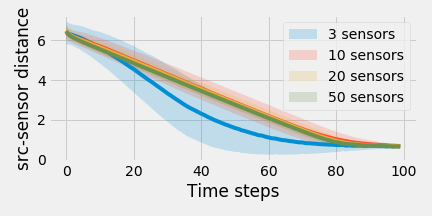}
		\caption{Circular formation field climbing \cite{moore_source_2010}.}
	\end{subfigure}	\centering
	\vspace{-3pt}
	\caption{The time evolution of the source-sensor distance. 
	}
	\label{fig:Dist-Plots}
	\vspace{-8pt}
\end{figure*}

\begin{figure*}
	\centering
	\begin{subfigure}[t]{0.175\linewidth}
		\centering
		\includegraphics[width=\linewidth]{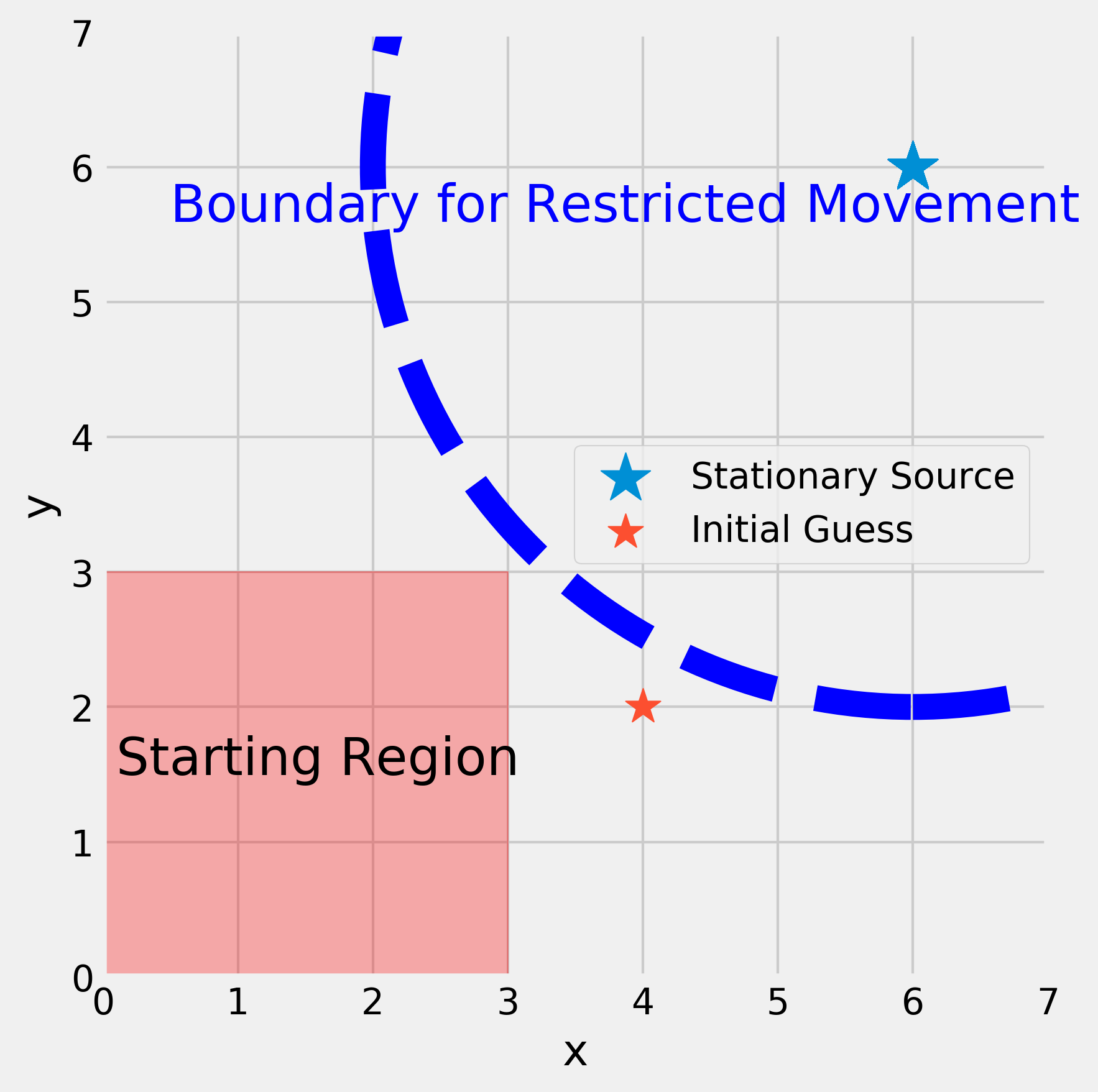}
		\caption{Experiment Setting}
		\label{fig:experimentsetting}
	\end{subfigure}
	\begin{subfigure}[t]{0.265\linewidth}
		\centering
		\includegraphics[width=\linewidth]{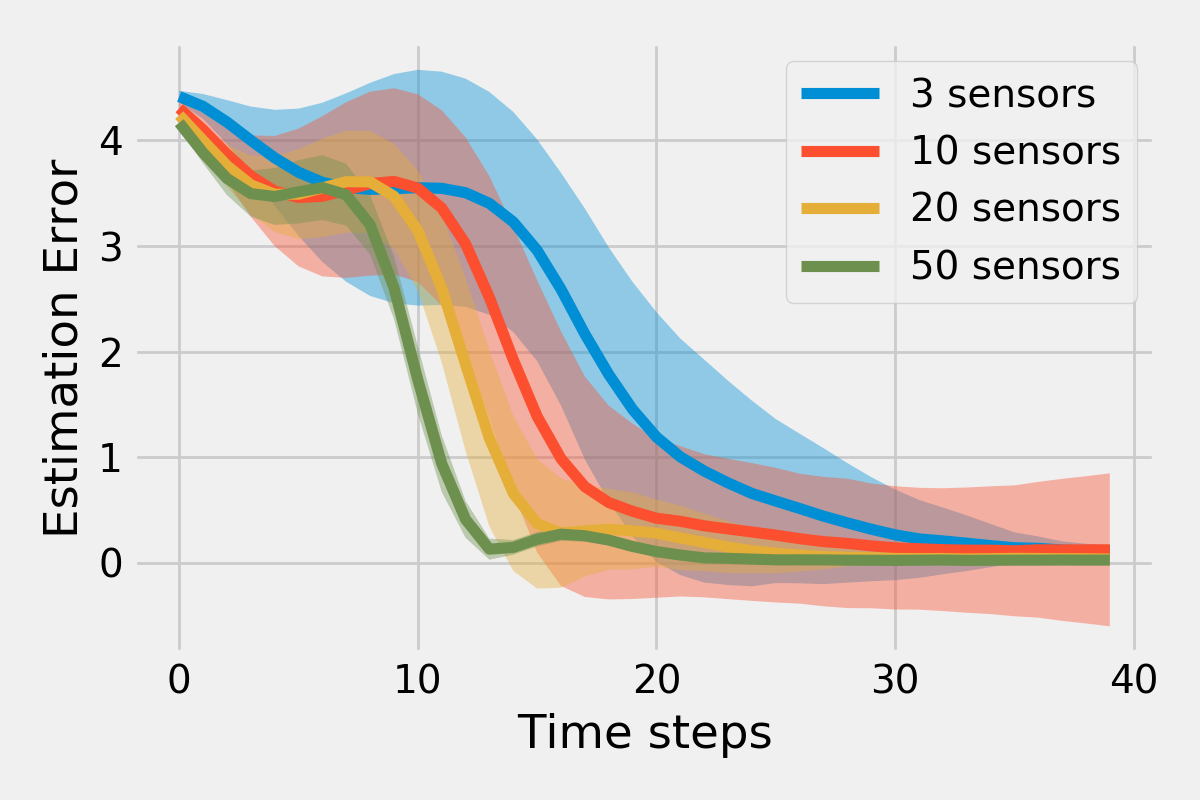}
		\caption{Sensors move freely. }
			\label{fig:free_sensors}
	\end{subfigure}
	\begin{subfigure}[t]{0.265\linewidth}
		\centering
		\includegraphics[width=\linewidth]{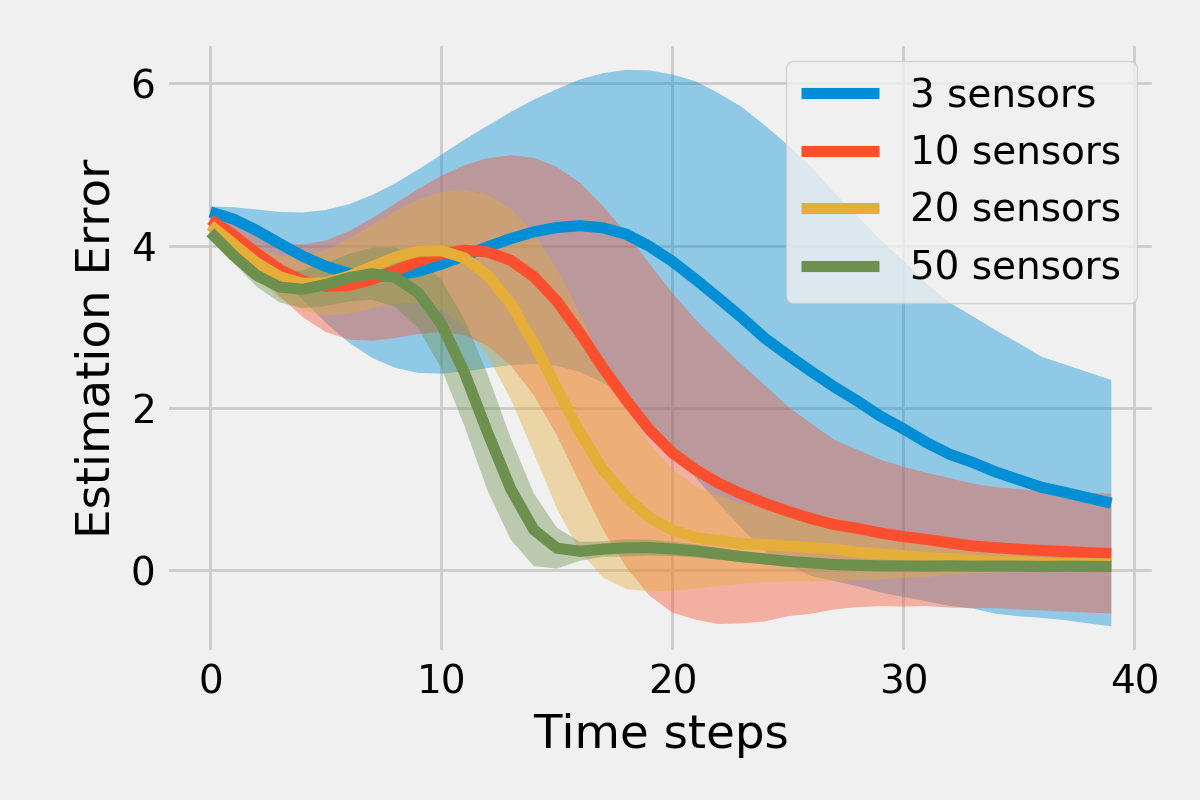}
		\caption{Sensors stay outside the boundary.}
	\end{subfigure}
	\centering
	\begin{subfigure}[t]{0.265\linewidth}
		\centering
		\includegraphics[width=\linewidth]{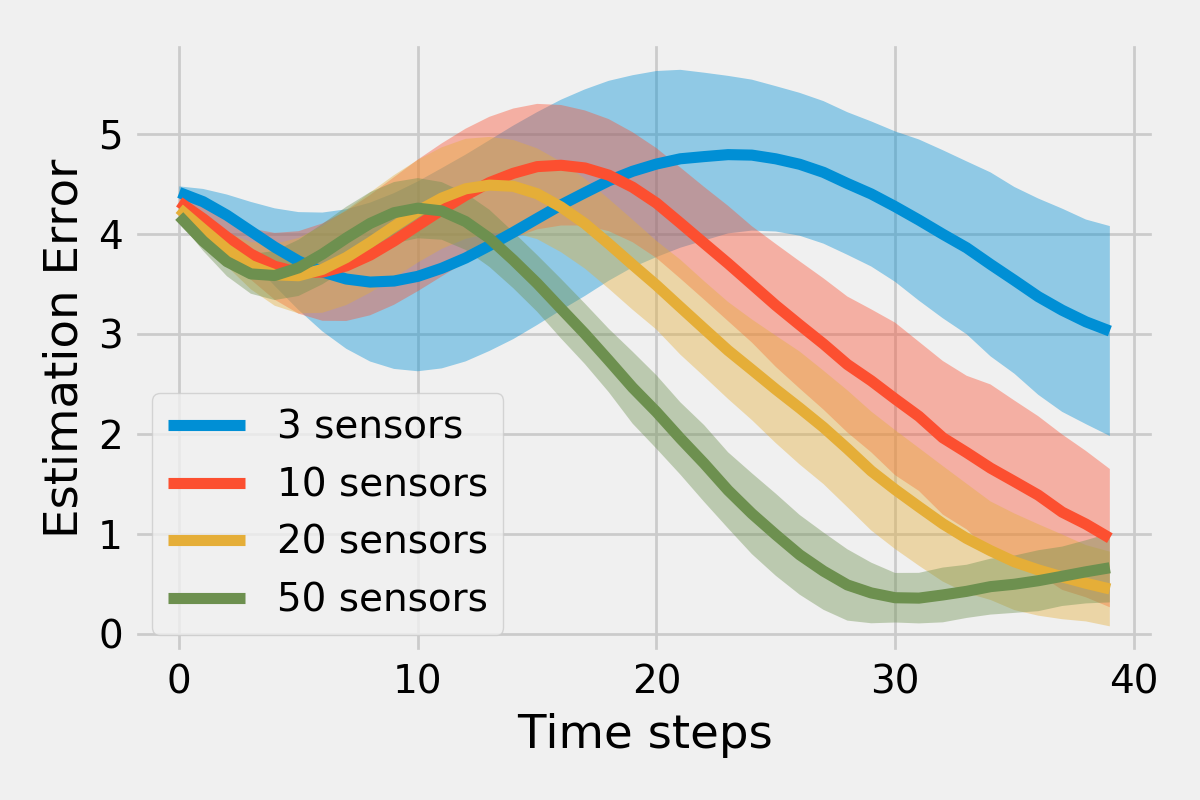}
		\caption{Stationary sensors.}
	\end{subfigure}
	\caption{The time evolution of source location estimation error. In Fig. \ref{fig:free_sensors}, the variance of the 50 sensor curve after 10 steps is small but not zero.
	}
	\label{fig:Error-Plots}
\end{figure*}

We now investigate the influence of the number of mobile sensors on the algorithm performance. We perform repetitive simulations to study the convergence rate of source-sensor distances and the evolution of \modified{the} estimation error. We use a stationary source fixed at position $(6.0,6.0)$ and randomize the initial sensor locations following uniform distribution in a $3.0$-by-$3.0$ rectangle, as shown in Fig. \ref{fig:experimentsetting}. 

The experiments in Fig. \ref{fig:Dist-Plots} study the convergence rate of \textit{the sensors' distances to the source}. Our algorithm is compared with two field-climbing methods\cite{ogren_cooperative_2004}\cite{moore_source_2010}. The performances of 3-, 10-, 20-, and 50-sensor teams are plotted. The solid curves are the average values over 100 repetitive experiments. The color bands indicate one standard deviation. The results show that our algorithm converges much faster than the others. 
In addition, increasing the number of sensors substantially reduces all algorithms' variance while not affecting the convergence time. Because the speed of sensors has an upper bound regardless of the number of sensors, the overall convergence rate of source-sensor distance is limited. Nevertheless, having more sensors provides more measurements, which reduces the variance and results in more consistent, stable trials. In particular, our algorithm benefits the most from having more sensors since more measurements contribute to better estimation, thus, faster convergence with smaller variance. 

To further test the above conjecture that sensors moving \textit{freely} and \textit{in the direction of improving Fisher information} brings richer measurement and therefore enhances the estimation, we conduct the following three experiments:
\begin{enumerate}
    \item The sensors move freely, guided by our proposed algorithm.
    \item The sensors are restricted to staying outside a radius of $3.0$ from the source, performing projected gradient descent of our proposed loss function at the boundary.
    \item The sensors do not move but only perform location estimation.
\end{enumerate}
The results in Fig. \ref{fig:Error-Plots} show that using more sensors and allowing sensors to move freely can both lead to a faster decline in the estimation error and a smaller variance. 

\subsubsection{Comparison of Different Information Metrics}\label{sec:Metric-Compare}
Although we specify the loss function as the A-optimality criterion in our algorithm, one can, in principle, replace it with other alternatives. In the following experiments, we test our algorithm's performance with four different loss functions: i) $\tr(\textup{FIM}^{-1})$; ii) $\lambda_{\max}(\textup{FIM}^{-1})$; iii) $-\log \det (\textup{FIM})$, and iv) $\tr P$ where $P$ is the posterior covariance of the EKF defined by \cite[Equation (15)]{yang_performance_2012},
   $ P = (\nabla_q H R^{-1}\nabla_q H^\top+P_0^{-1})^{-1},$
with $P_0$ being the current estimation covariance and $R$ being the measurement noise parameter of the EKF.

All loss functions are tested with a stationary source and three freely-moving sensors for 100 repetitive trials. Each trial is initialized randomly in the same way as in the previous subsection. Fig. \ref{fig:Metric-Compare-Traj} shows the gradient descent trajectories for the tested metrics. Figure~\ref{fig:Metric-Compare} shows that all the tested metrics except $-\log \det(\textup{FIM})$ yield relatively good performance when using our algorithm, with $\tr(\textup{FIM}^{-1})$ and $\lambda_{\max}(\textup{FIM}^{-1})$ achieving a better balance in convergence and estimation than the covariance metric $\tr(P)$. These results confirm the generality of our algorithms.

\begin{remark}[Rationale For Metrics]
Assume the $\textup{FIM}$ is a $k\times k$ matrix. Since the $\textup{FIM}$ is always positive semi-definite,  it has non-negative eigenvalues $\lambda_1,\lambda_2,...,\lambda_k$. The goal of minimizing CRLB is to make the matrix $\textup{FIM}^{-1}$ as close to the zero matrix as possible, which means $\frac{1}{\lambda_1},...,\frac{1}{\lambda_k}$ should all be as close to zero as possible. Therefore, $\tr{(\textup{FIM}^{-1})} = \sum_{i=1}^k \frac{1}{\lambda_i}$ and $\lambda_{\max}(\textup{FIM}^{-1})=\max_{i=1,...,k}\frac{1}{\lambda_i}$ are both reasonable metrics to minimize. Similarly, the third metric $-\log\det(\textup{FIM}) = \sum_{i=1}^k \log(\frac{1}{\lambda_i})$ also encourages an overall reduction of $\frac{1}{\lambda_i}$ values, but it may not be as good as the previous two since and its value can be made arbitrarily low by setting one single $\frac{1}{\lambda_i}$ is very close to zero while all others remain large. The covariance metric is essentially $\tr((\textup{FIM}+P_0^{-1})^{-1})$, which can be thought of as a regularized version of the first metric. 
\end{remark}

\begin{remark}[Geometric Properties]\label{remark:metric-geo-prop}
Although the trajectories in Fig. \ref{fig:Metric-Compare-Traj} may look very different, they share some important common geometric properties that make sense intuitively. Specifically, they all encourage a separation among the sensors at some point and create an overall trend of approaching the source. We need to look at the explicit form of FIM to understand these behaviors. In the proof in Appendix B, we show that $$\textup{FIM} = \sum_{i=1}^m |g_i'(r_i)|^2 \hat{r}_i\hat{r}_i^\top,$$ where $m$ is the number of sensors, $g_i$ is the measurement function for the $i$'th sensor, $r_i$ is the distance between the $i$'th sensor and the source, and $\hat{r}_i$ is the unit vector pointing from the source to the $i$'th sensor. Having $\textup{FIM}^{-1}$ close to zero means $\textup{FIM}$ itself is `very positive definite'. And to achieve that, either 1) $|g_i'(r_i)|$ are large or 2) $\hat{r}_i\hat{r}_i^\top$ are `very linearly-independent', or both. Under the assumption that $|g_i'(r_i)|$ becomes larger as $r_i$ decreases, item 1) contributes to the overall trend of approaching the sensor, which seeks a stronger signal-to-noise ratio. Further, item 2) encourages the unit vectors $\hat{r}_i$ to point towards different directions, making the sensors separate from one another and creating different line-of-sights.

The covariance is the one that stands out the most among the metrics as it exhibits some follow-straight-line behavior at the beginning and separates only when close to the source. This phenomenon can be explained by the dominance battle between $P_0^{-1}$ and $\nabla H R^{-1} \nabla H^\top$, as discussed in detail in Appendix \ref{append:covariance}.
\end{remark}

\begin{figure}
\centering
\begin{subfigure}[t]{0.45\linewidth}
		\centering
		\includegraphics[width=\linewidth]{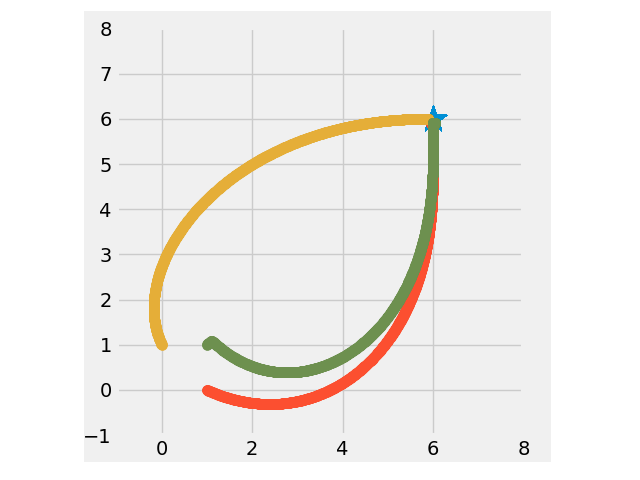}
		\caption{$\tr(\textup{FIM}^{-1})$}
	\end{subfigure}
	\begin{subfigure}[t]{0.45\linewidth}
		\centering
		\includegraphics[width=\linewidth]{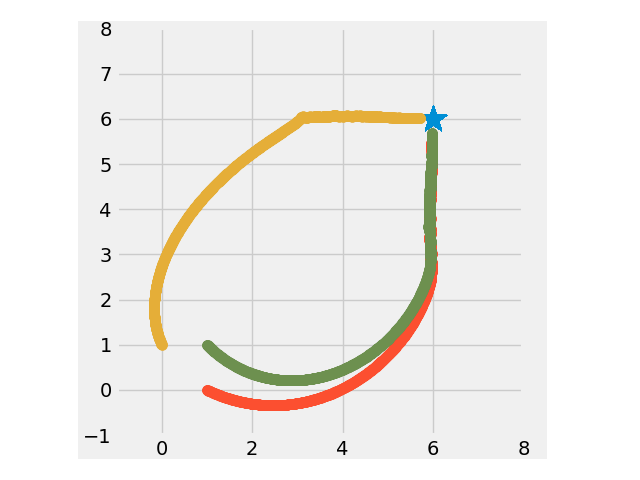}
		\caption{$\lambda_{\max}(\textup{FIM}^{-1})$}
	\end{subfigure}
	\begin{subfigure}[t]{0.45\linewidth}
		\centering
		\includegraphics[width=\linewidth]{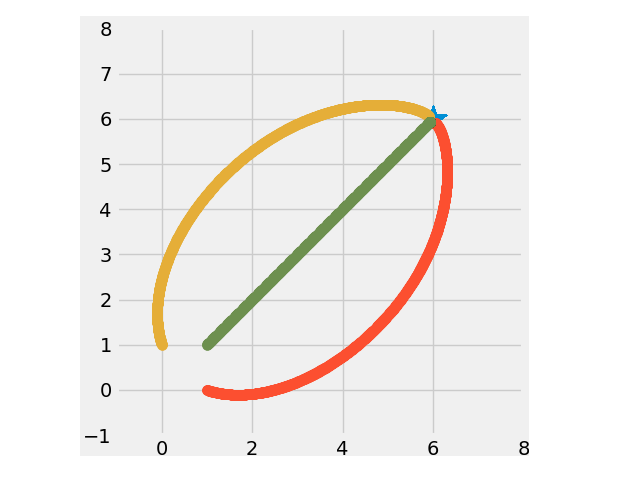}
		\caption{$-\log\det(\textup{FIM})$}
	\end{subfigure}
	\begin{subfigure}[t]{0.45\linewidth}
		\centering
		\includegraphics[width=\linewidth]{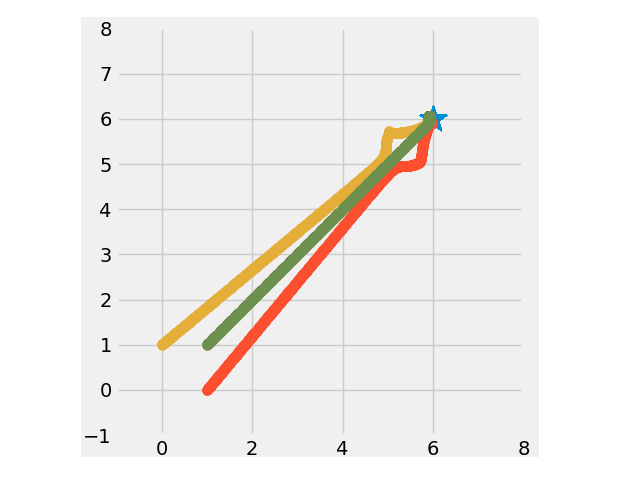}
		\caption{Covariance}\label{subfig:covariance}
	\end{subfigure}
	\caption{Assuming the source location is known, the above are the gradient descent trajectories of the tested information metrics.}
	\label{fig:Metric-Compare-Traj}
	\vspace{-15pt}
\end{figure}

\begin{figure}
    \centering
    \includegraphics[width=0.8\linewidth]{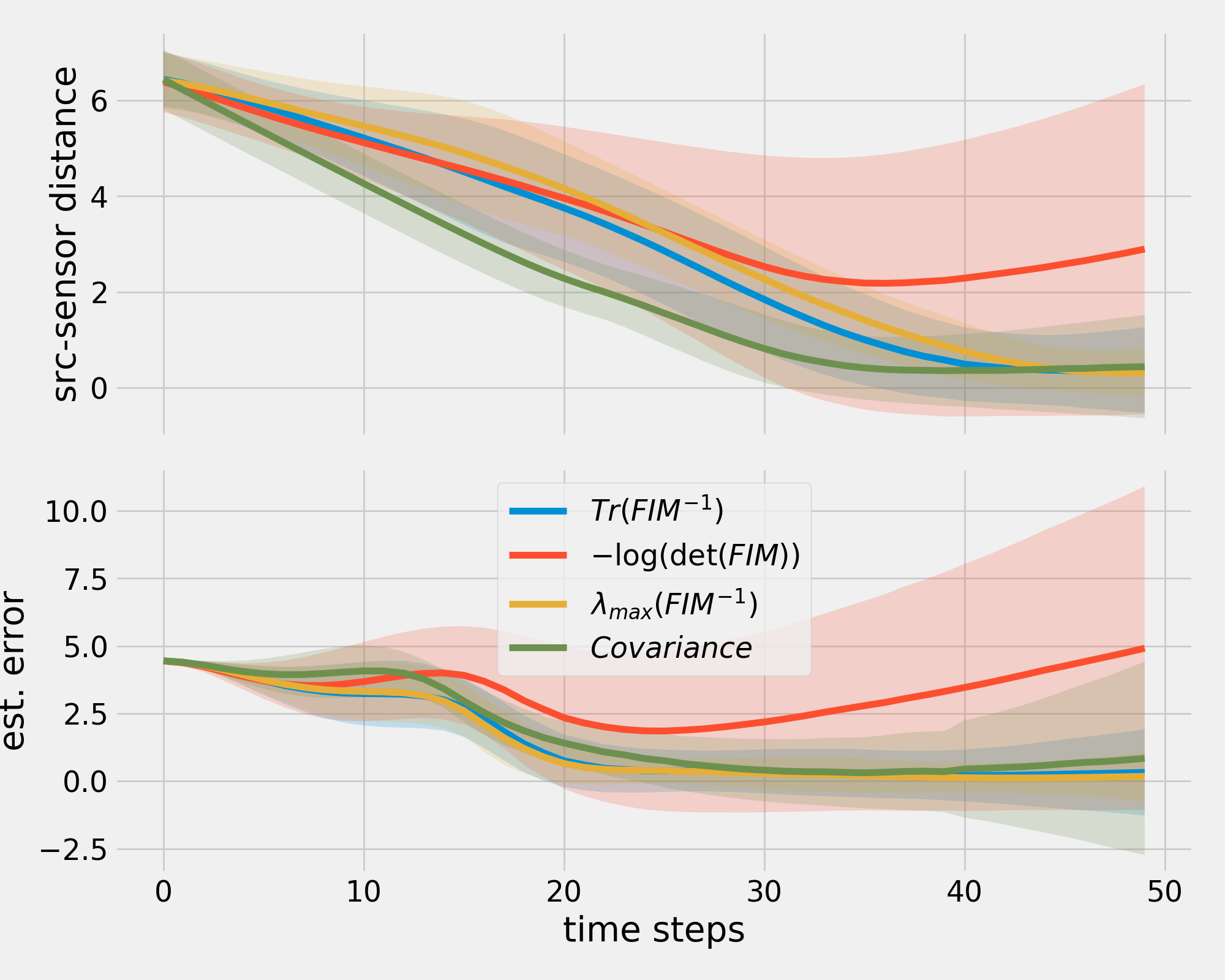}
    \caption{Comparison of the algorithm performance using different information metrics.}
    \label{fig:Metric-Compare}
\end{figure}

\subsubsection{Robustness to Measurement Modeling Error}\label{sec:robustness_modeling_error}
\begin{figure}
\centering
	\begin{subfigure}[t]{0.45\linewidth}
		\centering
		\includegraphics[width=\linewidth]{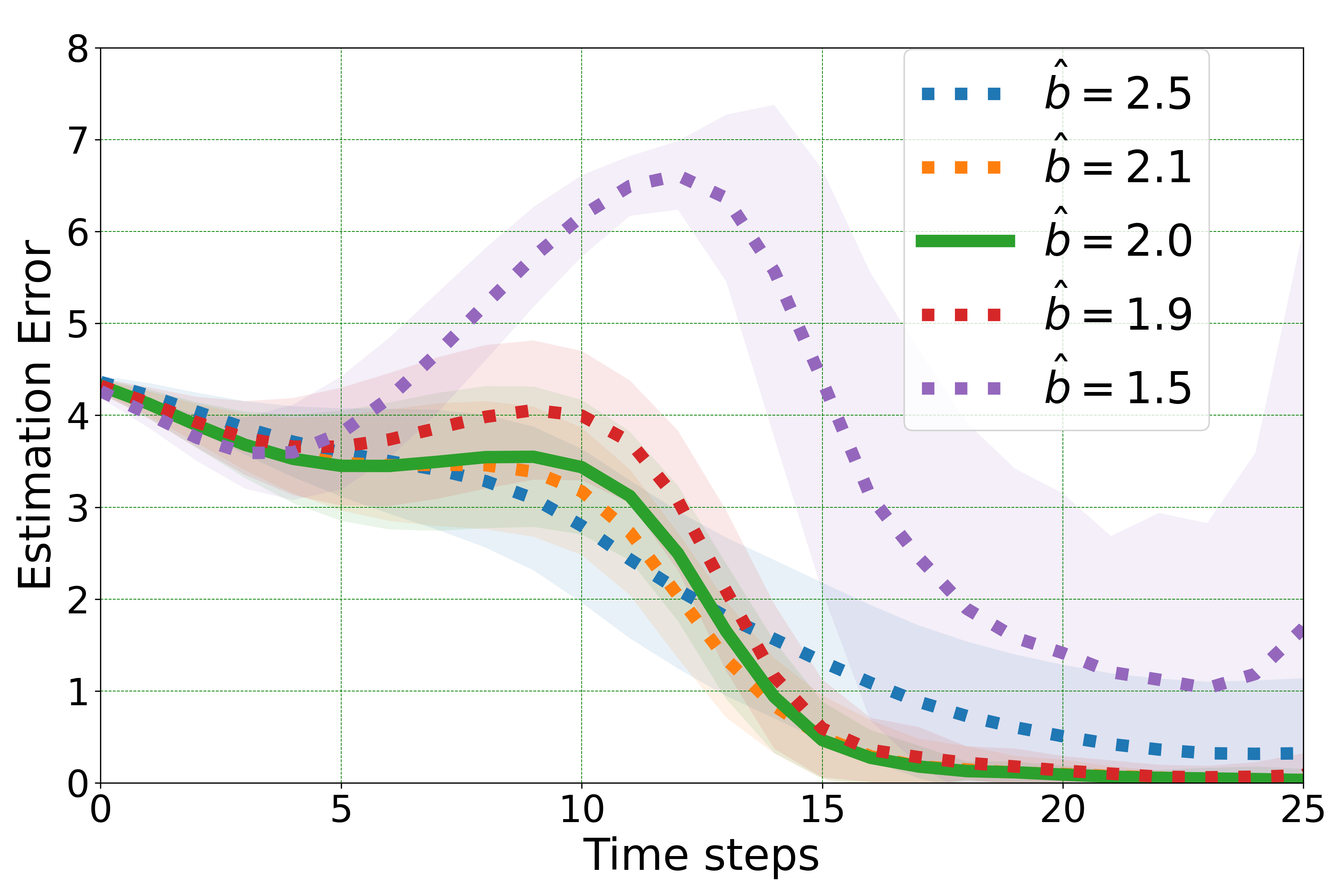}
		\caption{Our Algorithm}\label{fig:moving-robust}
	\end{subfigure}
	\hfil
	\begin{subfigure}[t]{0.45\linewidth}
		\centering
		\includegraphics[width=\linewidth]{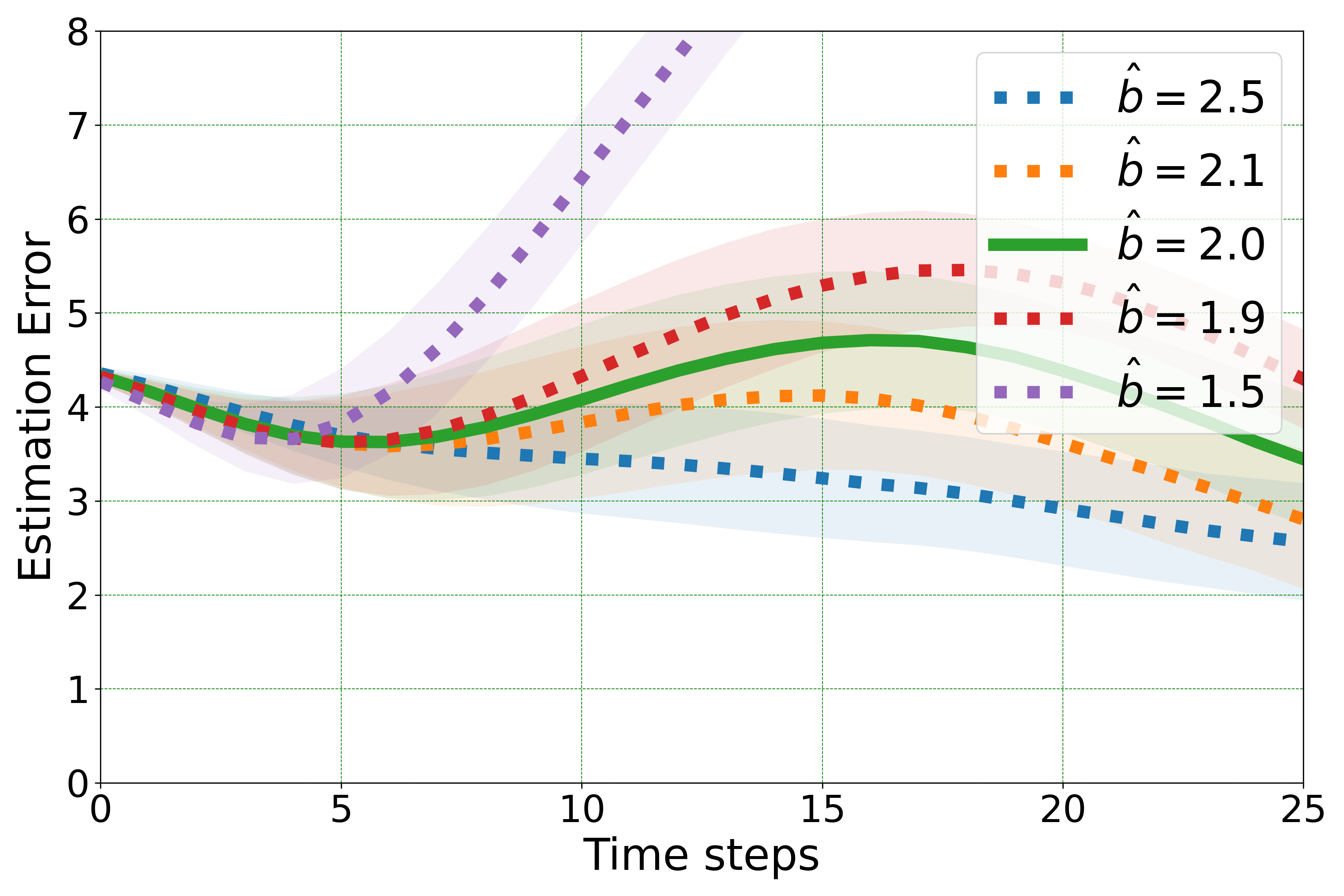}
		\caption{Stationary sensors}\label{fig:stationary-robust}
	\end{subfigure}
	\caption{Robustness. The solid lines correspond to the perfect measurement model; the dotted lines are measurement models with errors. The color bands show the standard deviation of errors across 100 repetitive trials.}
	\label{fig:Robustness}
	\vspace{-10pt}
\end{figure}
We next investigate whether our algorithm can function despite the error in the measurement model.
We simulate source seeking with ten mobile sensors and a stationary source, in which the measurement is generated by $y = 1/r^2+\nu,~\nu\overset{i.i.d}{\sim} \mathcal{N}(0,0.01)$. We provide imperfect measurement models to the EKF in the form of $h(r)=1/r^{\hat{b}}=1/r^{2+\Delta b}$, with $\Delta b$ taking values in $0,\pm 0.1,\pm 0.5$. We study the estimation error in two settings: 1) The sensors move freely using our algorithm; 2) The sensors are stationary. The results are shown in Fig. \ref{fig:Robustness}.

The robustness of our algorithm is two-fold: In Fig. \ref{fig:moving-robust}, when compared to using the perfect measurement model ($\Delta b=0$), our algorithm shows no significant degradation in estimation when $\Delta b$ is small, maintaining reasonable estimation quality despite imperfect measurement models. 
Besides, when comparing Fig. \ref{fig:moving-robust} and Fig. \ref{fig:stationary-robust}, our algorithm shows more robustness than stationary sensors, whose estimations tend to divergence when $\Delta b$ is at the value of $0.5$.

\subsection{The effectiveness of our distributed algorithm}\label{sec:distr-experiment}
	
	\begin{figure*}[ht]
		\centering
		\begin{subfigure}{0.24\linewidth}
			\includegraphics[width=\textwidth]{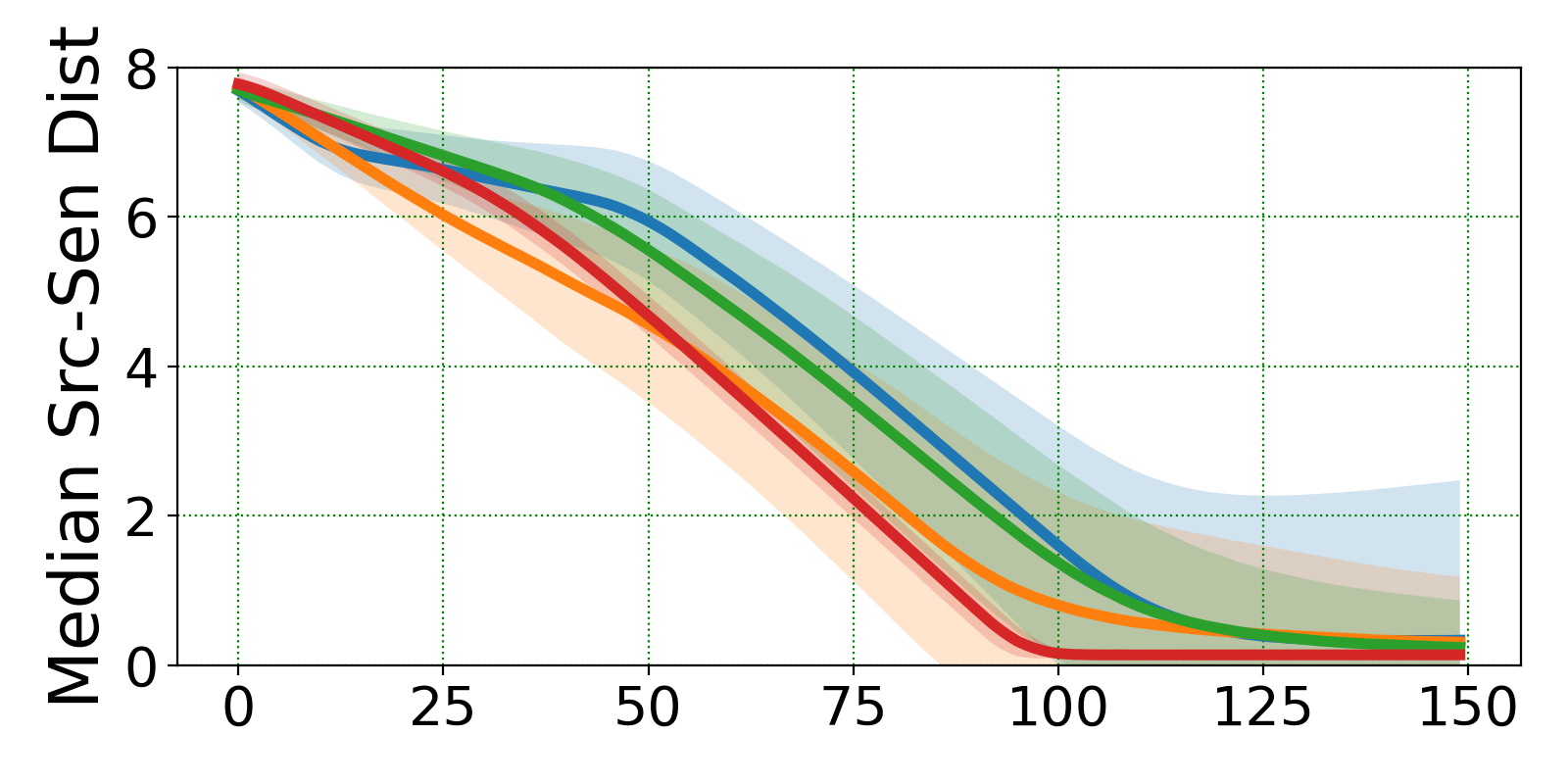}
		\end{subfigure}
		\begin{subfigure}{0.24\linewidth}
			\includegraphics[width=\textwidth]{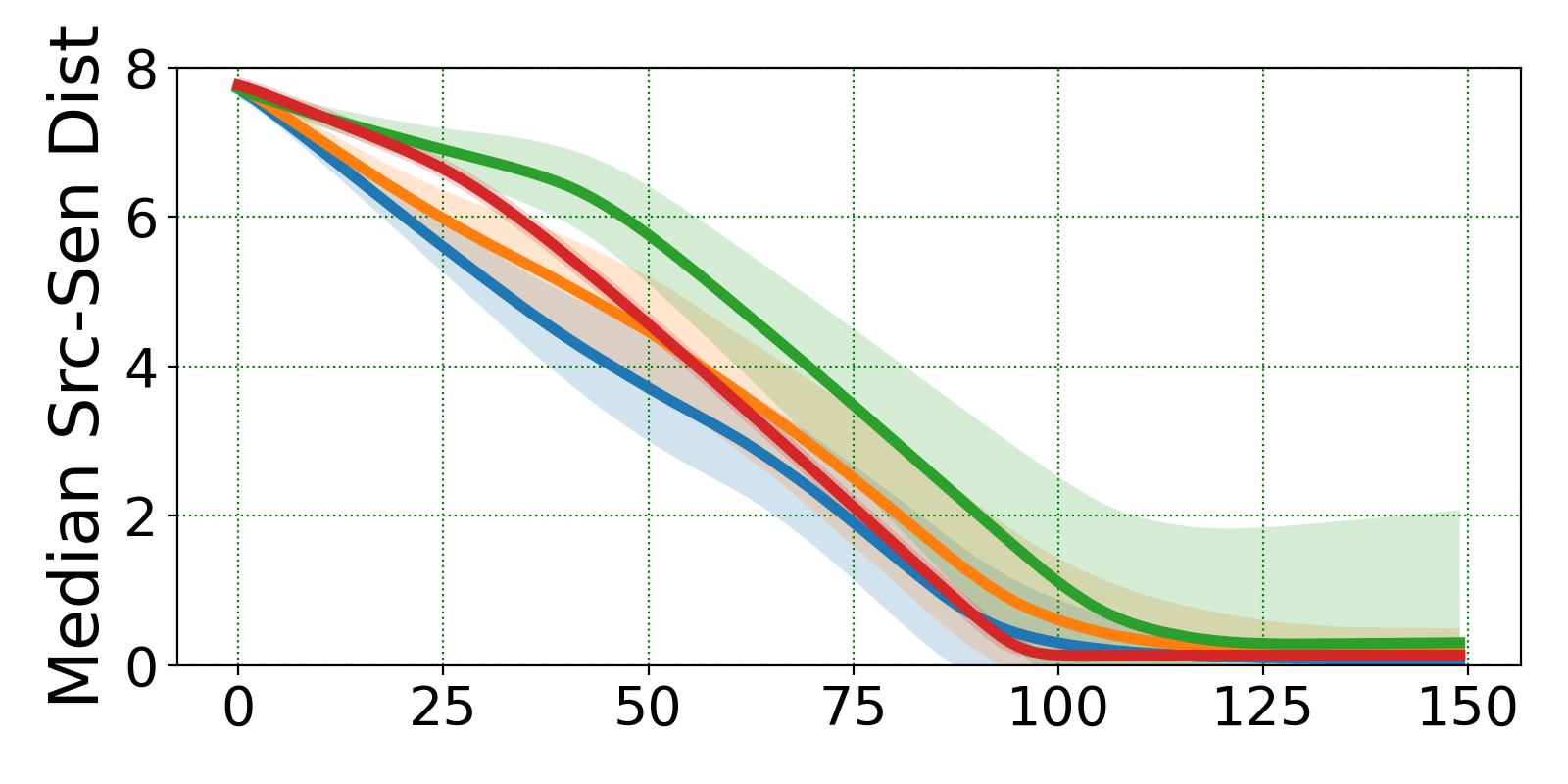}	
		\end{subfigure}
		\begin{subfigure}{0.24\linewidth}
			\includegraphics[width=\textwidth]{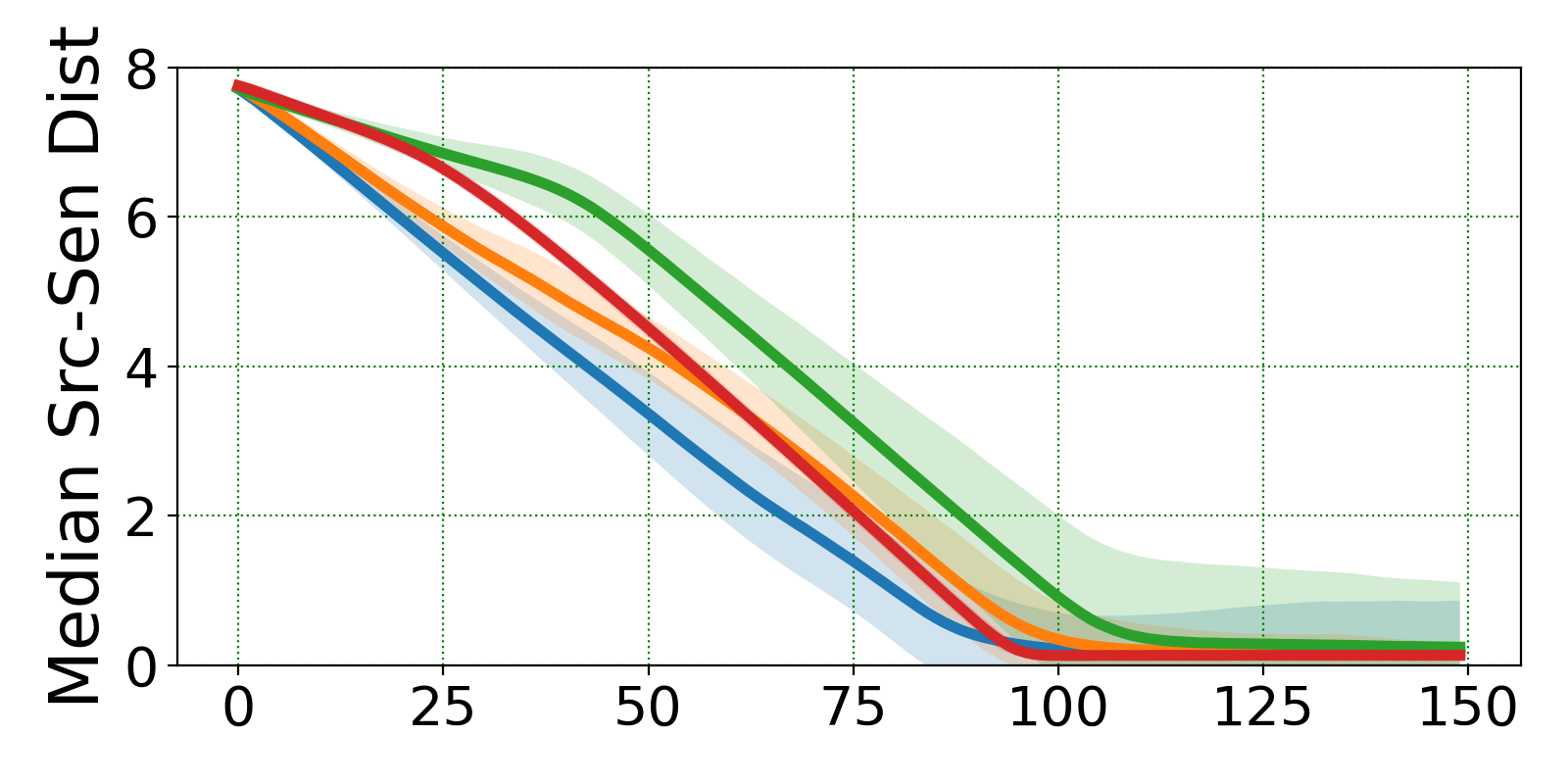}	
		\end{subfigure}
			\begin{subfigure}{0.24\linewidth}
		\includegraphics[width=\textwidth]{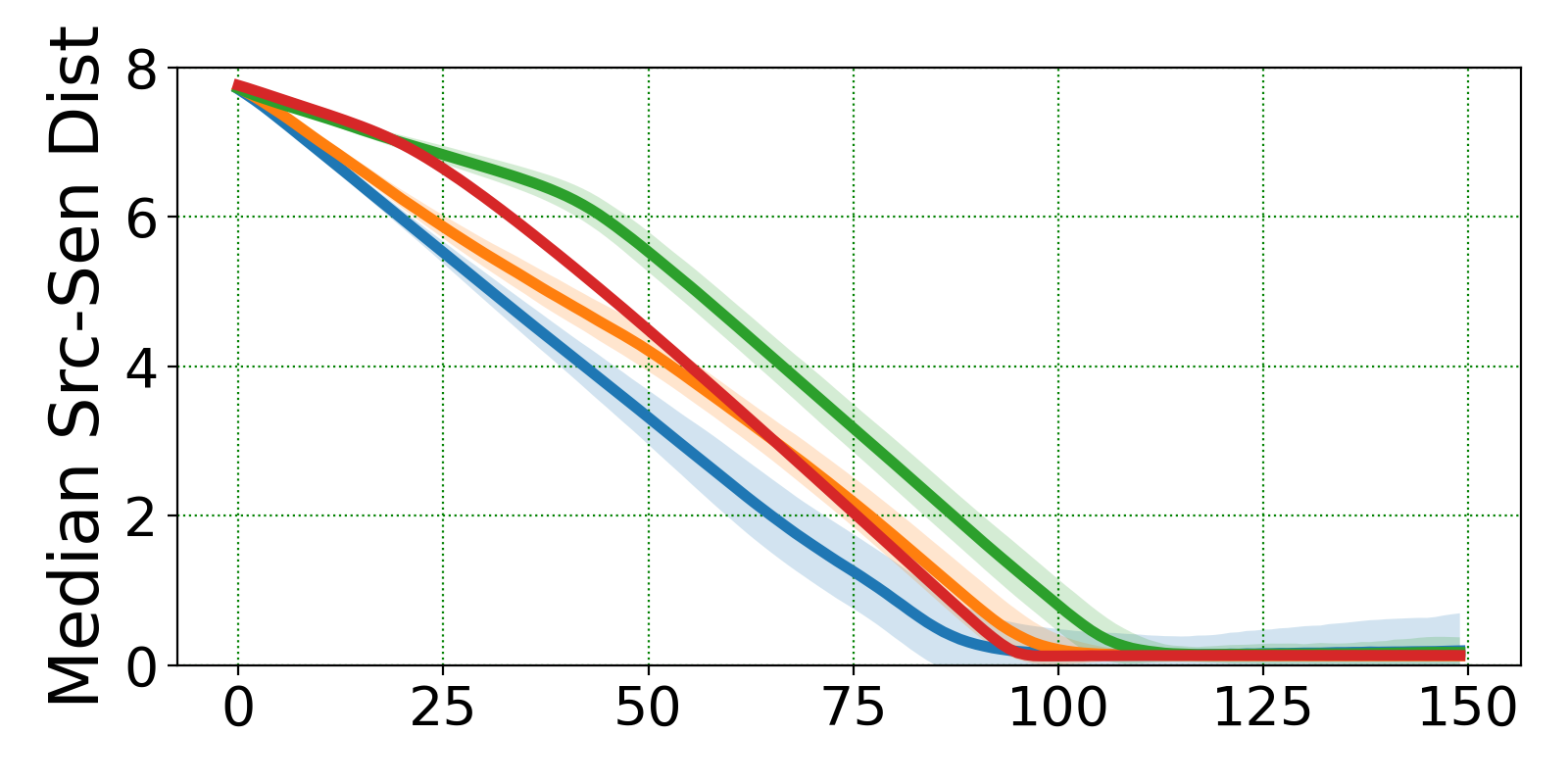}	
	\end{subfigure}
		\begin{subfigure}{0.24\linewidth}
		
		\includegraphics[width=\textwidth]{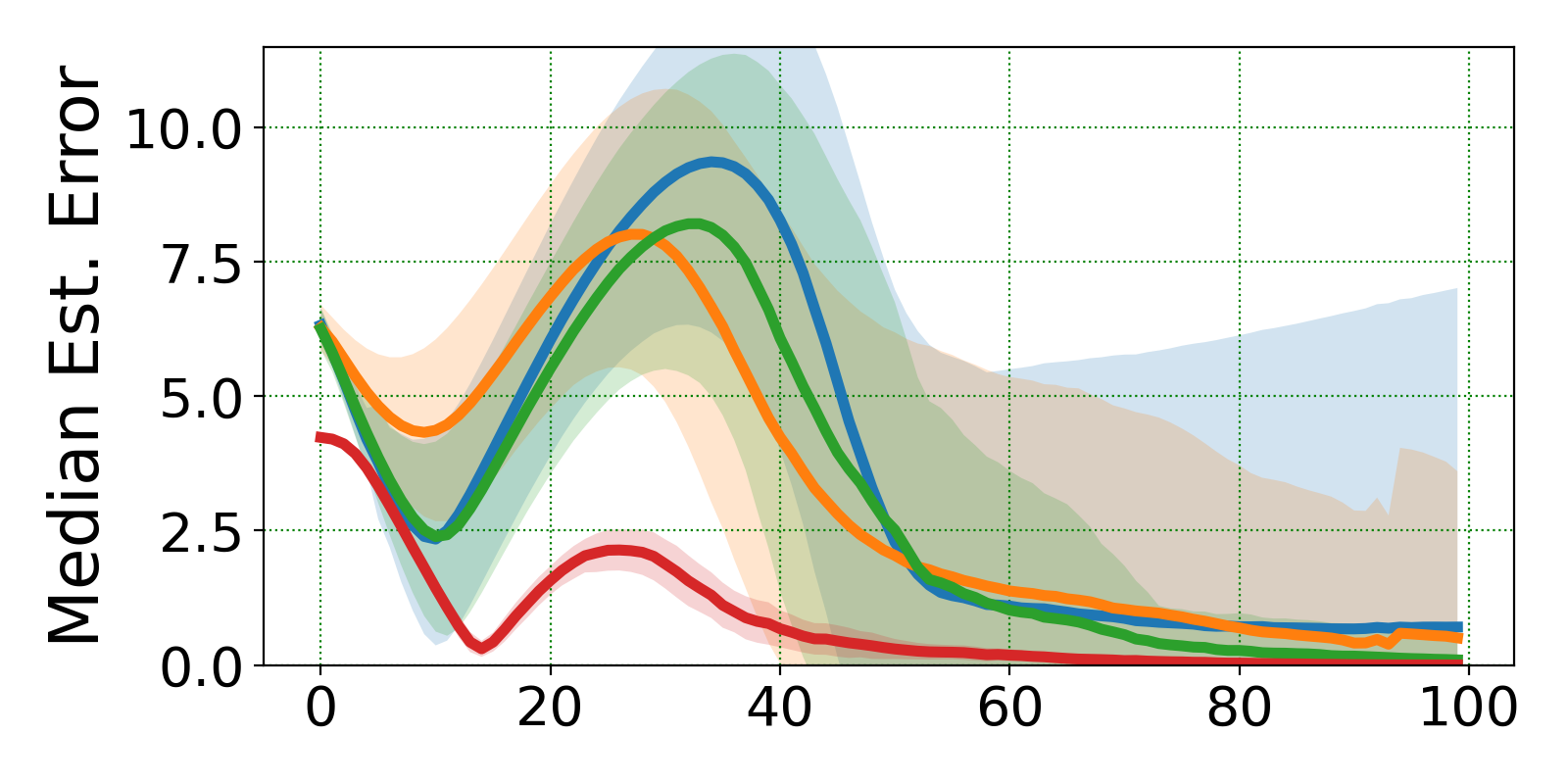}	\caption{4 Sensors}\label{subfig:4-sensors-internal-compare}
		\end{subfigure}
		\begin{subfigure}{0.24\linewidth}	\includegraphics[width=\textwidth]{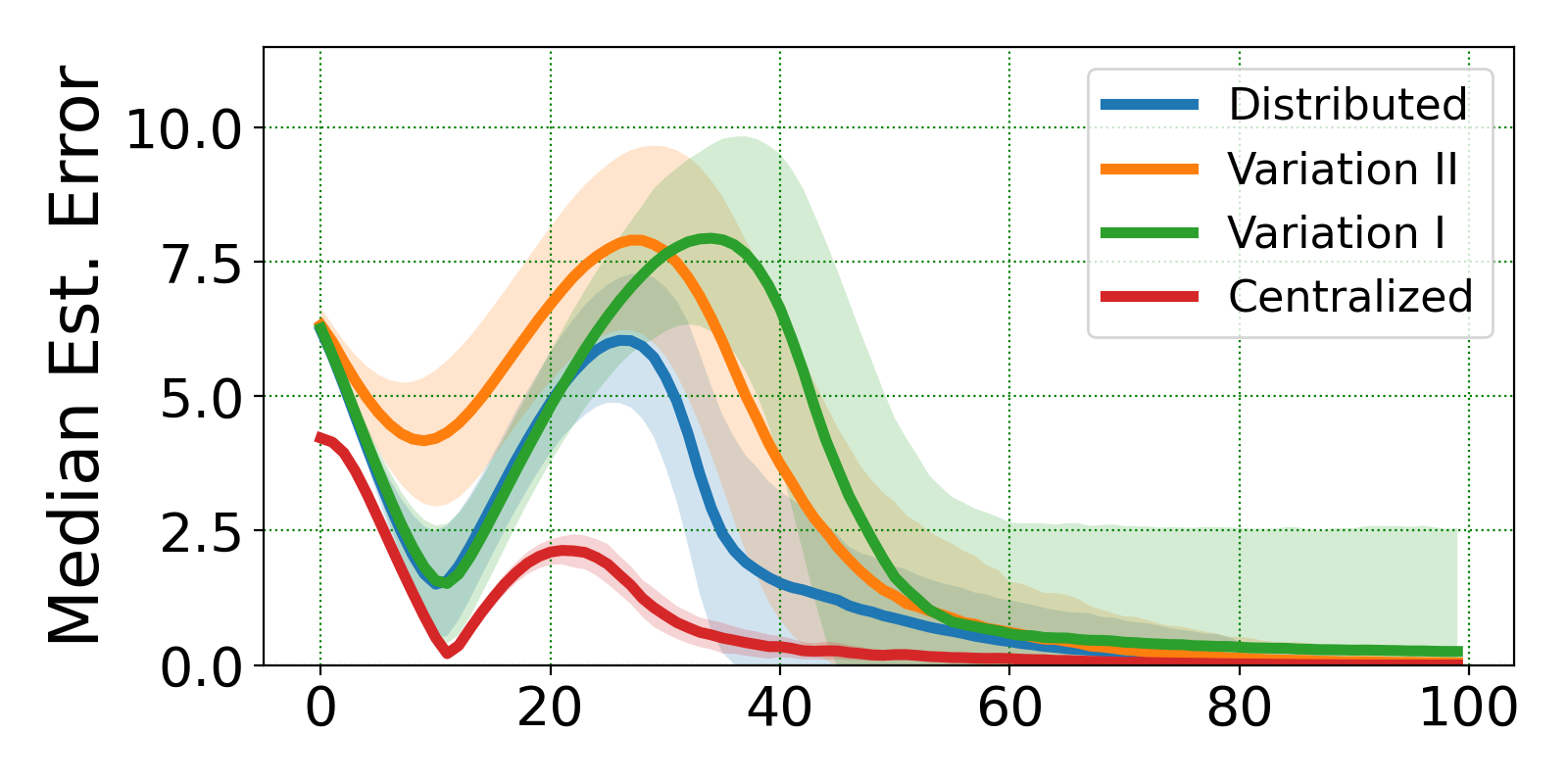}	\caption{10 Sensors}
		\end{subfigure}
		\begin{subfigure}{0.24\linewidth}
	\includegraphics[width=\textwidth]{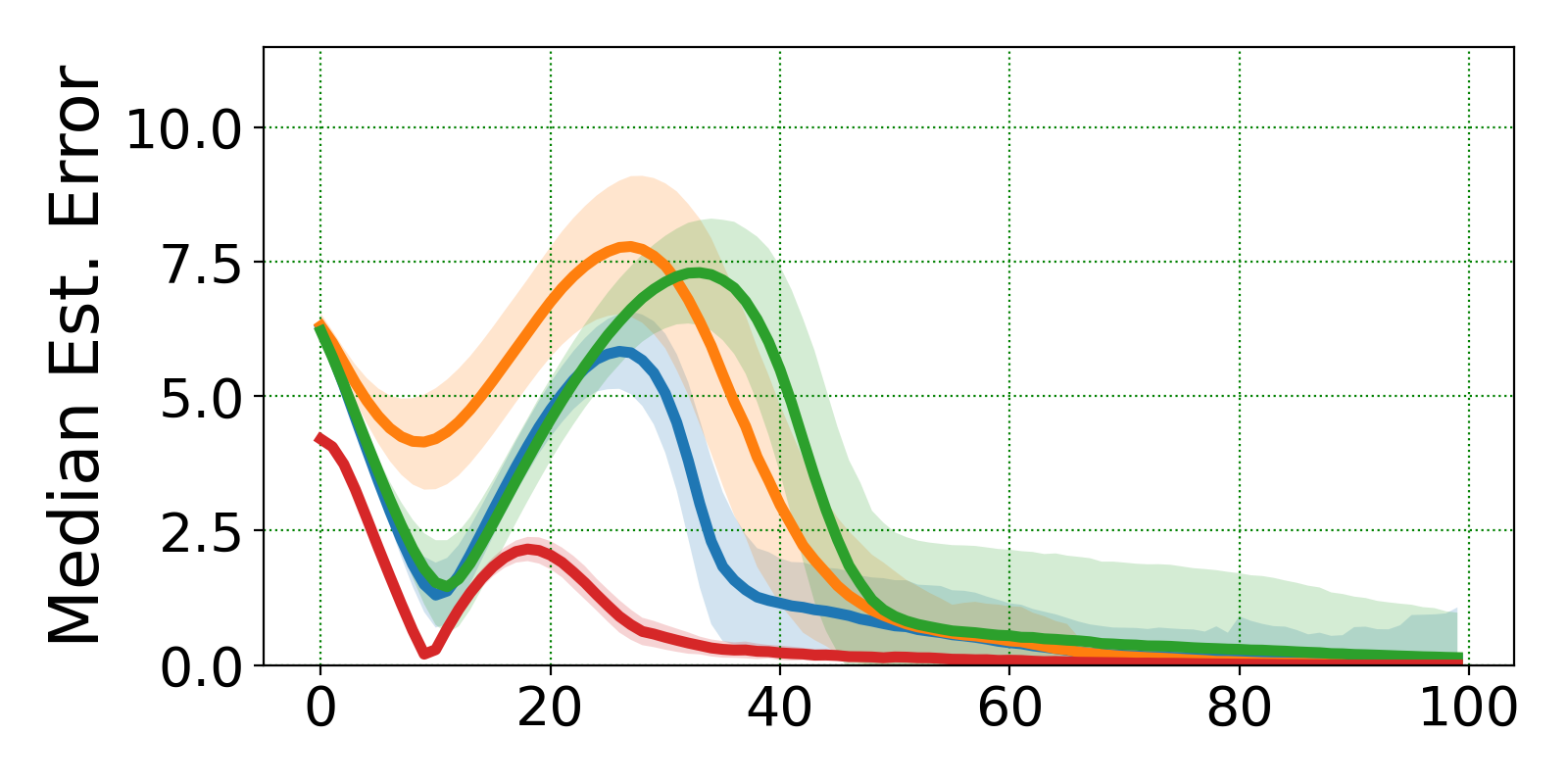}	\caption{20 Sensors}
		\end{subfigure}
    	\begin{subfigure}{0.24\linewidth}
		\includegraphics[width=\textwidth]{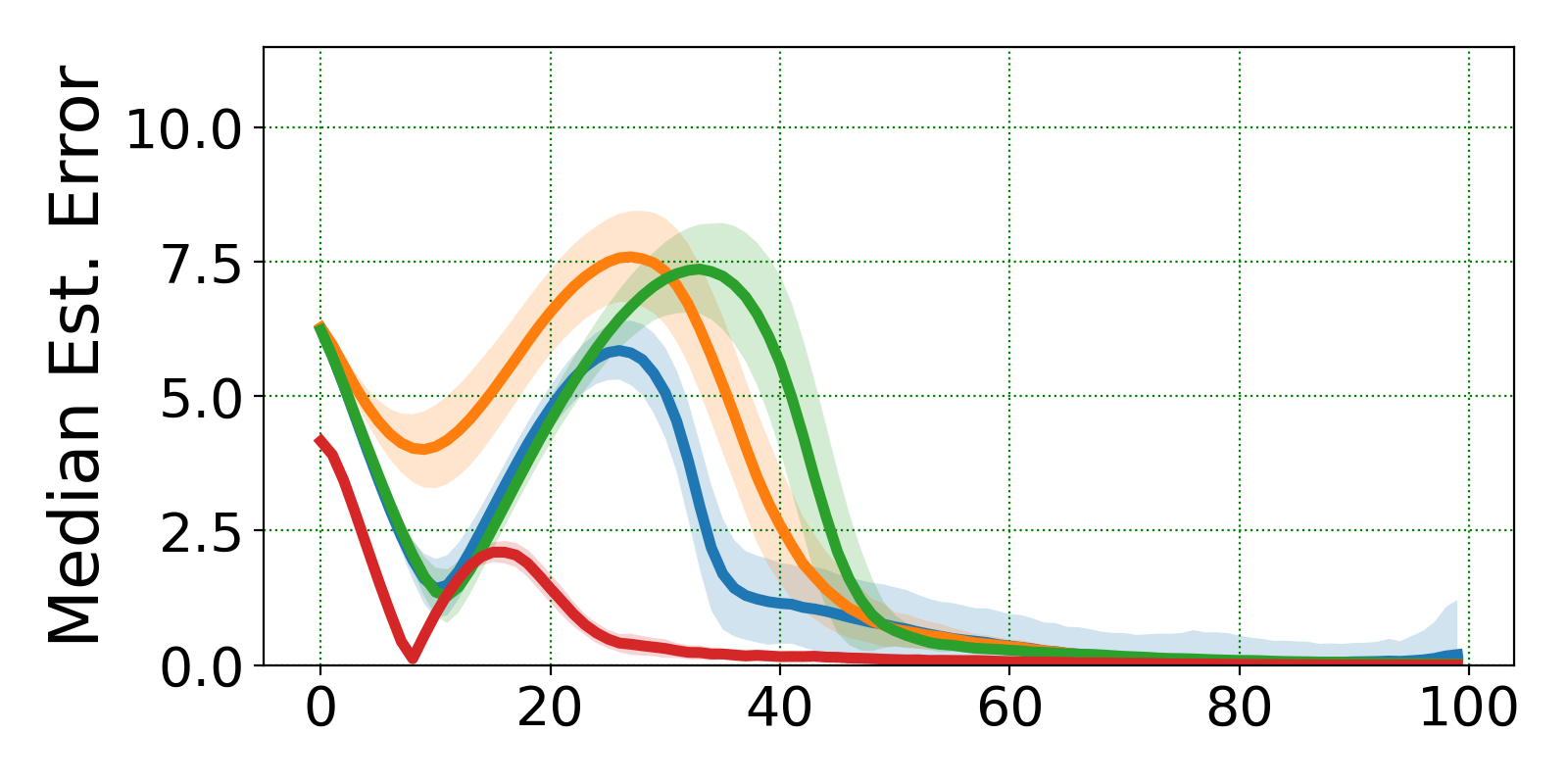}	\caption{40 Sensors}
	\end{subfigure}
		\caption{Comparison of the time evolution of sensor-source distance and estimation error for different variations, with the additional case where $4$ sensors are used. The x-axes correspond to time steps. The `Median Src-Sen Dist' in the y-axis label of the subfigures in the first row is defined as $\underset{i}{\text{median}} ~||p_i-q||$, where $||p_i-q||$ is sensor $i$'s distance to the source. The `Median Est. Error' in the second row is defined as $\underset{i}{\text{median}} ~ ||\hat{q}_i-q||$, where $||\hat{q}_i-q||$ is sensor $i$'s estimation error. Blue curves: our distributed algorithm. Green curves: Variation I. Orange curves: Variation II. Red curves: our centralized algorithm.
		}\label{fig:Internal-Compare}
	\end{figure*}

\begin{figure*}[ht]
\centering
\begin{subfigure}[t]{0.24\linewidth}
		\centering
		\includegraphics[width=\linewidth]{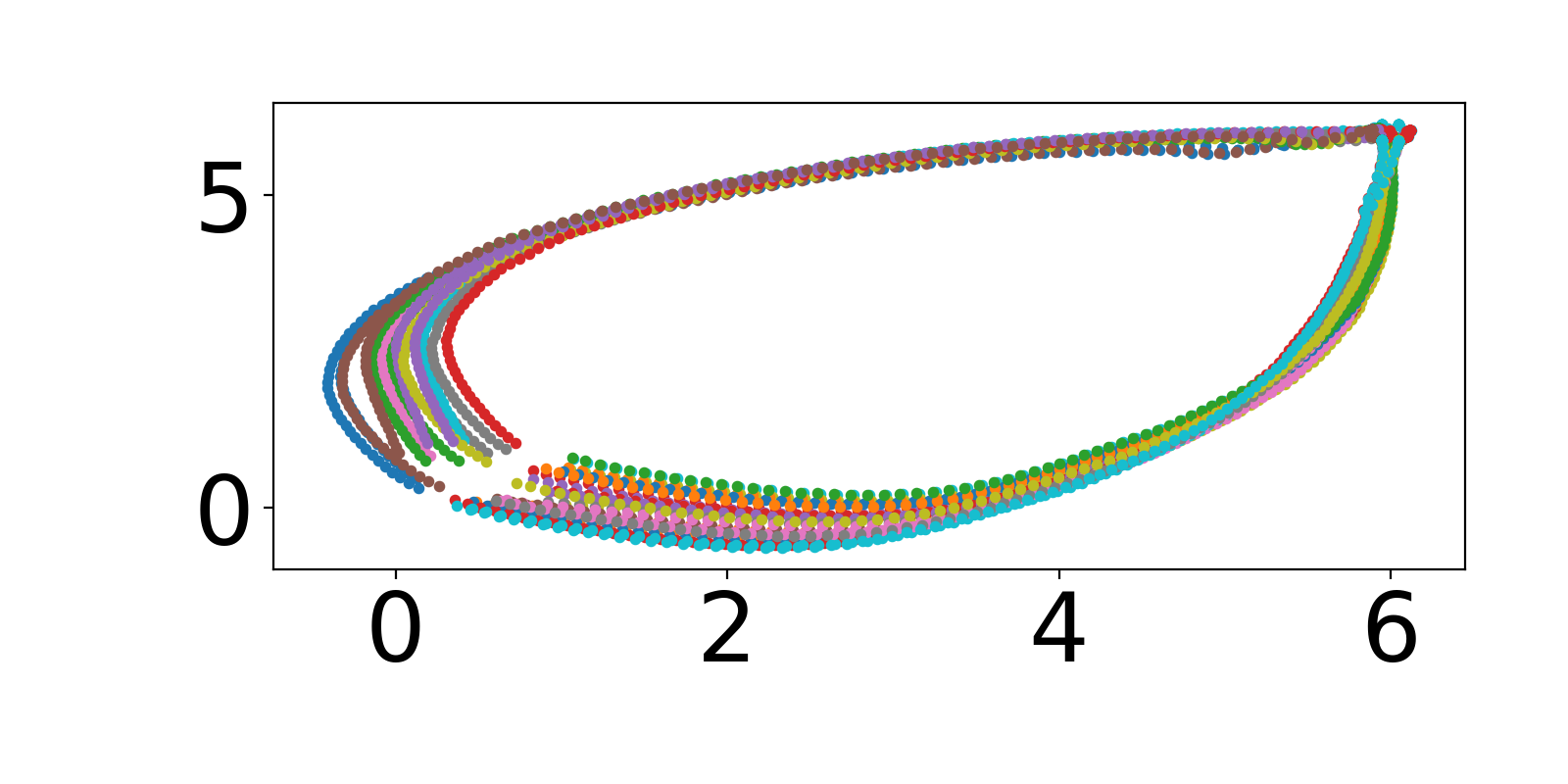}\caption{Centralized}
	\end{subfigure}
\begin{subfigure}[t]{0.24\linewidth}
		\centering
		\includegraphics[width=\linewidth]{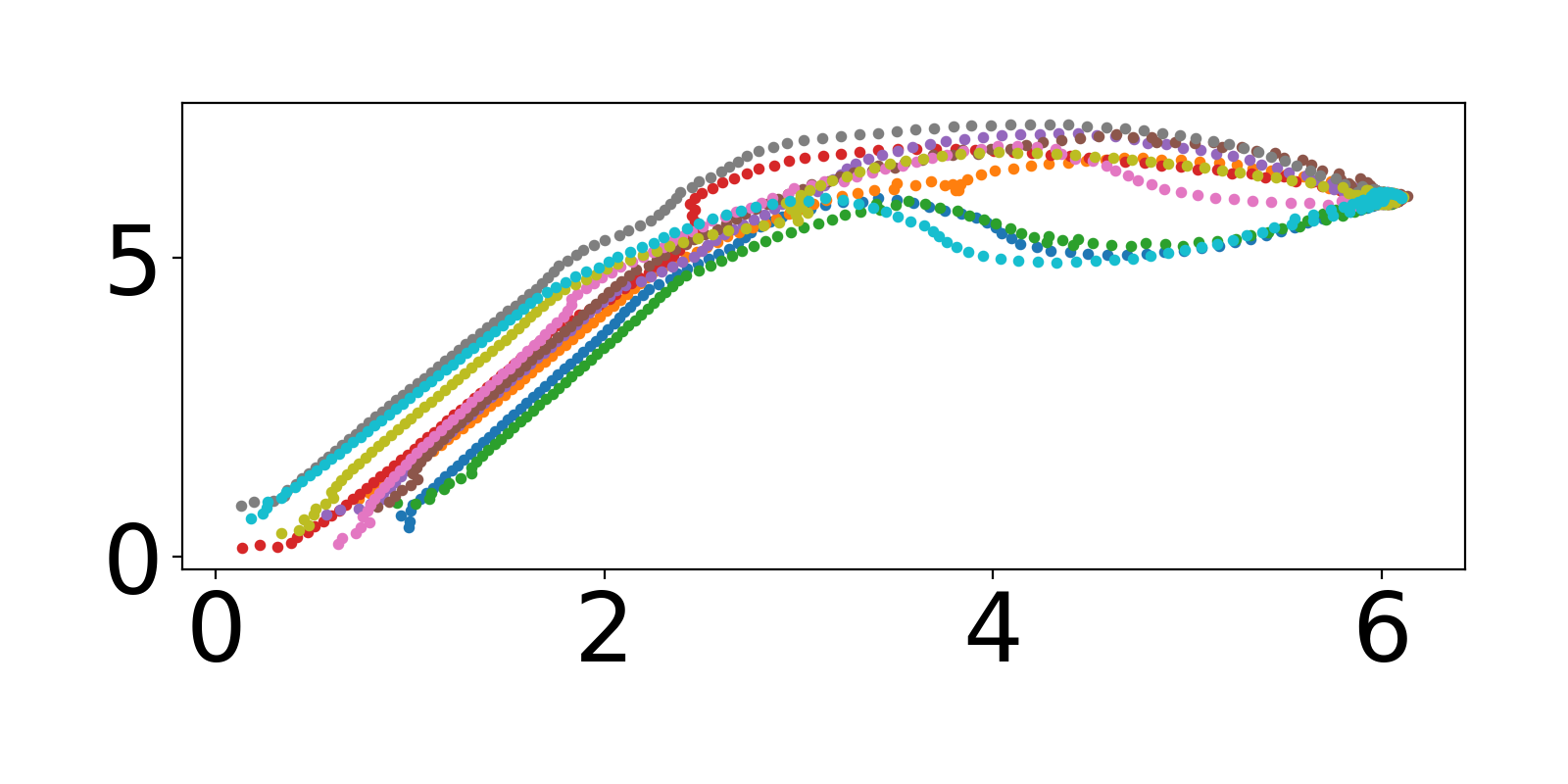}\caption{Distributed }
	\end{subfigure}
\begin{subfigure}[t]{0.24\linewidth}
		\centering
		\includegraphics[width=\linewidth]{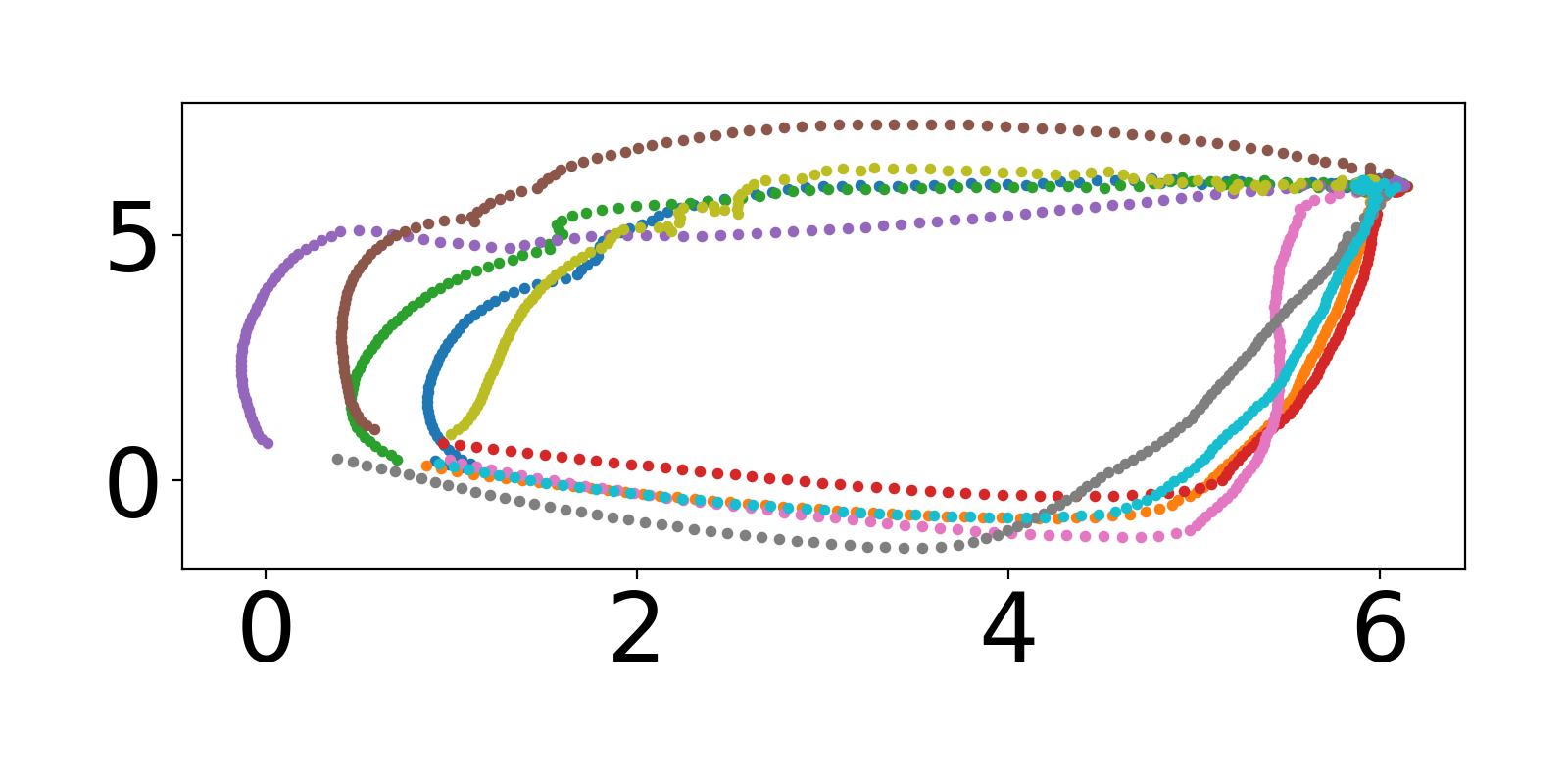}\caption{Variation I}
	\end{subfigure}
\begin{subfigure}[t]{0.24\linewidth}
		\centering
		\includegraphics[width=\linewidth]{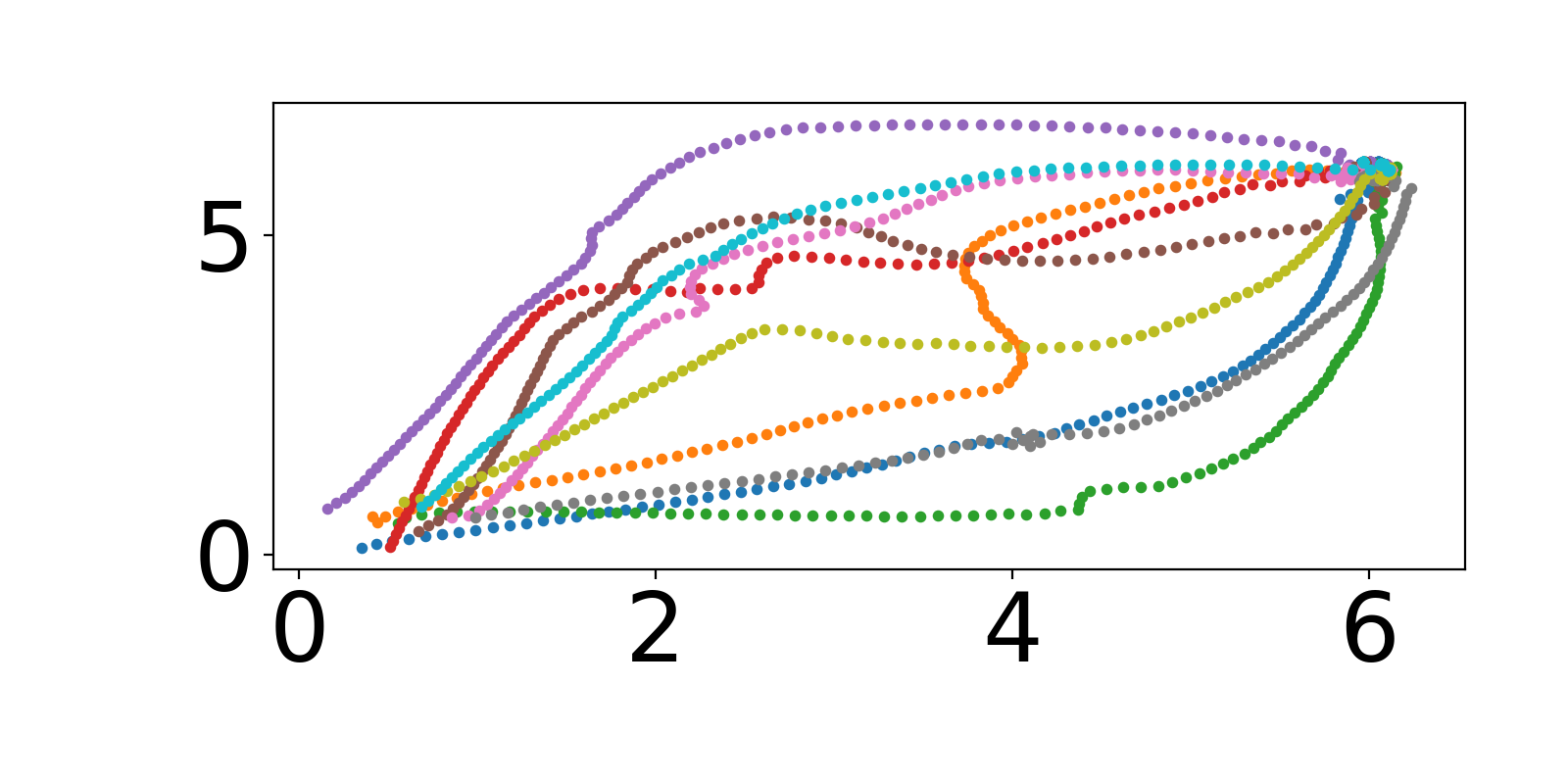}\caption{Variation II}
	\end{subfigure}
	\caption{Sample trajectories from the experiments in Fig. \ref{fig:Internal-Compare}. The samples are taken randomly from 100 simulations. There are $10$ sensors in each figure. The figures illustrate the general trend that the centralized algorithm tends to spread out more than the distributed variations in the initial steps. Also, Variation I spreads out more than Variation II and our distributed algorithm.}
	\label{fig:Variation-Sample-Trajectory}
\end{figure*}
	In this subsection, we showcase the effectiveness of our algorithm when extended to the distributed setting. {In Fig. \ref{fig:Internal-Compare}-\ref{fig:Delay}, the x-axes are time steps, and the color bands show the standard deviation across 100 repetitive experiments. }
	Specifically, our experiments compare the following groups: 
	\begin{itemize}
		\item \textbf{Distributed Algorithm \ref{alg:distri-alg}}: Consensus is used in estimation and gradient update. The sensors have a shared loss function.
		\item \textbf{Variation I}: Consensus is used in estimation but not gradient update. The sensors have different local loss functions. 
		\item \textbf{Variation II}: No consensus in either estimation or gradient update.   Each sensor gathers measurement information from its neighbors and use local EKF for estimation. The sensors have different local loss functions.
		\item \textbf{Centralized Algorithm \ref{alg:centralized-alg}}.
	\end{itemize}
	The first three groups are implemented as distributed algorithms. The fourth is Algorithm \ref{alg:centralized-alg} with a central controller, which serves as the control group. Detailed descriptions for the two variants I and II are provided in Appendix \ref{append:Distributed-Variants}.  
	All sensors use the same measurement model as defined in \eqref{eq:experiment-measurement-model}. We simplify the sensor communication network so that the distributed algorithms use a static, undirected circulant network with each mobile sensor connected with two neighbors. 
    
Although the neighbors of a sensor in physical locations could change over time, the communication neighborhood is fixed and does not vary in our experiments. The robots communicate via a single-hop, static network, where each communicates only with a fixed set of neighbors on the network.\footnote{This feature is also implementable in our hardware experiments because our lab space is small enough that all sensors are within the maximal range of communication with one another.} Using a fixed and sparsely connected communication network simulates the effect of limited information in the distributed setting while allowing simpler, more interpretable results from the experiments. However, fixed communication is indeed a limiting assumption. Our algorithm can be revised to accommodate a more general, potentially time-varying network, such as the ad-hoc network where the robots can only communicate with teammates within a limited distance. Based on the previous studies on distributed consensus and optimization over time-varying communication graphs, e.g., \cite{you2013consensus,nedic2014distributed,yang2019survey}, 
we expect our algorithm to achieve good performance as long as the network remains connected for a sufficient fraction of the time. This is left as our future work.

We repeat the experiments of seeking a static source. The sensor locations are initialized in the same way as in the previous subsection \ref{sec:advantage}. The source is still situated at $(6,6)$.  We use the fixed location $(3,3)$ as the initial guess of the source location for the centralized control group. As for the first three distributed groups, each sensor uses an independent random location in $[0,3]\times [0,3]$ as the initial guess. 

The main results are shown in Fig. \ref{fig:Internal-Compare}, showing the time evolution of median sensor-source distance and estimation error. All variations converge to the source in around $100$-$125$ steps. Sections \ref{sec:entire-team-convergence} and \ref{sec:success-perfect-finish} further analyze the data from Fig. \ref{fig:Internal-Compare} to demonstrate the advantage of our distributed algorithm over Variations I and II in terms of full-team convergence and success/perfect finish rates.
	
\subsubsection{Discussion on the main results}\label{sec:distributed-main-result}
The curves in the first row of Figure \ref{fig:Internal-Compare} represent $\median_i ~||p_i-q||$, the median of source-sensor distances within the team. The second row shows the median estimation error $\median_i ~||\hat{q}_i-q||$. All the estimation error plots in Figure \ref{fig:Internal-Compare} show the distributed variations have an overall worse estimation than the centralized algorithm. Also, note that when the number of sensors is small, as in Figure \ref{subfig:4-sensors-internal-compare}, the centralized algorithm has a clear advantage over the distributed variations in estimation and convergence to the source. These results exemplify one of the weaknesses of the distributed variations: each agent has access to much less measurement information than the centralized algorithm.

However, as the number of sensors increases, we can observe that distributed variations improve in both estimation and convergence to the source. As the number of sensors exceeds $10$, Variation II and our distributed algorithm achieve comparable convergence to the source as the centralized algorithm. The centralized algorithm still reaches the source the fastest, in around $100$ steps, while Variation II and our distributed algorithm take slightly longer than that. But interestingly, the distance to the source for the centralized algorithm seems to decrease slower than these two variations in the initial steps, and the reason is likely that the centralized algorithm tends to spread out more than the distributed variations due to the stronger coordination in sensor movements, as shown in the illustrative example in Figure \ref{fig:Variation-Sample-Trajectory}. 

It is also interesting that Variation I seems to be consistently worse than other distributed variations in convergence to the source. Recall that Variation I performs consensus in estimation but not in gradient update, giving it more information than Variation II.  We do observe from Figure \ref{fig:Internal-Compare} that in the early stages, its estimation error is comparable with our distributed algorithm and is lower than Variation II. However, as illustrated in Figure \ref{fig:Variation-Sample-Trajectory}, Variation I tends to spread out more than other distributed variations, yet it does not coordinate the movement among the agents, meaning it is slow to converge to the source. These factors potentially explain why Variation I does not have an advantage over Variation II in terms of the $\median_i ~||p_i-q||$ metric.

The main results suggest that with a large number of sensors, even with the limited information for each sensor, as in Variation II, the distributed variations can still perform well. The sheer size of the distributed sensor team allows a high chance of some sensors getting close to the source and obtaining good measurement readings. Nevertheless, our distributed algorithm still outperforms Variations I and II due to the additional consensus in estimation and movement coordination.

Overall, our distributed algorithm converges to the source at a comparable rate as the centralized implementation. It also converges to the source slightly faster than Variations I/II and exhibits more improvement when using more sensors. Furthermore, our algorithm maintains a clear advantage in estimation accuracy over Variations I/II throughout the seeking process. 

\subsubsection{The convergence of the entire team}\label{sec:entire-team-convergence}
The advantage of our distributed algorithm over Variations I and II is more prominent in terms of the entire team's convergence to the source. Fig. \ref{fig:farthest-sensor} shows the performance of the variations in terms of $\max_i ||p_i-q||$ using the same data as in Fig. \ref{fig:Internal-Compare}, where $||p_i-q||$ is the distance between sensor $i$ and the source. The convergence of $\max_i ||p_i-q||$ to $0$ thus characterizes the entire team's convergence to the source, a condition stronger than the one we proposed in Eq. (4). Figure \ref{fig:farthest-sensor} shows the $\max_i ||p_i-q||$  of our distributed algorithm converges to the source almost as fast as the centralized algorithm when the number of sensors exceeds $10$; meanwhile, as the number of sensors increases, the performance of Variation I gets slightly worse, and Variation II gets significantly worse. These results demonstrate that with consensus in estimation and motion planning, our distributed algorithm can ensure all sensors come to the source sufficiently fast. In contrast, without sufficient coordination via consensus, as in Variations I and II, some sensors could perform much worse than others.
\begin{figure*}[ht]
    \centering
    \begin{subfigure}{0.24\linewidth}
    \includegraphics[width=\textwidth]{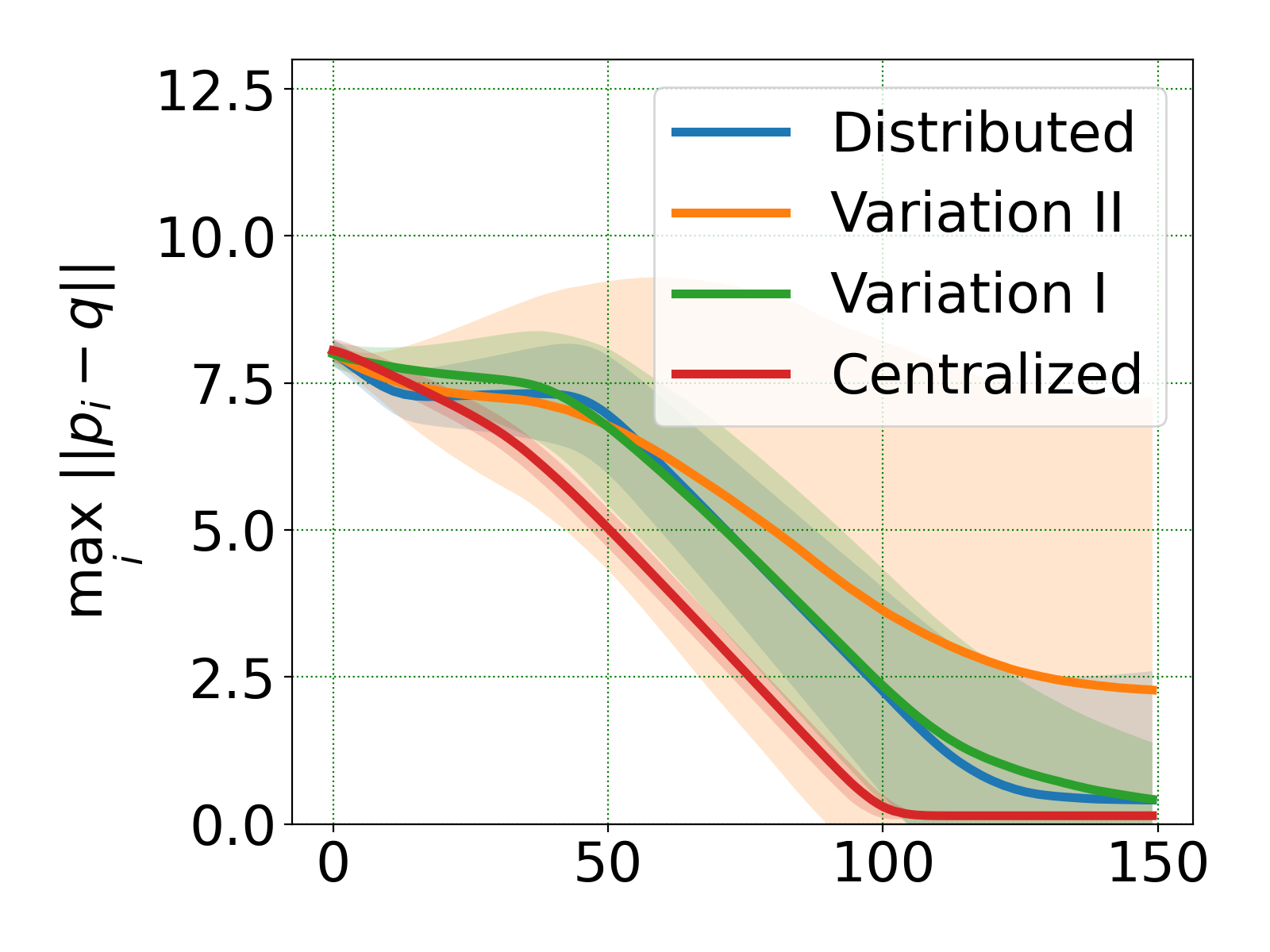}	\caption{4 Sensors}
    \end{subfigure}
    \begin{subfigure}{0.24\linewidth}	\includegraphics[width=\textwidth]{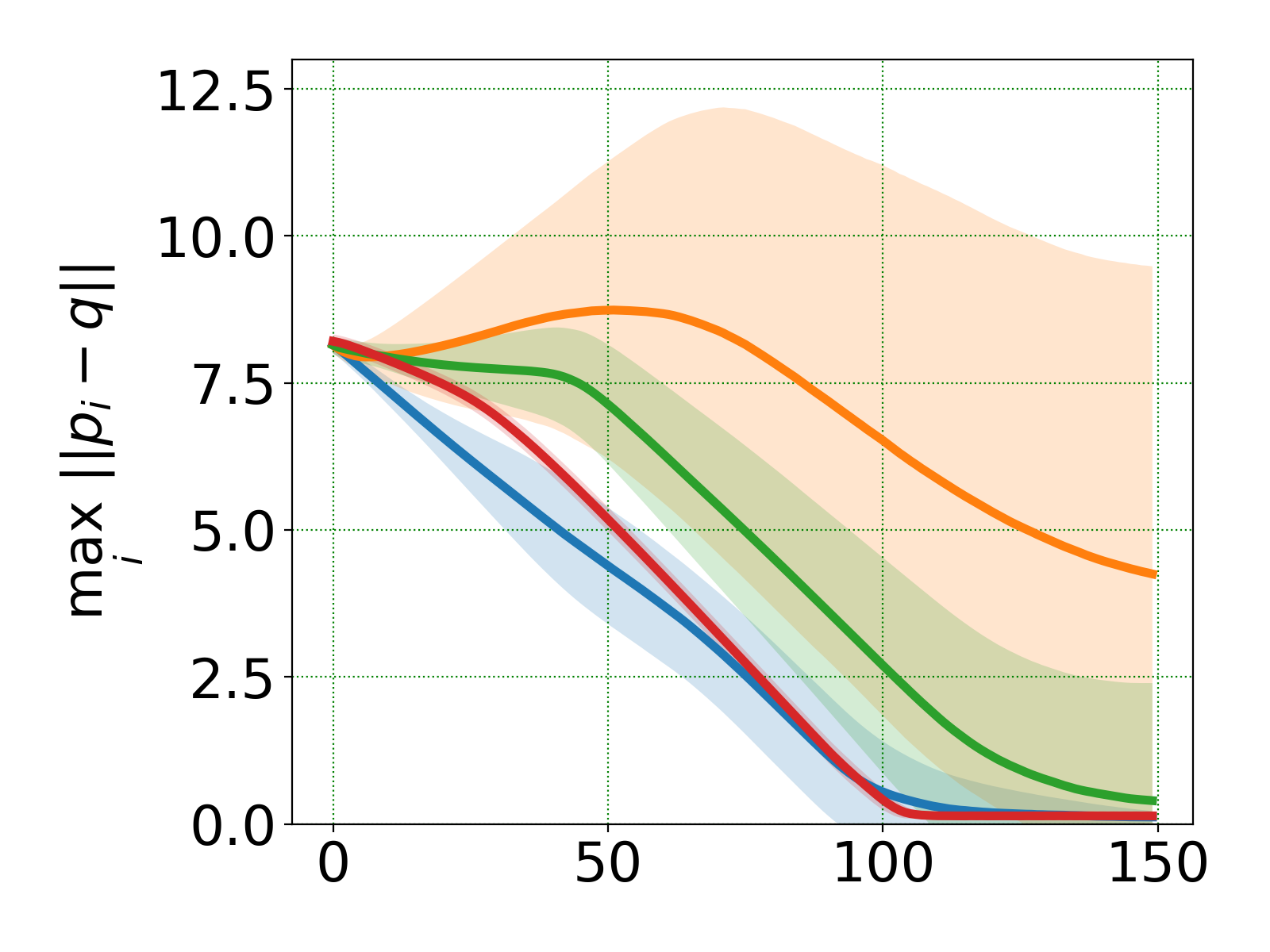}	\caption{10 Sensors}
    \end{subfigure}
    \begin{subfigure}{0.24\linewidth}
\includegraphics[width=\textwidth]{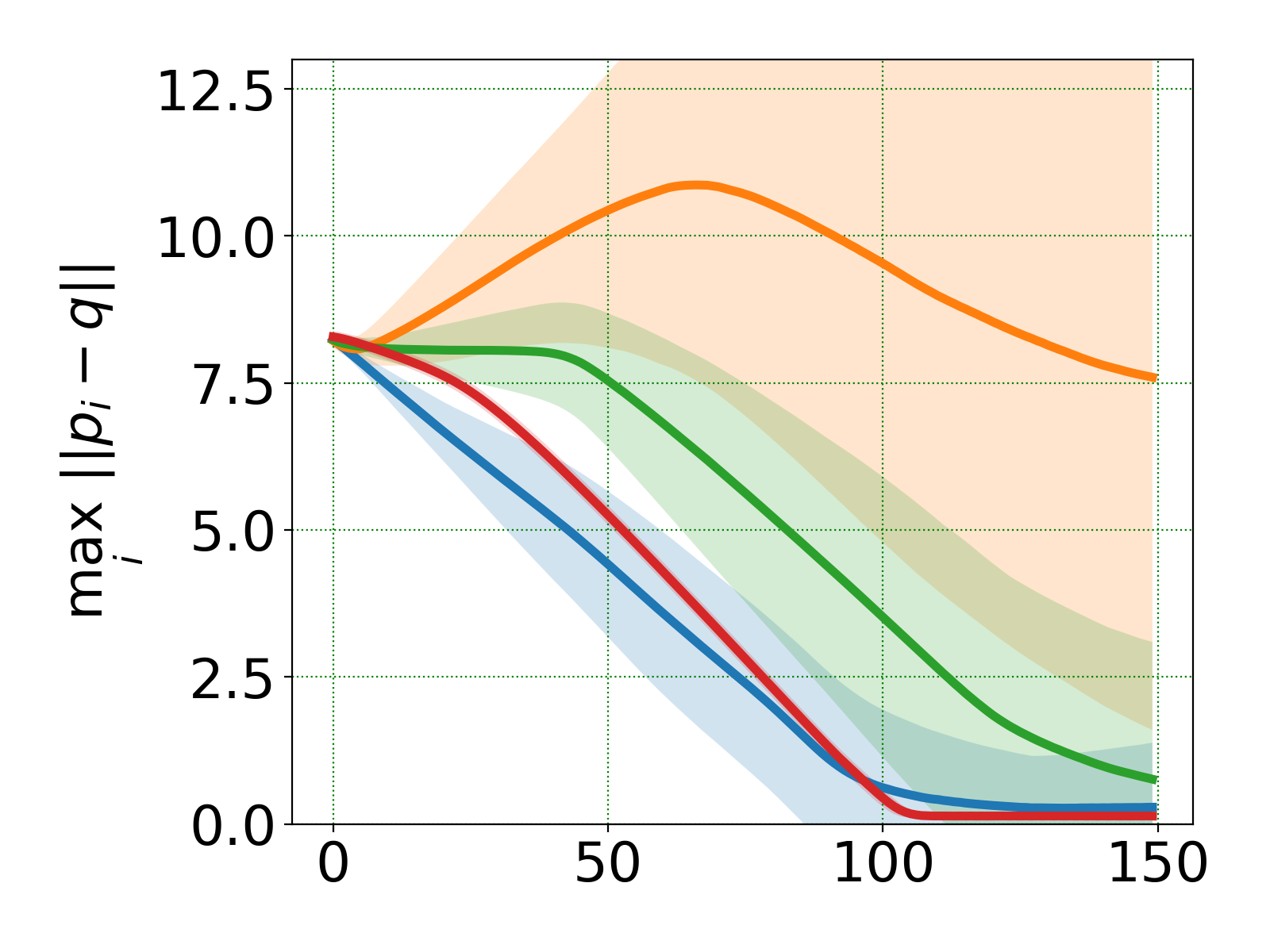}	\caption{20 Sensors}
    \end{subfigure}
    \begin{subfigure}{0.24\linewidth}
    \includegraphics[width=\textwidth]{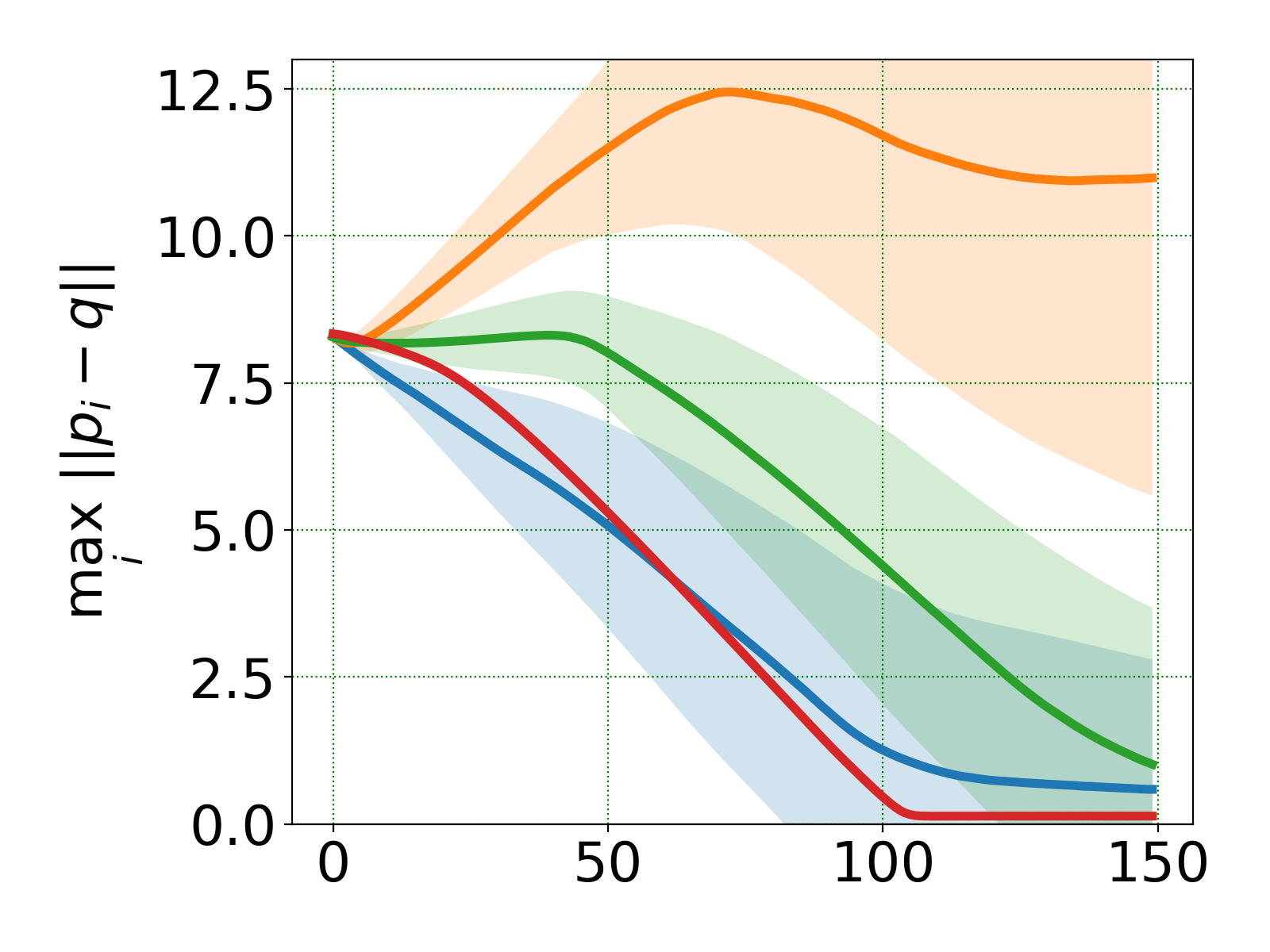}	\caption{40 Sensors}
\end{subfigure}
    \caption{Comparison of the entire team's convergence to the source for different variations, in terms of $\max_i ||p_i-q||$. The figures are generated using the same data in Fig. \ref{fig:Internal-Compare}. The x-axes correspond to time steps.
    }\label{fig:farthest-sensor}
\end{figure*}

\subsubsection{The Success and Perfect Finish Rates}\label{sec:success-perfect-finish}

Fig. \ref{fig:Perfect-Finish-Robustness} shows the perfect finish/success rates of different variations of our algorithm. The figures are generated using the same data from Fig. \ref{fig:Internal-Compare}. A robot is deemed to have reached the source if its distance from the source is below $0.2$. A `success' is defined in Eq. \eqref{SourceSeekingGoal}, corresponding to a trial where \textbf{at least one} sensor reaches the source at the final step. A `perfect finish' is a trial where \textbf{all} sensors achieve that, which, like the $\max_i ||p_i-q||$ metric, is a notion characterizing the entire team's convergence. The success and perfect finish rates are computed over $100$ trials. The results show that all variations have close to $100\%$ success rates, which demonstrates the effectiveness of our information-based source-seeking framework. The perfect finish rates of our distributed algorithm are very close to one over any number of sensors. In contrast, the perfect finish rates for Variation I and II drop significantly as the number of sensors increases. The perfect finish rate results again demonstrate the crucial role of consensus in estimation and motion planning in our distributed algorithm: It allows the entire team to reach the source, achieving a more challenging control goal than Eq. \eqref{SourceSeekingGoal}.

\begin{figure}[ht]
		\centering
		\begin{subfigure}{0.45\linewidth}
			\includegraphics[width=\linewidth]{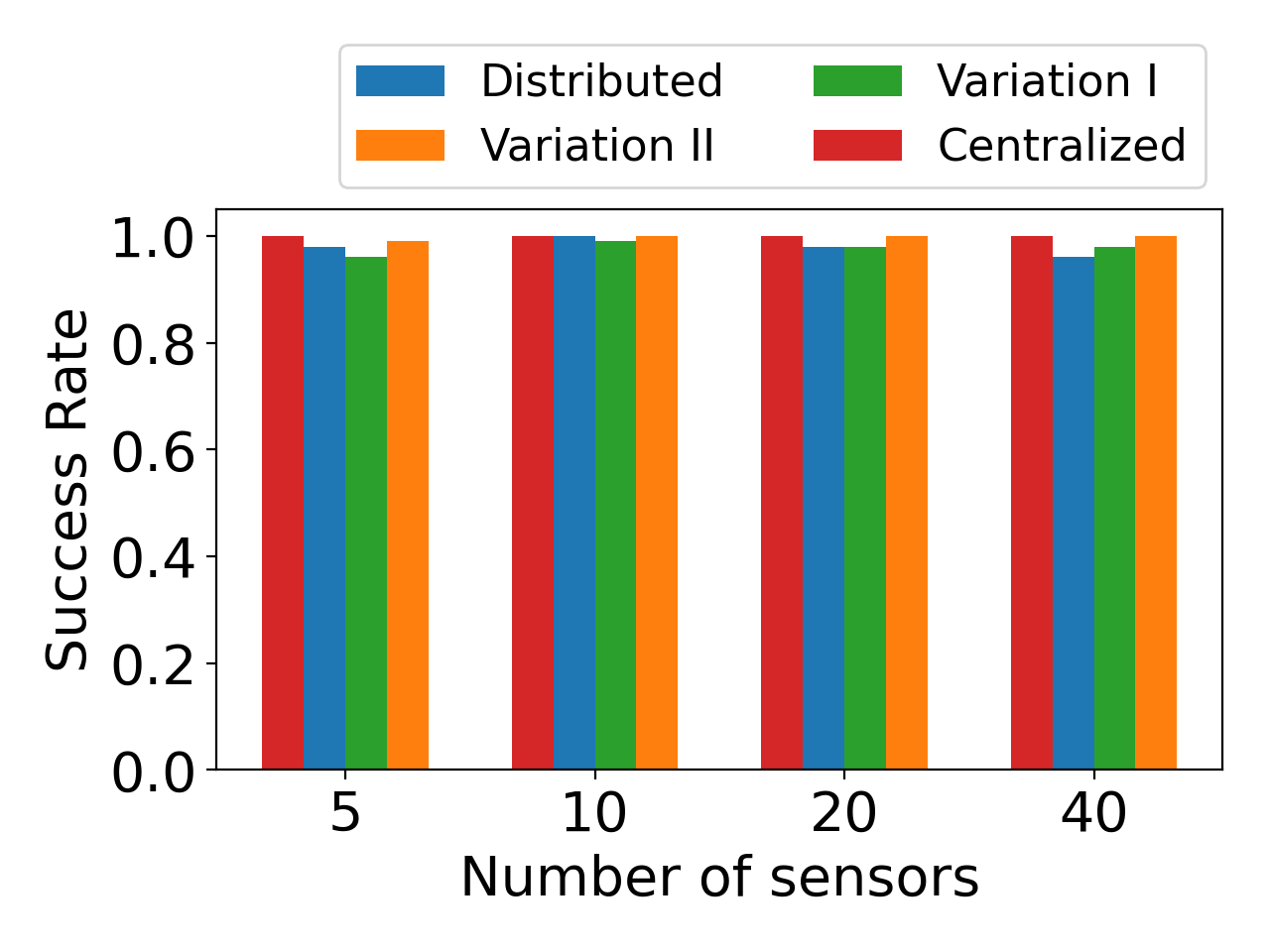}\caption{Success Rate}
		\end{subfigure}
		\begin{subfigure}{0.45\linewidth}
		\includegraphics[width=\linewidth]{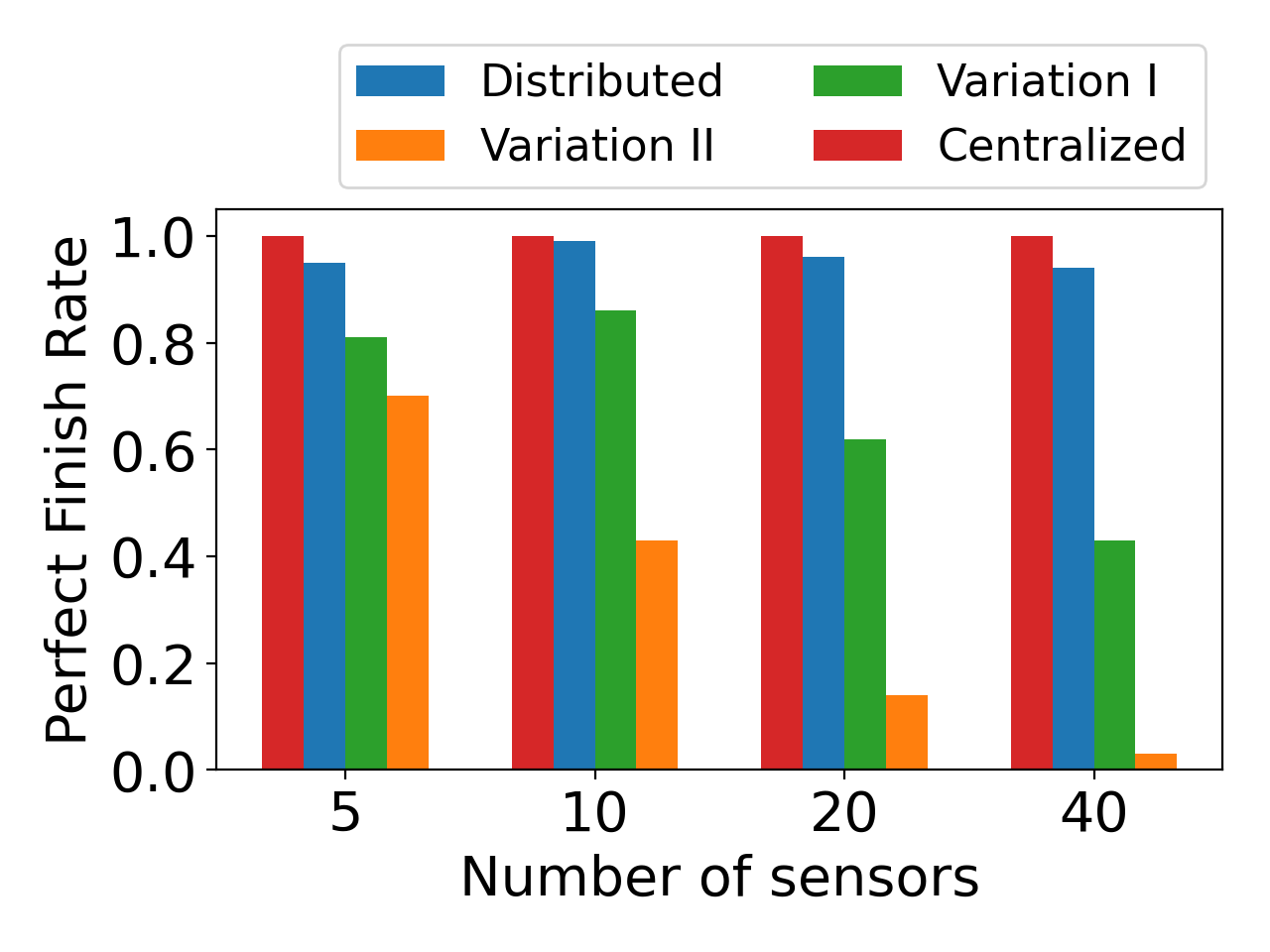}\caption{Perfect Finish Rate}
		\end{subfigure}
		\caption{The success/perfect finish rates for the variations.}\label{fig:Perfect-Finish-Robustness}
\end{figure}

\subsubsection{Discussion on non-monotonic estimation error}

\begin{figure*}[ht]
\centering
\begin{subfigure}[t]{0.24\linewidth}
		\centering
		\includegraphics[width=\linewidth]{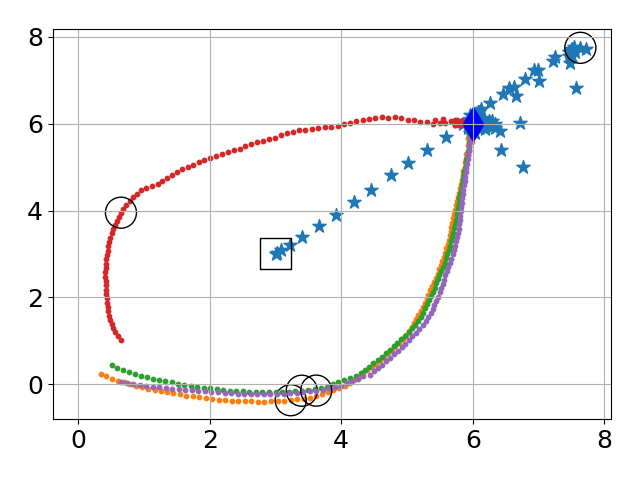}\caption{Centralized}
	\end{subfigure}
\begin{subfigure}[t]{0.24\linewidth}
		\centering
		\includegraphics[width=\linewidth]{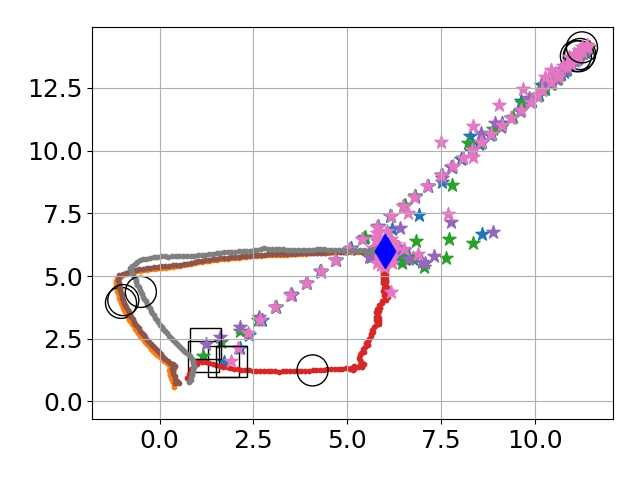}\caption{Distributed }
	\end{subfigure}
\begin{subfigure}[t]{0.24\linewidth}
		\centering
		\includegraphics[width=\linewidth]{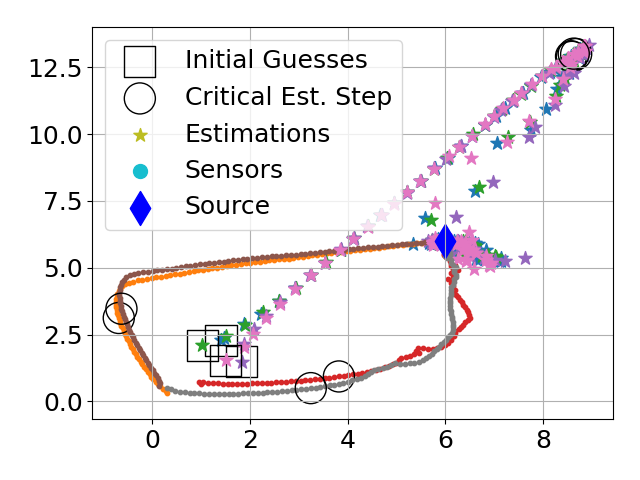}\caption{Variation I}
	\end{subfigure}
\begin{subfigure}[t]{0.24\linewidth}
		\centering
		\includegraphics[width=\linewidth]{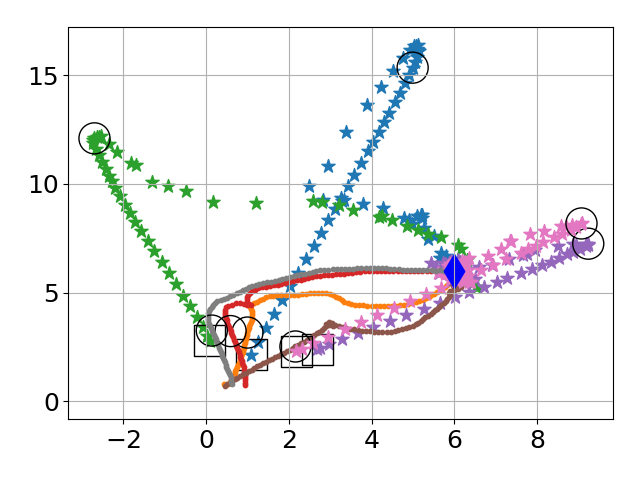}\caption{Variation II}\label{subfig:Variation-Details-Variation-II}
	\end{subfigure}
	\caption{Sample trajectories and estimations for the variations of our algorithm. The samples are randomly selected from 100 simulations. For visualization clarity, we only show the results for $4$-sensor teams. The circles with label \textbf{`Critical Est. Step'} indicates the estimations and sensor locations when the estimation error is the largest. }
	\label{fig:Variation-Details}
\end{figure*}

In Fig. \ref{fig:Internal-Compare}, we observe that the estimation errors do not decrease monotonically. Indeed, classical results for Kalman filters show that monotonic convergence can be achieved, but that applies only to linear systems, where both the measurement and motion models are linear. In contrast, the estimation of EKF on a non-linear measurement model, the estimation error of which is generally not monotonic and could diverge given bad initialization. Although the criteria for the convergence of EKF is discussed in detail in \cite{reif1999stochastic}, we are unaware of any literature that discusses the criterion for monotonicity of the estimation error for EKF.

We also observe in Fig. \ref{fig:Internal-Compare} the somewhat counter-intuitive behavior that source-sensor distances can decrease as the estimation errors increase. Fig. \ref{fig:Variation-Details} unveils the reason for this curious phenomenon. It shows a common situation in the experiments, where the estimation from the EKF could move toward and overshoot the true source location but then comes back to the source as the sensors spread out and move closer to the source. During the overshoot, the estimation error gets larger, but the estimation still guides the sensors to move toward the direction of the source. That is why we observe that the estimation error increases while the sensor-source distance decreases during this period. We also note that in Fig. \ref{subfig:Variation-Details-Variation-II} for Variation II, the estimations of two sensors are in the direction of the source, while the estimations for the other two are not in the direction of the source. But because the two sensors with the `right direction' come closer to the source, their measurements become better/more informative, which is shared with the other two sensors and `saves' them from divergence. This phenomenon is common in experiments and explains why Variation II performs well in terms of the median distance metric in Fig. \ref{fig:Internal-Compare} even though it uses no consensus in estimation and motion planning.

	\subsection{Sensitivity to the Initial Guess}\label{sec:sensitive-initial-guess}
	Our distributed algorithm is not only as effective in source seeking as the centralized implementation; it is also more robust in multiple ways. For example, the distributed algorithm is much less sensitive to the initial guesses of the source location. In the following experiments, we compare the sensitivity of the distributed and centralized implementation to initial guesses in terms of convergence and estimation error. Five sensors are used in all experiments. The locations of the sensors and the source are initialized in the same way as in the previous experiments. Random initial guesses $\hat{q}_0$ are given to estimators in both the distributed and centralized implementations. It is defined by
	\begin{equation}
	    \hat{q}_0 = q + \frac{D}{2}\cdot\begin{bmatrix}
	    \Delta_1\\
	    \Delta_2
	    \end{bmatrix},\Delta_{1,2}\overset{i.i.d.}{\sim}Unif([-1,1]).
	\end{equation}
	In other words, $\hat{q}_0$ is a random location in a box with side length $D$ centered at $q$. Different estimators may have different $\hat{q}_0$, but the level of deviation $D$ is kept the same among estimators for one experiment.
	
	Fig. \ref{fig:InitialGuess-Robustness} shows the distributed algorithm exhibit clear advantages over the centralized implementation in convergence and estimation when deviation $D$ increases. The distributed algorithm prevails because using more estimators adds robustness. Although the centralized implementation has more data for estimation, it runs only one estimator. If the initial guess is poor, the estimator may fail. In contrast, sensors of the distributed algorithm form an ensemble of estimators. As long as the majority of the initial guesses are decent, individual estimations should improve after rounds of consensus. 
	
	\begin{figure*}
		\centering
		\begin{subfigure}{0.24\linewidth}
			\includegraphics[width=\textwidth]{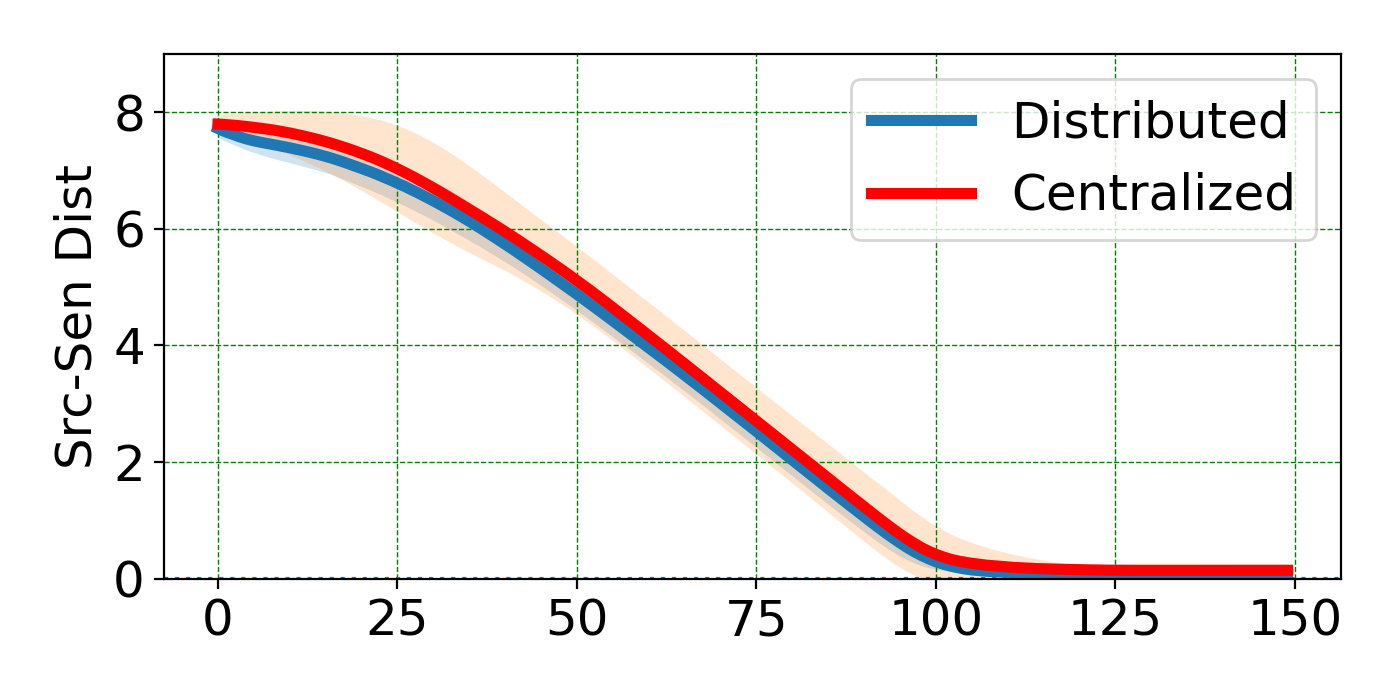}	
		\end{subfigure}
		\begin{subfigure}{0.24\linewidth}
			\includegraphics[width=\textwidth]{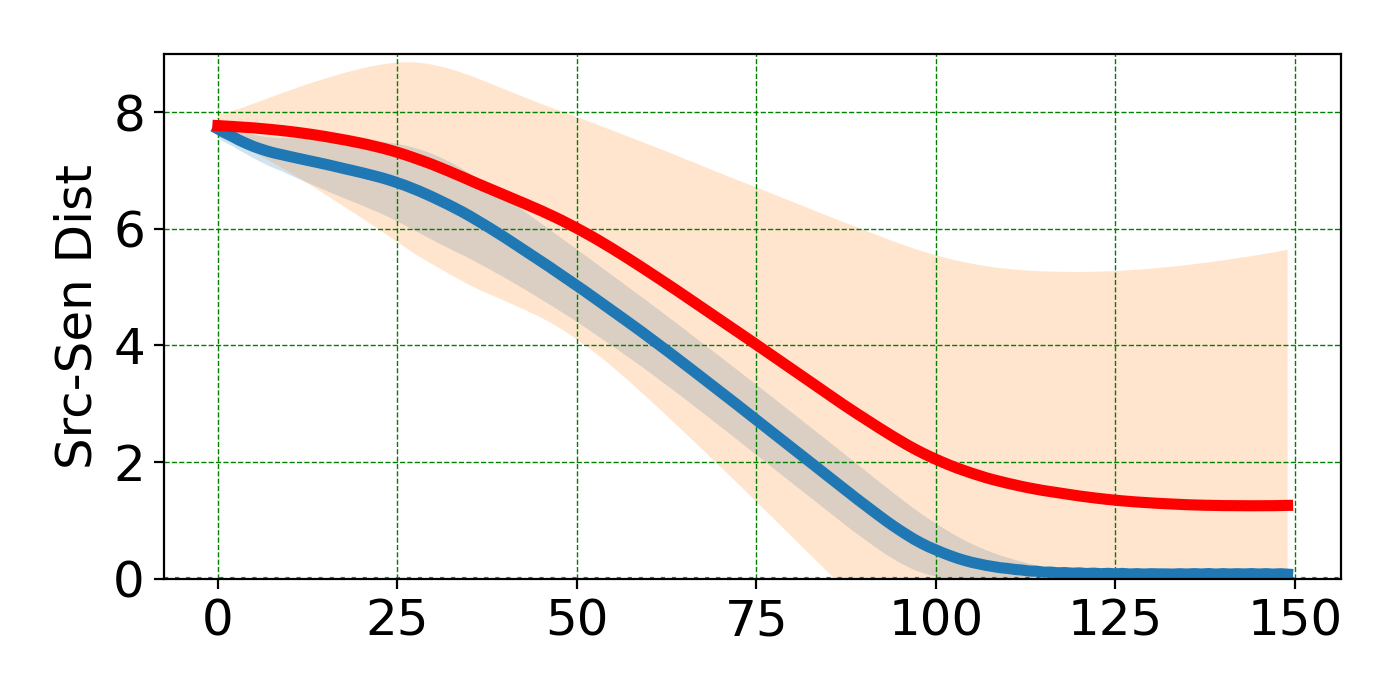}	
		\end{subfigure}
	\begin{subfigure}{0.24\linewidth}
		\includegraphics[width=\textwidth]{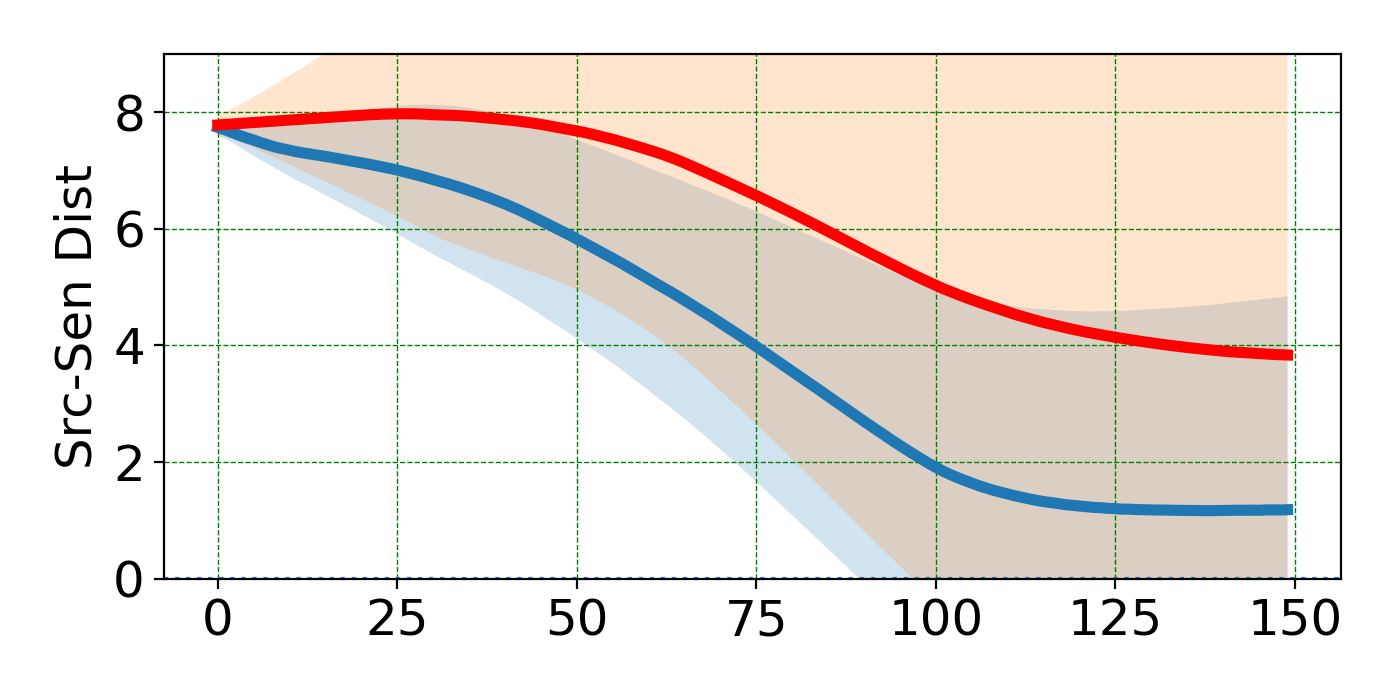}	
	\end{subfigure}
	\\
		\begin{subfigure}{0.24\linewidth}
			\includegraphics[width=\textwidth]{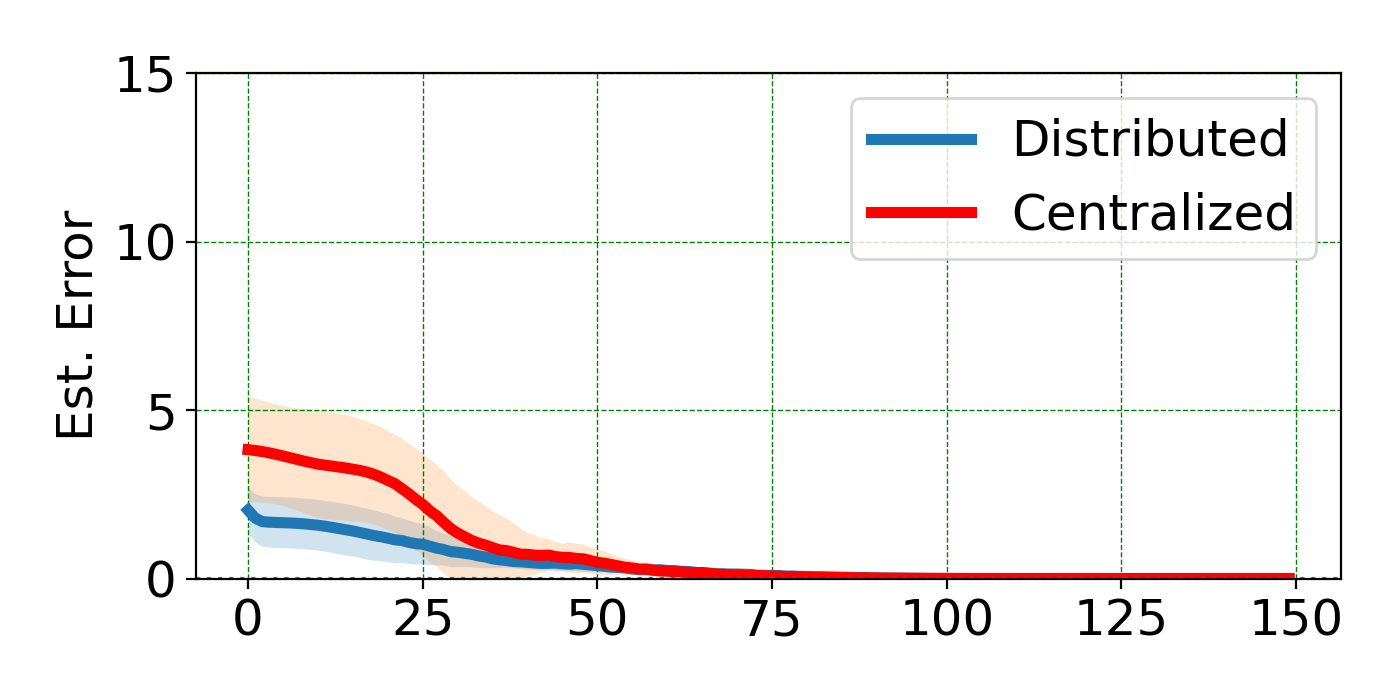}	\caption{$D=10$}
		\end{subfigure}
		\begin{subfigure}{0.24\linewidth}
	\includegraphics[width=\textwidth]{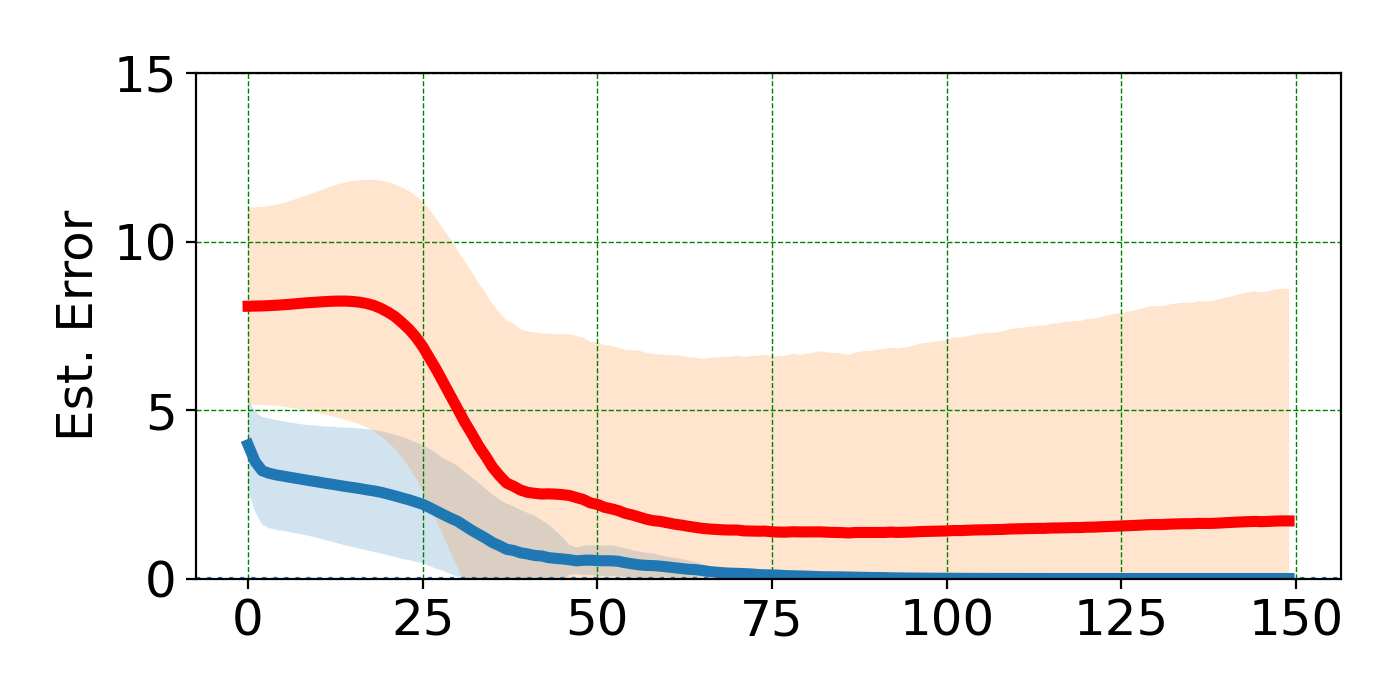}	\caption{$D=20$}
		\end{subfigure}
			\begin{subfigure}{0.24\linewidth}
		\includegraphics[width=\textwidth]{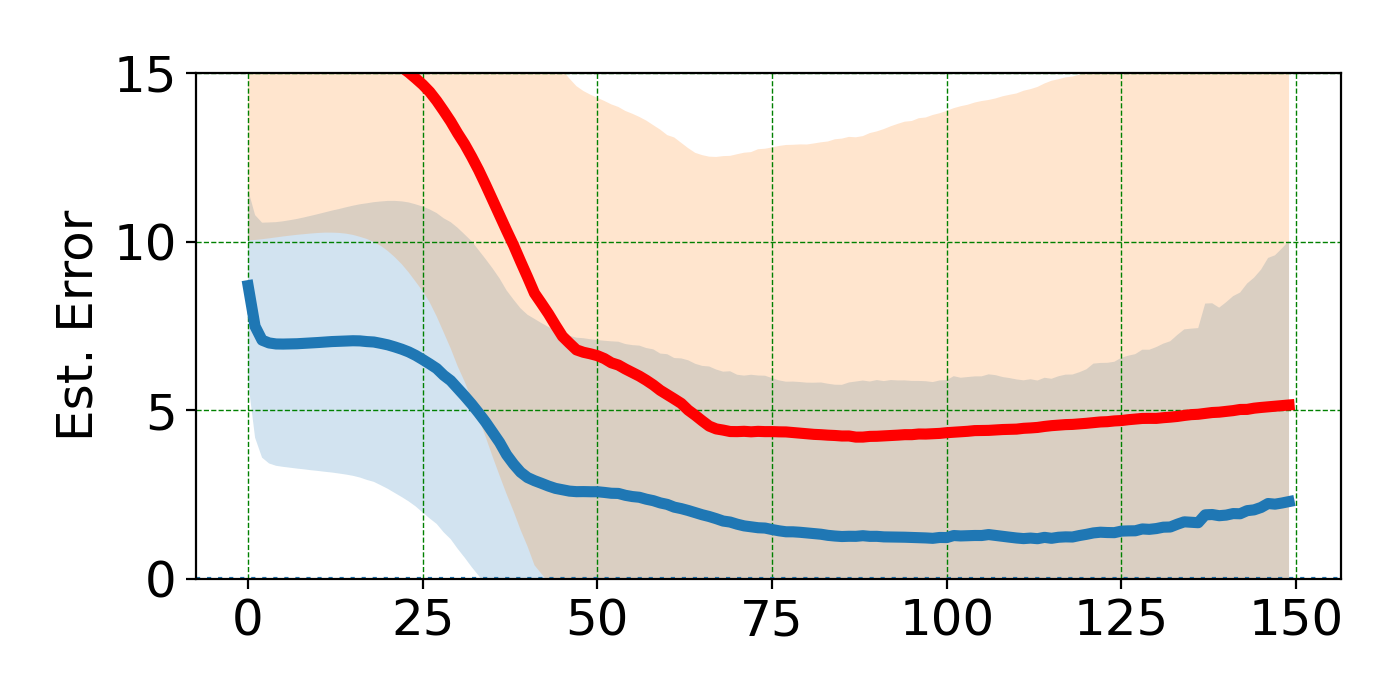}	\caption{$D=40$}
	\end{subfigure}
		\caption{How the initial guess deviation influences convergence and estimation error. The x-axes correspond to time steps.
    }\label{fig:InitialGuess-Robustness}
	\end{figure*}
	
	\subsection{Robustness to Communication Delay}\label{sec:delay}
Apart from being robust to errors in initial guesses, the distributed algorithm is more robust than the centralized implementation to communication delay. Note that communication delay typically grows with the incoming information rate\cite[M/M/1 system, Section 4.1.2]{ng2008queueing}, which usually increases with the number of incoming connections. Therefore, the delay increases with the number of incoming connections. The consequence is that the delay in centralized implementation increases with the number of sensors used, while the delay in the distributed algorithm can remain low as the network expands as long as the incoming degree of each node stays unchanged.
	
The following numerical experiments show that the distributed algorithm is more robust to communication delay. We assume each incoming connection to a node brings a $0.5$ time step delay in information passing. That is, if $m$ sensors are used in the centralized algorithm, the centralized controller receives all information with a delay of $\lfloor\frac{m}{2} \rfloor$ time steps; if there are $k$ neighbors connected to sensor $j$ in the distributed algorithm, sensor $j$ receives all information with a delay of $\lfloor\frac{k}{2} \rfloor$ time steps. The delay is modeled in this way since each sensor takes and sends out measurements to the neighbors at a roughly constant rate. Initialization is the same as in the previous experiments. In particular, each sensor in the distributed algorithm is connected to two neighbors regardless of the number of sensors. 
	
Fig. \ref{fig:Delay} shows that when the delay is in effect, the advantage of the centralized implementation over the distributed algorithm dissipates as the number of sensors grows. In the centralized implementation, the benefit of more sensor data does not offset the negative effect of increasing delay, and the performance gradually degrades. On the other hand, distributed sensors benefit from the increasing global information shared in the network while the delay remains low. Therefore, the centralized implementation is eventually overtaken by the distributed algorithm. 
	
	\begin{figure*}
		\centering
		\begin{subfigure}{0.24\linewidth}
			\includegraphics[width=\textwidth]{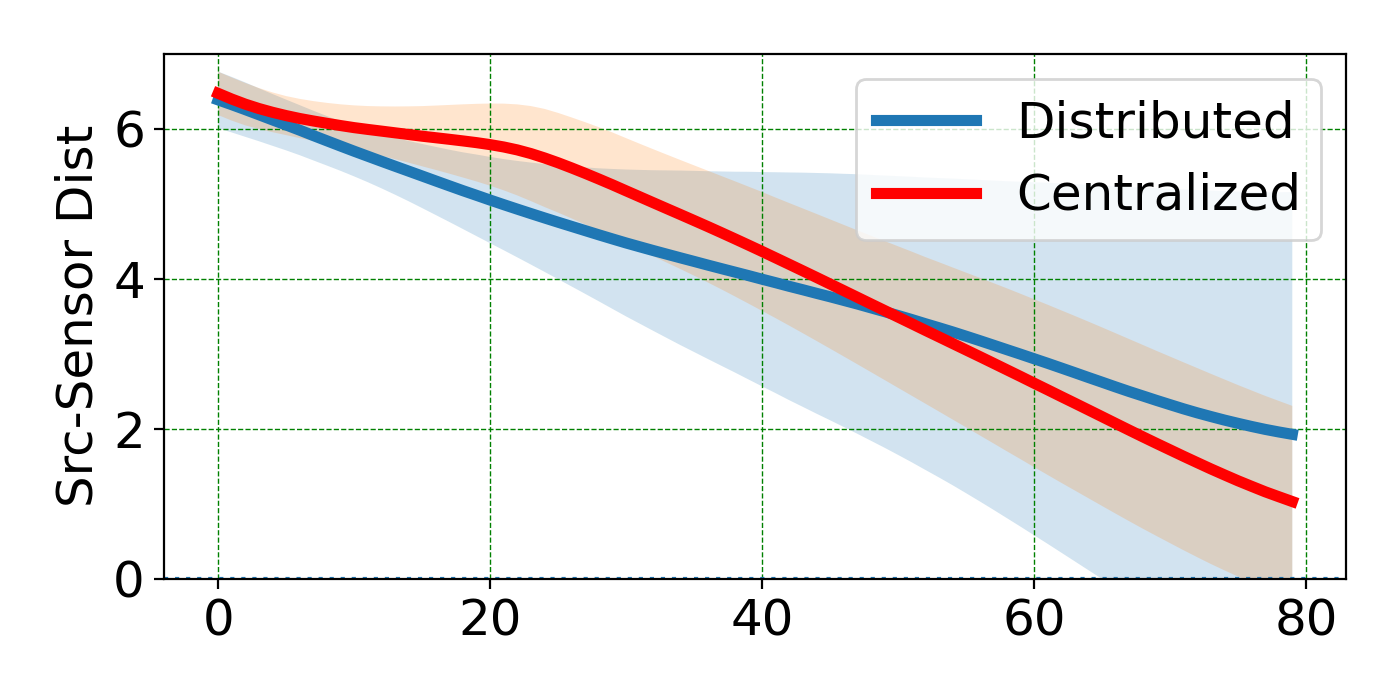}	
		\end{subfigure}
		\begin{subfigure}{0.24\linewidth}
			\includegraphics[width=\textwidth]{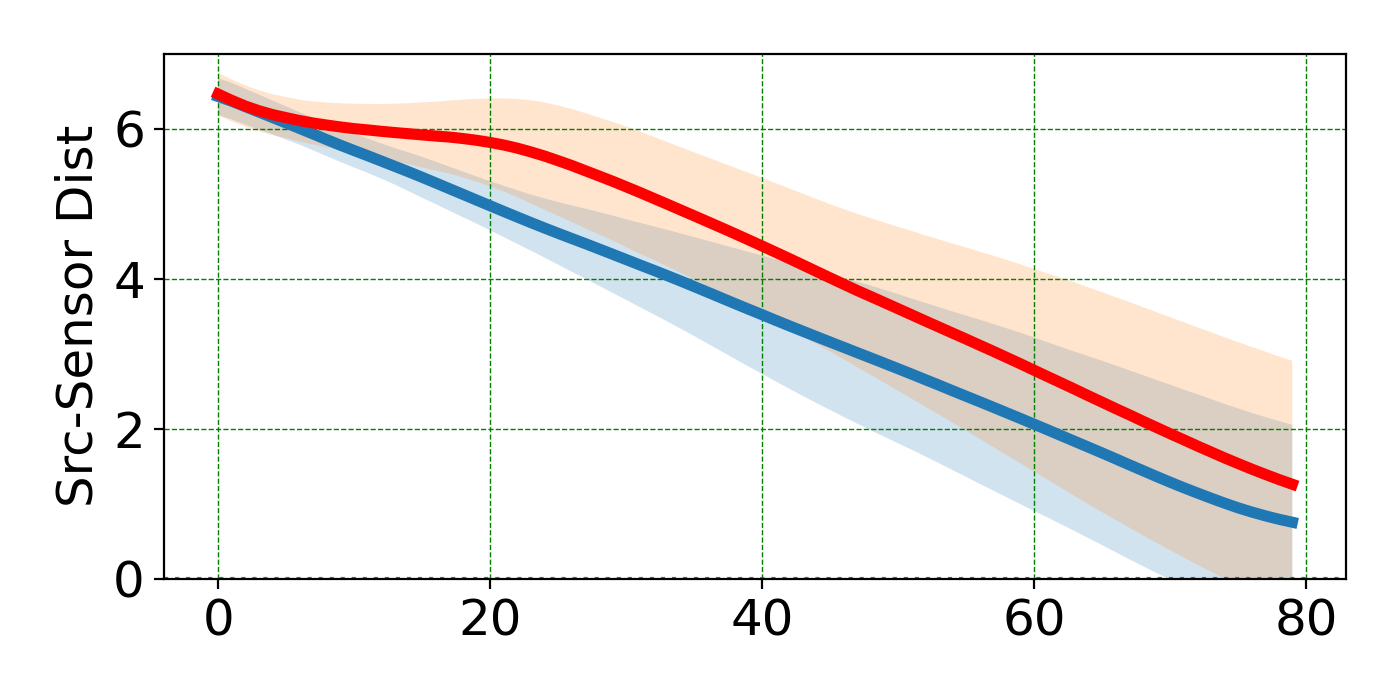}	
		\end{subfigure}
	\begin{subfigure}{0.24\linewidth}
		\includegraphics[width=\textwidth]{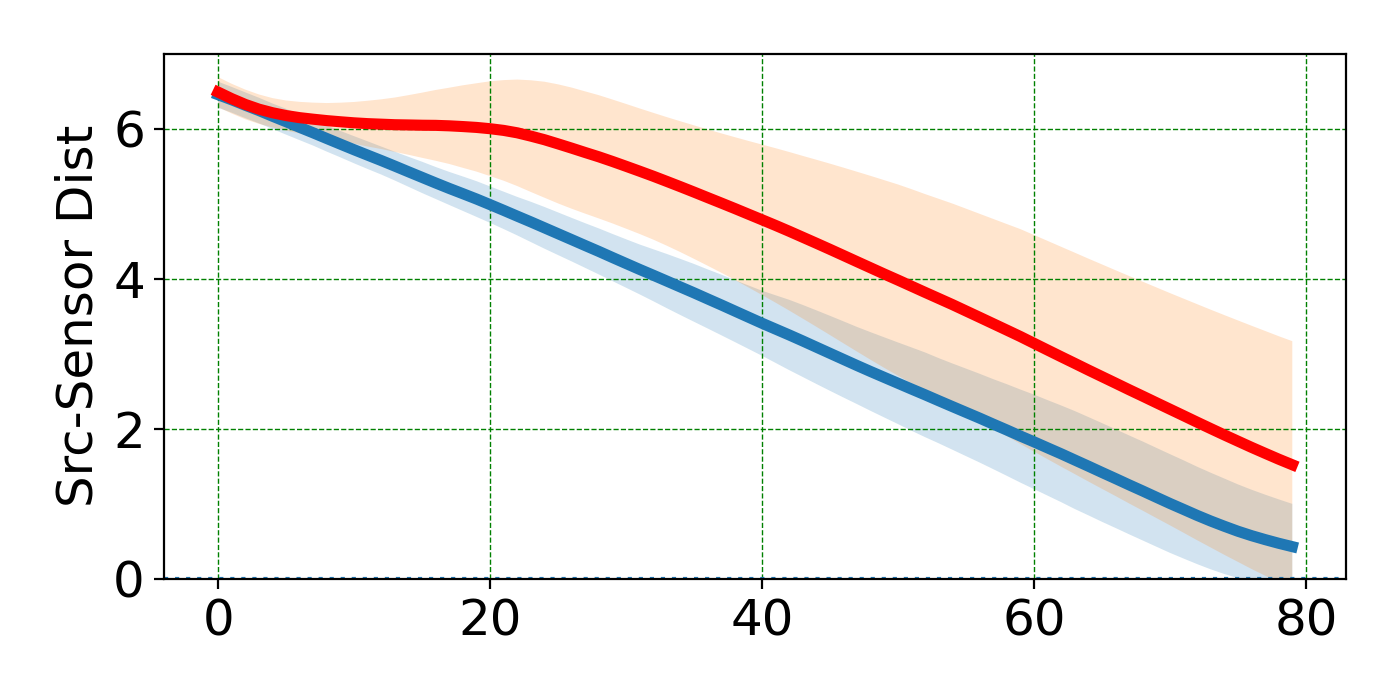}
	
	\end{subfigure}
	\\
		\begin{subfigure}{0.24\linewidth}
			\includegraphics[width=\textwidth]{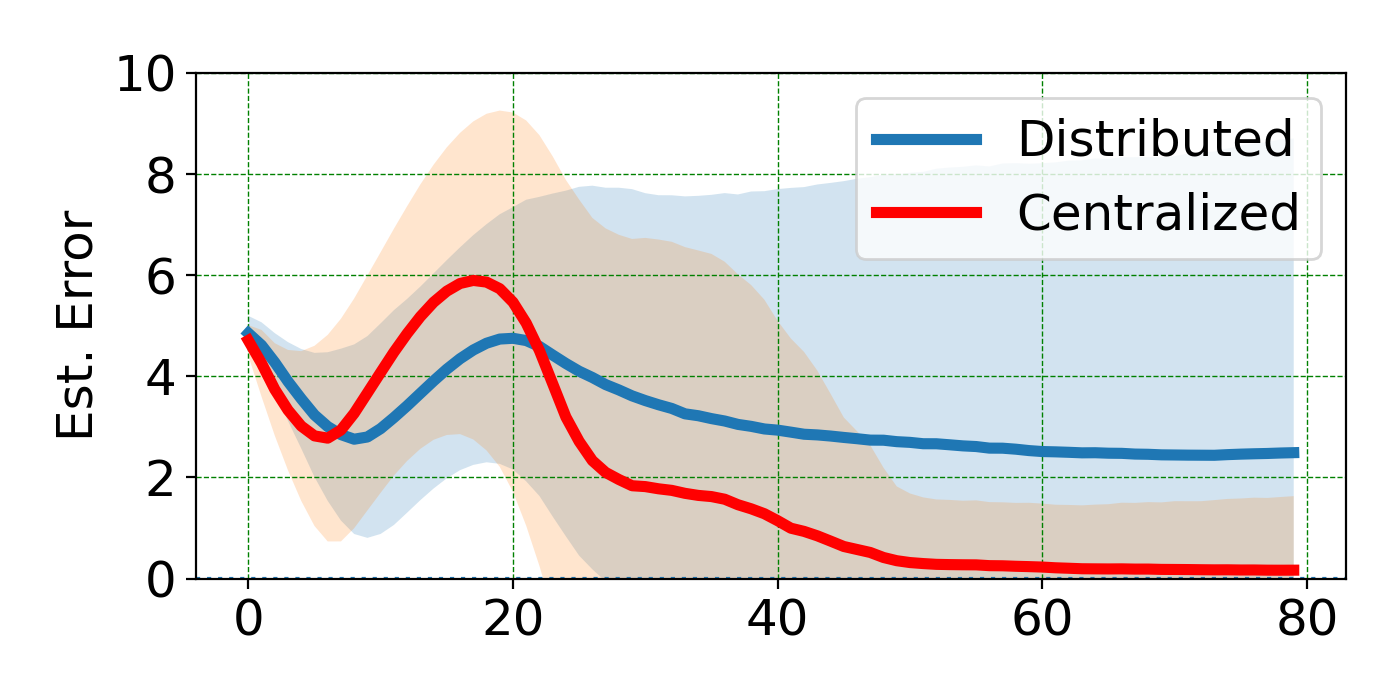}	\caption{10 Sensors with Delay}
		\end{subfigure}
		\begin{subfigure}{0.24\linewidth}
			\includegraphics[width=\textwidth]{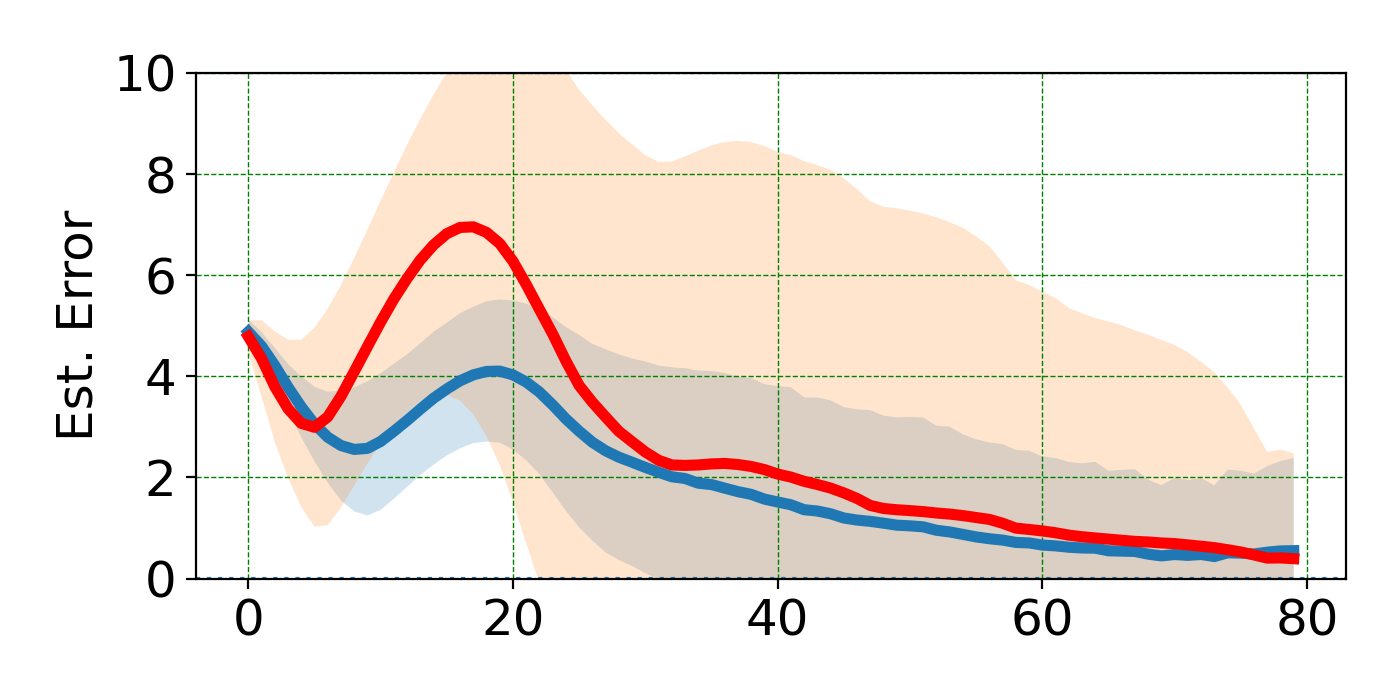}	\caption{20 Sensors with Delay}
		\end{subfigure}
			\begin{subfigure}{0.24\linewidth}
		\includegraphics[width=\textwidth]{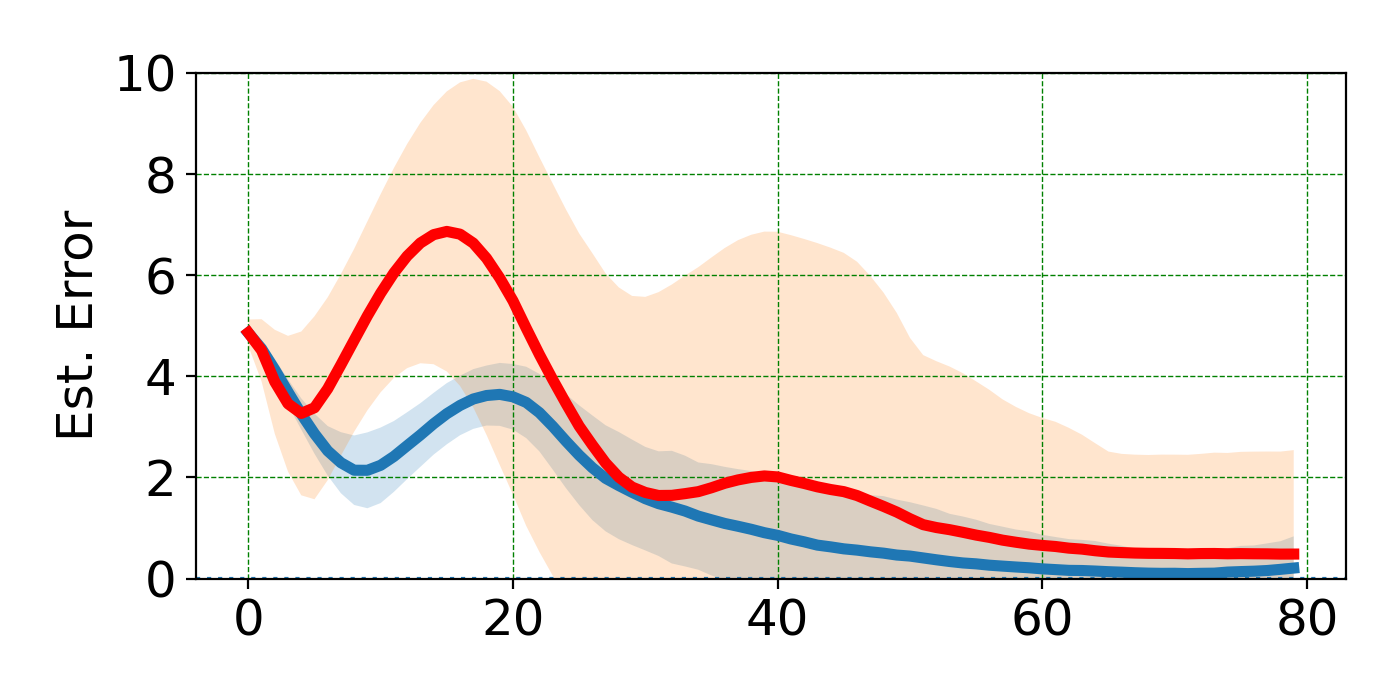}	\caption{40 Sensors with Delay}
	\end{subfigure}
		\caption{How the communication delay influences convergence and estimation. The x-axes correspond to time steps.
		}\label{fig:Delay}
	\end{figure*}

\subsection{Lab Implementations}\label{sec:real-lab-demo}
We implement our distributed algorithm on Turtlebot3 robotic ground vehicles to seek a light source in a dark room. The light source is an LED lamp. Light sensors are installed on Turtlebots to measure the local light intensity. The robots communicate through an ad-hoc WiFi network, each being an individual node. We only allow each robot to communicate with two other pre-determined robots throughout the experiment as a reasonable approximation of a non-fully connected network. An indoor motion capture system captures the positions of the robots and the source.
	
The algorithm is implemented as a ROS 2 package in Python language\cite{Tianpeng_ROS_2_Package_2021}. The package is installed on the onboard computer of every robot. By running the package, every robot gathers its position from the motion capture system, communicates information required in estimation and waypoint planning, makes estimations, and decides its control actions.
	
As a measurement model is needed to estimate the source location, we fit a function for each robot describing the relationship between their measurements and distances to the source before deploying them in source seeking. The function is defined by 
	
\begin{equation}\label{eq:exact_meas_model}
    g_i(r_i) = k_i(r_i-C_{1,i})^{-b_i}+C_{0,i},
\end{equation}
where $y_i,r_i$ are the measured data from the $i$th mobile sensor, and $k_i>0,\,b_i>0,\,C_{0,i},\,C_{1,i}$ are the associated model parameters.

In the implementation, a robot is deemed to have reached the source if its distance from the source is below $0.5m$. A separate program that runs independently from the source seeking algorithm monitors whether any robot has reached the source using the motion capture data and notifies the robots if they contact the source. 
The accompanying video shows that the robots implementing our distributed algorithm converge to the source consistently. 

\begin{remark}
Our lab resources constrain the hardware implementation; thus, the hardware experiments are less comprehensive than the numerical tests, an aspect of the experiments that could be improved. Nevertheless, the hardware experiments serve as a proof-of-concept showing the proposed algorithms are implementable on physical systems with satisfactory performance. We hope the combination of extensive numerical tests and hardware experiments has demonstrated how the algorithm works.
\end{remark}

\section{Conclusion}
This paper proposes a method for multi-robot source seeking that utilizes the Fisher Information associated with the estimated source location to direct the robots' movements. We show that improving the trace of the inverse of Fisher Information aids estimation while guiding the robots to converge to the source. The method is extended to the distributed setting where a central controller is absent, and the robots make individual decisions while communicating via a distributed network. We propose consensus schemes that make distributed estimation and gradient update effective. The algorithms are verified in numerical experiments and on physical robotic systems. 
	
Our work leads to a few future directions. One is to combine our work with Bayesian optimization methods to learn the measurement model while the robots seek the source. Progress in this direction removes the requirement of a known measurement model, making the method more flexible. It is also imperative to develop the ability to navigate an environment with obstacles during source seeking, especially when deploying our method in complex environments, such as search-and-rescue missions.

\appendices
\section{Motion Planning}\label{append:motion-planning}
The motion planner $MP$ can be viewed as a device that transforms the planned waypoints into low-level actuation of a mobile sensor: by applying $(u_i(1),\ldots,u_i(T))=MP(\tilde{p}_i(0),\ldots,\tilde{p}_i(T))$ to sensor $i$, the sensor will follow the trajectory of the waypoints $\tilde{p}_i(0),\ldots,\tilde{p}_i(T)$. The motion planner $MP$ typically requires the knowledge of the sensors' motion dynamics to compute the control inputs, and any method that can fulfill this task can be a motion planner in Algorithm~\ref{alg:centralized-alg}. In our implementation, the motion planner combines spline-based motion generation described in\cite{walambe_optimal_2016} and the Linear Quadratic Regulator (LQR). Regarding the planning horizon $T$ choice, in the Gazebo simulations and hardware implementation, we set $T=20$ to ensure stability in sensor movement and robustness to disturbances. Meanwhile, in numerical studies where the robot dynamics are not simulated, it suffices to set $T=1$.

\section{Proof of Propositions \ref{prop:reach-the-source} and \ref{prop:reach-the-source-2}}\label{append:reach-the-source}

Recall that $\hat{r}_i$ is the directional unit vector pointing from $q$ (the source) to $p_i$ (sensor $i$'s position), i.e.,
$$
\hat{r}_i=\frac{p_i-q}{\|p_i-q\|}.
$$
We first derive expressions for $\textup{FIM}$ and the loss function $L$. 
\begin{lemma}\label{prop:FIM-formula}
Suppose individual measurements are isotropic as in Assumption~\ref{assumption:measurement}. Then
	\begin{equation}\label{eq:FIM-formula}
		\textup{FIM} = \sum_{i=1}^m \left|g_i'(r_i)\right|^2 \hat{r}_i\hat{r}_i^T.
	\end{equation}
Moreover,
	\begin{equation}\label{eq:L_sum_inverse_eig_FIM}
		L(\textbf{p},q)=\tr(\textup{FIM}^{-1})= \sum_{i=1}^m \frac{1}{\lambda_i(\textup{FIM})}>0,
	\end{equation}
	where $\lambda_i(\textup{FIM})$ is the $i$th eigenvalue of $\textup{FIM}$.
\end{lemma}

\begin{proof}
	By noticing that $r_i=\|p_i-q\|$ and using the chain rule of calculus, one can show that
	\begin{equation*}
		\begin{split}
			\nabla_q h_i (p_i,q) = -g_i'(r_i)\,\hat{r}_i.
		\end{split}
	\end{equation*}
	Denote $A = \nabla_q H(\textbf{p},q)$ and $A_i = \nabla_q h_i(p_i,q)$. Then from the definition of $H$ in~\eqref{eq:MeasurementFunc}, we have
	\begin{equation*}
		A = \begin{bmatrix}
			A_1|A_2|\cdots|A_m
		\end{bmatrix},
	\end{equation*}
	and so
	\begin{equation*}
			\textup{FIM} = AA^T = \sum_{i=1}^m A_i A_i^T = \sum_{i=1}^m \left|g_i'(r_i)\right|^2 \hat{r}_i\hat{r}_i^T.
	\end{equation*}
	
	Then~\eqref{eq:L_sum_inverse_eig_FIM} follows since $$\textup{FIM}\succeq (\min_i |g_i'(r_i)|)\sum_{i=1}^m  \hat{r}_i\hat{r}_i^T\succ 0, $$ and the eigenvalues of $\textup{FIM}^{-1}$ are the reciprocals of those of $\textup{FIM}$.
\end{proof}

With the formulas of $L$ and $\textup{FIM}$, and by further imposing the monotonicity property of $h_i$ in Assumption~\ref{assumption:measurement}.2, we can show that minimizing $L$ leads to reaching the source as stated in Proposition~\ref{prop:reach-the-source}.

\begin{proof}[Proof of Proposition~\ref{prop:reach-the-source}]

First, we take the maximum over all the summation coefficients in $\textup{FIM}$ to get
	$$
	\textup{FIM} \preceq \max_{i=1,2...m} \left|g_i'(r_i)\right|^2 \sum_{i} \hat{r}_i\hat{r}_i^T
	$$
Let $\lambda_{\min}(P)$ denote the smallest eigenvalue of a positive semi-definite matrix $P$. Then by~\eqref{eq:L_sum_inverse_eig_FIM}, we can see that
$$
	\frac{1}{\lambda_{\min}(\textup{FIM})}\leq L(\textbf{p},q)
	\leq \frac{k}{\lambda_{\min}(\textup{FIM})},
$$
	where we recall that $k$ is the dimension of $q$ and also the size of $\textup{FIM}$. Consequently,
	$$
	\begin{aligned}
		\frac{1}{L(\textbf{p},q)}
		& \leq \lambda_{\min}(\textup{FIM})\\
		& \leq \max_{i} \left|g_i'(r_i)\right|^2
		\cdot\lambda_{\min}\!\left(\sum_{i=1}^m \hat{r}_i\hat{r}_i^T\right).
	\end{aligned}
	$$
	Since we assume $\sum_{i=1}^m\hat{r}_i\hat{r}_i^T\succ 0$, there is $\lambda_{\min}(\sum_{i=1}^m \hat{r}_i\hat{r}_i^T)>0$. Then because $\lambda_{\min}(\sum_{i=1}^m \hat{r}_i\hat{r}_i^T)\leq m$, we have
	\begin{equation}
		\frac{1}{m\cdot \max_{i} \left|g_i'(r_i)\right|^2}\leq  L(\textbf{p},q),
	\end{equation}
	So by Assumption~\ref{assumption:measurement}, $|g_i'(r_i)|$ monotonically increases as $r_i$ decreases, so we have $\min_i r_i$ tends to decrease as $L(\textbf{p},q)$ decreases, which completes the proof.
\end{proof}

 \begin{proof}[Proof of Proposition~\ref{prop:reach-the-source-2}]
     Let $\bar{\mathbf{p}}=[\bar{p}_1^\top,\bar{p}_2^\top,...,\bar{p}_m^\top]^\top := \lim_{t\rightarrow \infty} \mathbf{p}(t)$ be the location vector where the sensors converge to. 
     
     We prove the theorem by contradiction. Assume $\min_i ||p_i(t)-q||$ does not converge to $0$. Then there must be constant $\epsilon>0$ such that
      \begin{equation}\label{eq:not-reaching-assumption}
          \bar{r}_i := ||\bar{p}_i-q||\geq\epsilon,\forall i.
      \end{equation}
     Together with the continuity of $L(\mathbf{p},q)$ in $\mathbf{p}$ on $\{\mathbf{p}:p_{i}\neq q,\forall i\}$, the assumption leads to the conclusion that the infimum of $L$ in the theorem should be replaced with the minimum, and is attainable at $\bar{\mathbf{p}}$.

    We now construct a contradiction to the assumption in Eq. \ref{eq:not-reaching-assumption} showing $\bar{\mathbf{p}}$ cannot be the minimum of $L$. Consider the location vector $\mathbf{b}=[b_1^\top,b_2^\top,...,b_m^\top]$ defined as
    $$
     b_i := q+\frac{\epsilon}{2}\cdot\frac{\bar{p}_i-q}{\bar{r}_i},\forall i
    $$
    So that $$||b_i-q|| = \frac{\epsilon}{2}.$$
    Note that $\mathbf{b}$ satisfies $b_i\neq q$ for all $i$. We also note that
    \begin{equation*}
        \begin{aligned}
            &\textup{FIM}(\mathbf{b},q) \\
            =& \sum_{i=1}^m |g_i'(\frac{\epsilon}{2})| (\frac{\bar{p}_i-q}{\bar{r}_i})(\frac{\bar{p}_i-q}{\bar{r}_i})^\top\\
            \succeq &\alpha \underbrace{\sum_{i=1}^m |g_i'(\bar{r}_i)| (\frac{\bar{p}_i-q}{\bar{r}_i})(\frac{\bar{p}_i-q}{\bar{r}_i})^\top}_{\textup{FIM}(\bar{\mathbf{p}},q)},\\
        \end{aligned}
    \end{equation*}
    where $$\alpha = \min_{i=1,2,...,m} \frac{|g_i'(\epsilon/2)|}{|g_i'(\bar{r}_i)|}.$$

    Note that $\alpha>1$ since $\bar{r}_i>\epsilon/2$ and $|g_i'(r_i)|$ is assumed strictly monotone in $r_i$ in the theorem. Therefore, 

    \begin{equation*}
        \begin{aligned}
            L(\mathbf{b},q) = &\sum_{i=1}^m \frac{1}{\lambda_i(\textup{FIM}(\mathbf{b},q))}\\
            \leq& \frac{1}{\alpha} \sum_{i=1}^m \frac{1}{\lambda_i(\textup{FIM}(\mathbf{\bar{p}},q))} = \frac{1}{\alpha} L(\mathbf{\bar{p}},q)\\
            <& L(\mathbf{\bar{p}},q)
        \end{aligned}
    \end{equation*}

    Therefore, the location vector $\mathbf{b}$, which is closer to $q$ than $\mathbf{\bar{p}}$, achieves a lower $L$ value than $\mathbf{\bar{p}}$. We have reached a contradiction, and Eq. \eqref{eq:not-reaching-assumption} cannot hold. So $\bar{r}_i=0$ for some $i$ and $\lim_{t\rightarrow\infty}\min_{i} ||p_i(t)-q|| = 0$.
 \end{proof}

\section{Distributed Algorithm Variants I and II}\label{append:Distributed-Variants}
The two variations are formally defined as follows:
	\begin{enumerate}
	\item Sensor $j$ in variation I and II no longer share a loss function with all sensors. Instead, it uses the following local loss that may be different for other sensors.
	\begin{equation}\label{eq:local_objective}
	\begin{split}
	    &L_{j}(\mathbf{p}_j,q) =\\
	    &\tr\left[\big(\nabla_q H_j(\mathbf{p}_j,q)\cdot\nabla_q H_j(\mathbf{p}_j,q)^\top\big)^{-1}\right]
	\end{split}
	\end{equation}
	This local loss reflects the Fisher information about $q$ contained in the neighborhood measurements as opposed to the entirety of sensor measurements.
		\item Sensor $j$ in variation II performs the \textbf{local Kalman Filter} update defined in \eqref{eq:local_KF} using neighborhood information $\mathbf{p}_j,\mathbf{y}_j$. No consensus is required.
		\item Sensor $j$ in variation I and II skips line 5-7 in Algorithm \ref{alg:distri-alg}, and changes line 8-9 to be:
		
		\begin{itemize}
			\item 	$\tilde{\mathbf{p}}_j(0)\gets\mathbf{p}_j$
			\item 	Calculate joint-waypoints $\tilde{\mathbf{p}}_j(1),...,\tilde{\mathbf{p}}_j(T)$
			\begin{equation}
				\begin{aligned}
					\tilde{\mathbf{p}}_j(t+1) \gets \tilde{\mathbf{p}}_j(t)-\alpha_t M_t \nabla_{\mathbf{p}_j}L_j(\mathbf{p}_j=\tilde{\mathbf{p}}_j(t),\hat{q}_j)\\
				\end{aligned}
			\end{equation}
			\item Extract sensor $j$'s waypoints $\tilde{p}_j(1),...,\tilde{p}_j(T)$ from joint-waypoints $\tilde{\mathbf{p}}_j(1),...,\tilde{\mathbf{p}}_j(T)$
		\end{itemize}
	\end{enumerate}
	Note that no consensus is required to compute the gradient of $L_j$. The remaining parts of the algorithm are identical to our algorithm and the variations. In particular, sensor $j$ in Variation I still performs consensus EKF update.

\section{Understanding the geometric properties of the trajectory of the covariance metric}\label{append:covariance}

Figure \ref{subfig:covariance} shows that the sensors under the covariance metric approach the source in straight lines and then spread out a little when close to the source. We can understand the initial straight-line behavior by analyzing the direction of the gradient when sensors are far away from the source.  Recall that the covariance metric defined in Section \ref{sec:Metric-Compare} is $\tr(P)=\tr((\nabla H R^{-1} \nabla H^\top+P_0^{-1})^{-1})$. We set $R=I$ in Figure \ref{fig:Metric-Compare-Traj}, so the covariance metric is effectively $\tr(P)=\tr((\nabla H \nabla H^\top+P_0^{-1})^{-1})$.  Through matrix calculus, we can show that the covariance metric's gradient is given by the following.

\begin{equation*}
\begin{aligned}
    G_j:=&\nabla_{p_j} \tr((\nabla H \nabla H^\top+P_0^{-1})^{-1})\\
    = &-2\nabla_{p_j} A_j (\nabla H \nabla H^\top+P_0^{-1})^{-2} A_j.
\end{aligned}
\end{equation*}

Where $A_j:=\nabla_q g_j(||p_j-q||)$, and
$A_j$, $\nabla_{p_j} A_j$ can be shown to be

\begin{equation*}
    \begin{aligned}
   A_j&= -g_j'(r_j)\hat{r}_j\\
   \nabla_{p_j}A_j&=-g_j''(r_j)\hat{r}_j \hat{r}_j^\top - \frac{g_j'(r_j)}{r_j} \hat{t}_j\hat{t}_j^\top\\
    \end{aligned}
\end{equation*}

Where $\hat{r}_j = \frac{p_j-q}{||p_j-q||}$ is the unit vector pointing from $q$ to $p_j$, and $\hat{t}_j$ is the unit vector orthogonal to $\hat{r}_j$.

Since we assume the individual measurement function $|g_j'(r_j)|$ gets smaller as $r_j$ increases, $P_0^{-1}$ dominates $\nabla H \nabla H^\top$ when $r_j$ are large(sensors are far away from the source).
We note that we set $P_0=I$ in Figure \ref{subfig:covariance}. Therefore, when sensors are far away from the source, the gradient for sensor $j$ is approximately

\begin{equation*}
    \begin{aligned}
    G_j\approx &-2\nabla_{p_j} A_j (P_0^{-1})^{-2} A_j\\
    = &-2(\nabla_{p_j} A_j)  A_j\\
    = &-2g_j'(r_j)g_j''(r_j) \hat{r}_j.
    \end{aligned}
\end{equation*}
Also, since the $g_j$ we consider are in the form of $g_j(r_j) = k_j(r_j-C_{1,j})^{-b_j}+C_{0,j}$, its first and second order derivatives satisfies $g'_j(r_j)\leq 0$ and $g''_j(r_j)\geq 0$. Therefore, the descent direction $-G_j$ is approximately aligned with the direction of $-\hat{r}_j$ when the sensors are far away from the source, causing the straight-line behavior.

When sensors are close to the source, analyzing the gradient direction is less trivial, but we do note that the covariance metric is roughly equivalent to $\tr(FIM^{-1})$ in this condition since the $\nabla H \nabla H^\top$ term dominates $P_0^{-1}$. We argued in Remark \ref{remark:metric-geo-prop}, that minimizing the $\tr(FIM^{-1})$ amounts to approaching the source and creating separation between the sensors simultaneously, which explains why the sensors spread out a little at the end.

\begin{figure}[ht]
\centering
\begin{subfigure}[t]{0.3\linewidth}
		\centering
		\includegraphics[width=\linewidth]{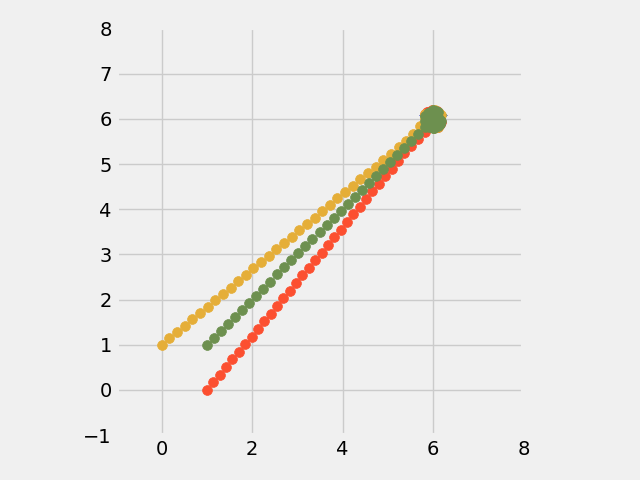}
		\caption{$P_0^{-1}=100 I$}\label{subfig:large-P0-inv}
	\end{subfigure}
\begin{subfigure}[t]{0.3\linewidth}
		\centering
		\includegraphics[width=\linewidth]{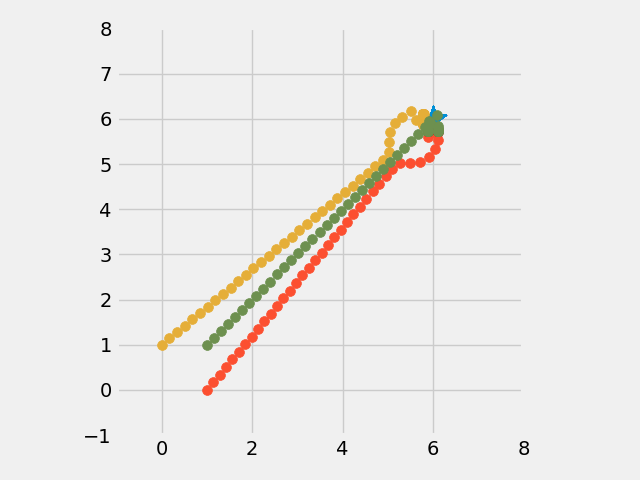}
		\caption{$P_0^{-1}=1.0 I$}\label{subfig:paper-P0-setup}
	\end{subfigure}
\begin{subfigure}[t]{0.3\linewidth}
		\centering
		\includegraphics[width=\linewidth]{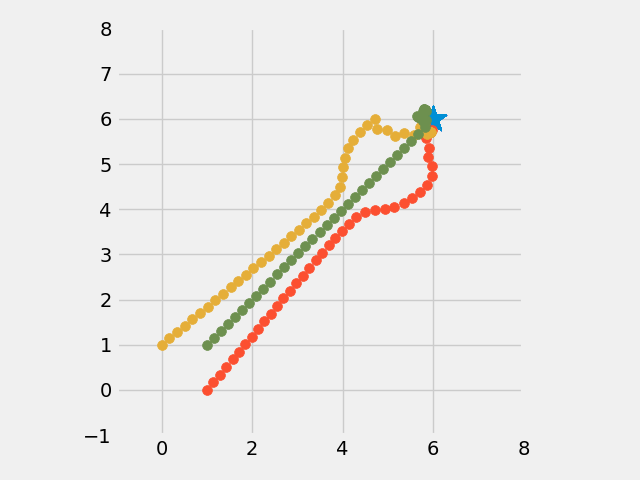}
		\caption{$P_0^{-1}=0.01 I$}
	\end{subfigure}
\begin{subfigure}[t]{0.3\linewidth}
		\centering
		\includegraphics[width=\linewidth]{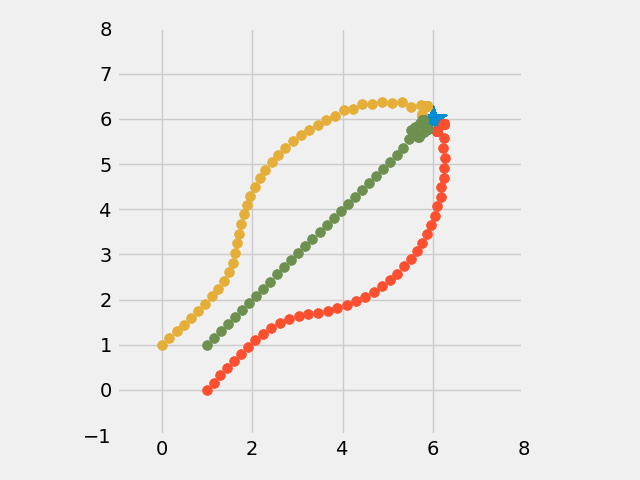}
		\caption{$P_0^{-1}=10^{-4} I$}
	\end{subfigure}
	\begin{subfigure}[t]{0.3\linewidth}
		\centering
		\includegraphics[width=\linewidth]{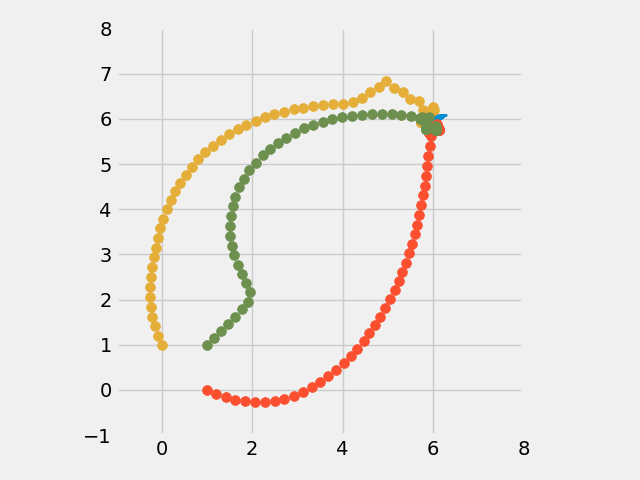}
		\caption{$P_0^{-1}=10^{-6} I$}\label{subfig;small-P0-inv}
	\end{subfigure}
	\caption{The effect of $P_0^{-1}$ on the size of the ``bump" of the gradient-descent trajectory of the covariance loss function. The matrix $R$ is fixed to be $1.0I$ for all figures above.}
	\label{fig:Covariance-bump}
\end{figure}

Fig. \ref{fig:Covariance-bump} provides further numerical evidence showing the dominance between $P_0^{-1}$ and $\nabla H \nabla H^\top$ terms influence the sensor movements. It shows that 
\begin{enumerate}
    \item When $P_0^{-1}$ is large, it dominates the covariance metric and gives rise to straight-line behavior; 
    
    \item On the other hand, when $P_0^{-1}$ becomes smaller, the $\nabla H \nabla H^\top$ term starts to dominate and the ``bump" in the trajectory becomes increasingly larger. 
    
\item When $P_0^{-1}$ becomes extremely small, the trajectory shows the behavior of encouraging separation among the sensors from the very start, which matches the behavior of the $\tr(FIM^{-1})$ metric, which is unsurprising since $FIM = \nabla H \nabla H^\top$.
\end{enumerate}

\bibliographystyle{IEEEtran}
\bibliography{citations}
 
\begin{IEEEbiography}[{\includegraphics[width=1in,height=1.25in,clip,keepaspectratio]{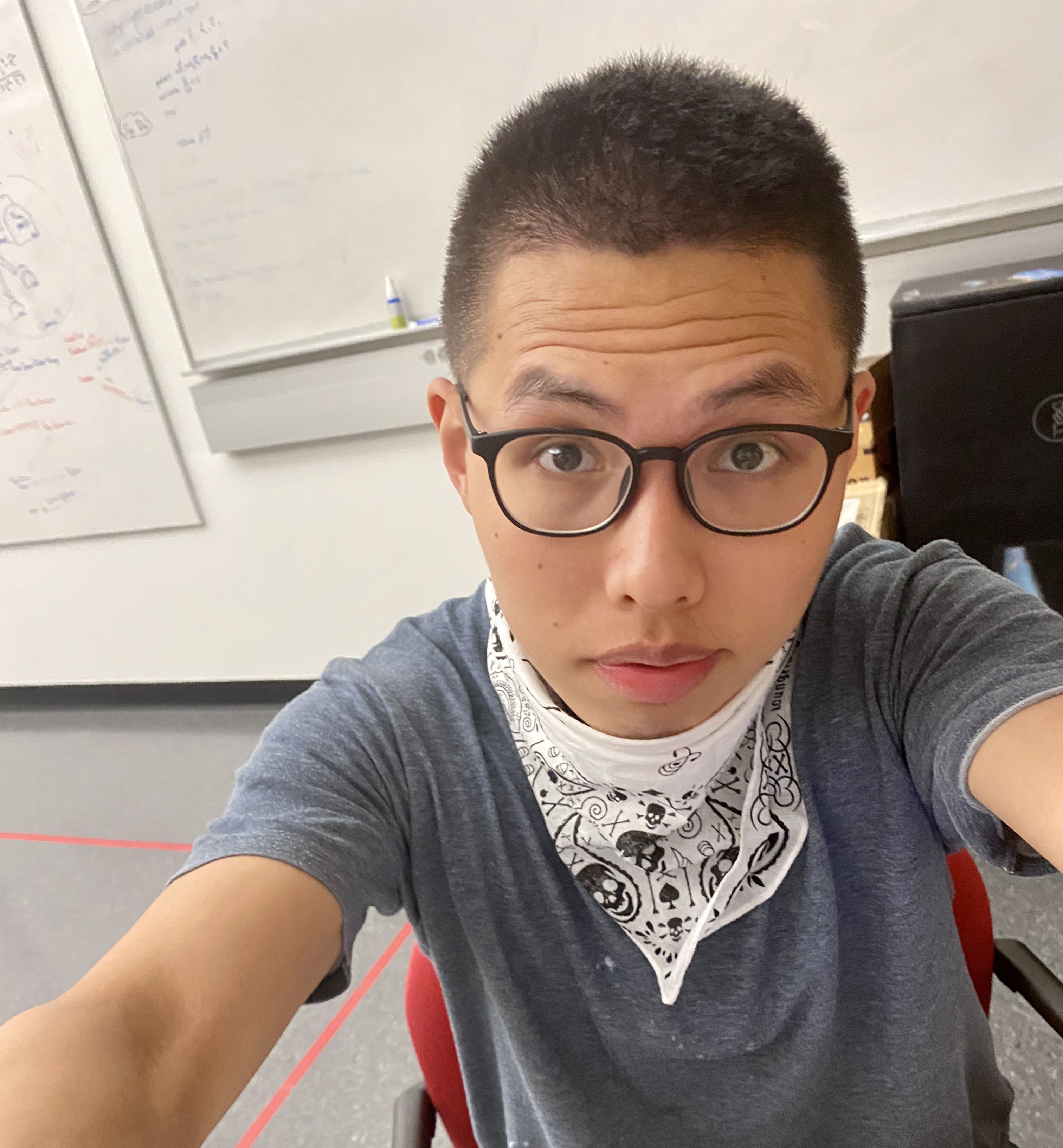}}]{Tianpeng Zhang}
received the B.S. degree in Computer Science and Mathematics from The Hong Kong University of Science and Technology in 2015. He is currently pursuing a Doctoral degree in Applied Mathematics at Harvard University, where his research focuses on multi-agent vehicle teams, multi-armed bandit learning, and robotic manipulators.
\end{IEEEbiography}

\begin{IEEEbiography}[{\includegraphics[width=1in,height=1.25in,clip,keepaspectratio]{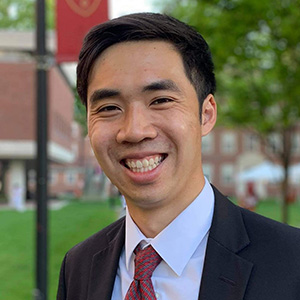}}]{Victor Qin} is a Ph.D. student in Aeronautics and Astronautics at the Massachusetts Institute of Technology. His interdisciplinary research focuses on the control, learning, and optimization of multi-agent networks, especially through incentive mechanisms and with an emphasis on unmanned air traffic management. The goal of his research is to highlight innovations in future semi-autonomous cyberphysical networks, and understand their interaction with economic and regulatory structures. Topics include cost-aware decentralized airspace allocation methods, cost-sharing frameworks for optimal network routing, and collision avoidance analysis of satellite constellations. He is a recipient of the NSF Graduate Research Fellowship and a Mathworks Fellow.
 \end{IEEEbiography}

\begin{IEEEbiography}[{\includegraphics[width=1in,height=1.25in,clip,keepaspectratio]{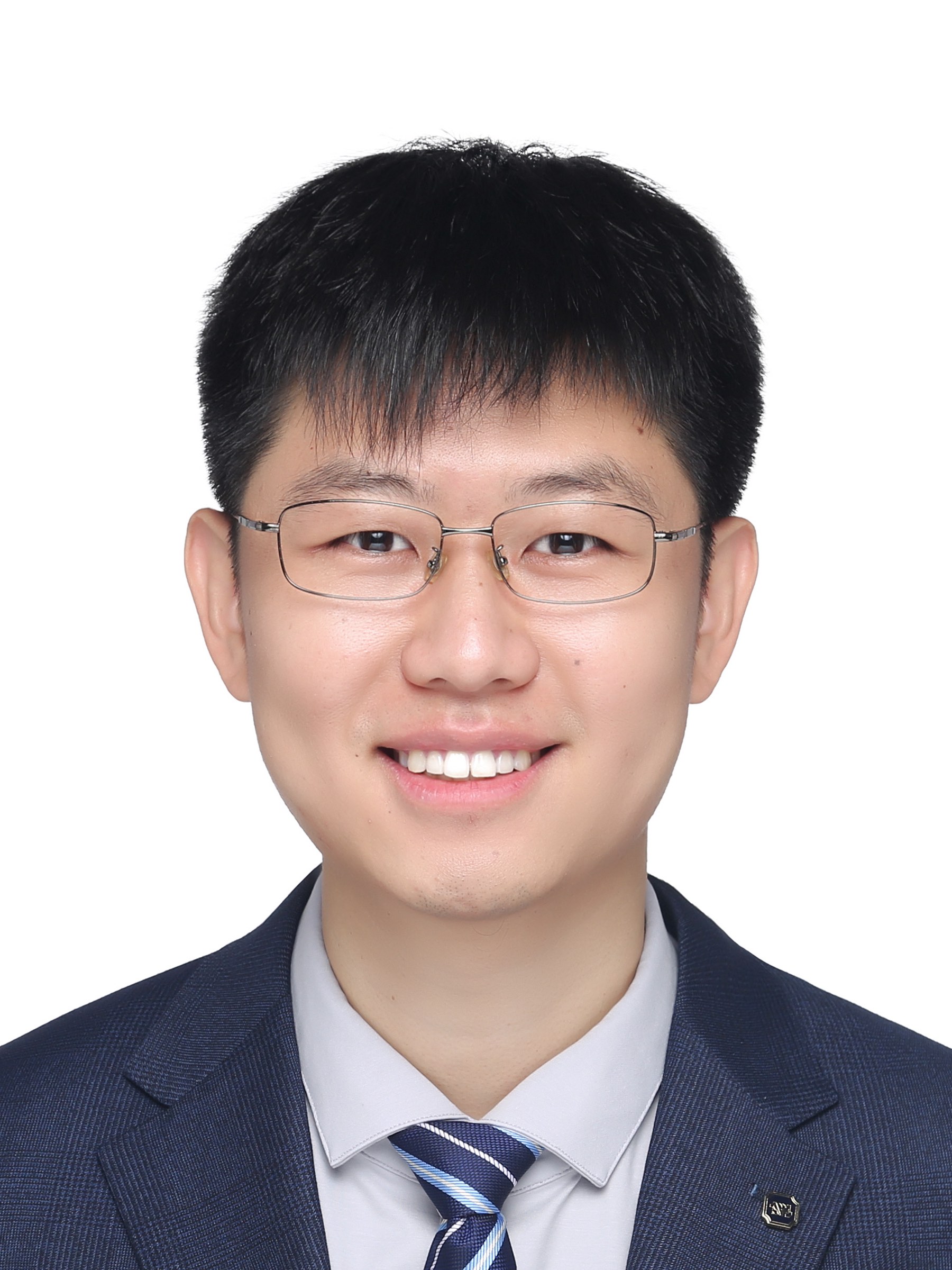}}]{Yujie Tang} received his bachelor’s degree in Electronic Engineering from Tsinghua University in 2013 and his Ph.D. degree in Electrical Engineering from the California Institute of Technology in 2019. He is currently an Assistant Professor in the Department of Industrial Engineering and Management at Peking University. Before joining Peking University, he was a postdoctoral fellow in the School of Engineering and Applied Sciences at Harvard University. His research interests lie in distributed optimization, control and learning and their applications in cyber-physical systems.
\end{IEEEbiography}

\begin{IEEEbiography}[{\includegraphics[width=1in,height=1.25in,clip,keepaspectratio]{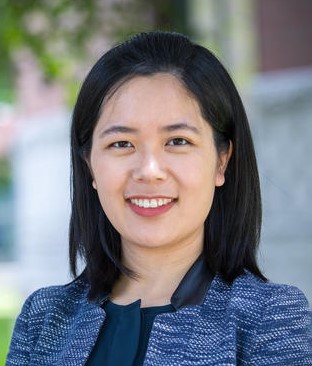}}]{Na Li}\ (Senior Member, IEEE) is the Winokur Family Professor of Electrical Engineering and Applied Mathematics at Harvard University. She received her Bachelor's degree in Mathematics from Zhejiang University in 2007 and her Ph.D. in Control and Dynamical systems from California Institute of Technology in 2013. She was a postdoctoral associate at the Massachusetts Institute of Technology from 2013-2014. She has held a variety of short-term visiting appointments, including the Simons Institute for the Theory of Computing, MIT, and Google Brain. Her research lies in the control, learning, and optimization of networked systems, including theory development, algorithm design, and applications to real-world cyber-physical societal system. She received the NSF career award (2016), AFSOR Young Investigator Award (2017), ONR Young Investigator Award(2019),  Donald P. Eckman Award (2019), McDonald Mentoring Award (2020), the IFAC Manfred Thoma Medal (2023), along with some other awards.
\end{IEEEbiography}

\end{document}